\documentclass{article} % For LaTeX2e
\usepackage{capt-of}
\usepackage{booktabs}
\usepackage{makecell}
\usepackage{cancel}
\usepackage{pifont}
\usepackage{enumitem}
\usepackage{wrapfig}
\usepackage{xspace}
\usepackage{float}
\usepackage{comment}
\usepackage{amsmath,amssymb,amsthm,mathtools,bm}
\usepackage{hyperref}
\usepackage{url}
\usepackage{amsmath,amssymb,amsfonts}
\usepackage{tikz}
\usepackage{pgfplots}
\usetikzlibrary{pgfplots.groupplots}
\pgfplotsset{compat=1.18}
\usepackage{caption}
\usepackage{booktabs}
\usepackage{multirow}
\usepackage{tabularx}
\usepackage{graphicx}
\usepackage{subcaption}
\usepackage{microtype}
\usepackage{subcaption}
\usepackage{hyperref}
\usepackage{amssymb}
\usepackage{mathtools}
\usepackage{amsthm}
\usepackage{aliascnt}
\usepackage{cancel}
\usepackage[capitalize,noabbrev]{cleveref}
\usepackage{algorithm,algorithmic}
\usepackage{graphicx}
\usepackage{subcaption}
\usepackage{cleveref}
\usepackage{xcolor}

\crefname{assumption}{Assumption}{Assumptions}

\newtheorem{assumption}{Assumption}
\newtheorem{definition}{Definition}
\newtheorem{proposition}{Proposition}
\newtheorem{corollary}{Corollary}
\newtheorem{theorem}{Theorem}
\newtheorem{lemma}{Lemma}

\newcommand{\vect}[1]{\mathbf{#1}}

\newcommand{\0}{\mat{0}}

\newcommand{\mat}[1]{\mathbf{#1}}

\newcommand{\z}{\vect{z}}
\newcommand{\X}{\mat{X}}

\newcommand{\C}{\mat{C}}

\newcommand{\y}{\vect{y}}

\newcommand{\I}{\mat{I}}

\newcommand{\tr}{\mathrm{Tr}}
\newcommand{\trace}{\tr}
\def\mSigma{{\mat{\Sigma}}}

\newcommand{\codeindent}{\hspace{1.1em}}

\newcommand{\dd}{\mathop{}\!\mathrm{d}}

\def\reals{\mathbb{R}}

\newcommand{\algmargin}{\hspace{-.9em}}

\newcommand{\explain}[1]{\left[\substack{#1}\right]}

\newcommand{\prn}[1]{\left({#1}\right)}

\newcommand{\biggprn}[1]{\bigg({#1}\bigg)}

\newcommand{\norm}[1]{\left\Vert{#1}\right\Vert}

\newcommand{\featnoise}{v_x} % kappa
\newcommand{\labelnoise}{v_z} % epsilon
\newcommand{\ratio}{\alpha} % gamma
\newcommand{\mnist}{\texttt{MNIST}\xspace}

\DeclareMathOperator*{\argmin}{argmin}
\makeatletter
\newenvironment{recall}[1][\proofname]{\par
\normalfont \topsep6\p@\@plus6\p@\relax
\trivlist
\item\relax
{\bfseries
Recall #1}%
{\bfseries\@addpunct{.}}\hspace\labelsep\ignorespaces
}
\makeatother

\usepackage{titletoc} % <-- this provides \startcontents and \printcontents
\usepackage{ICLR/iclr2027_conference,times}

\usepackage{amsmath,amsfonts,bm}

\def\eqref#1{equation~\ref{#1}}
\def\1{\bm{1}}

\def\mSigma{{\bm{\Sigma}}}

\DeclareMathAlphabet{\mathsfit}{\encodingdefault}{\sfdefault}{m}{sl}
\SetMathAlphabet{\mathsfit}{bold}{\encodingdefault}{\sfdefault}{bx}{n}

\newcommand{\E}{\mathbb{E}}

\usepackage{hyperref}
\usepackage{url}
\iclrfinalcopy

\title{Optimal Explicit and Implicit Regularization Strengths in High-Dimensional Continual Linear Regression}

\author{Gilad Karpel \\
Department of Data and Decisions Sciences, Technion\\
\texttt{karpelgilad@gmail.com} \\
\AND
Daniel Soudry  \\
Department of Electrical and Computer Engineering, Technion \\
\texttt{daniel.soudry@gmail.com}
}

\begin{document}

\maketitle

\begin{abstract}
We study continual linear regression in a high-dimensional regime, where a learner observes a sequence of tasks and updates its parameters sequentially. In this setting, we derive an explicit closed-form expression for the generalization loss of a model trained with gradient descent for a fixed number of steps per task, which holds for arbitrary teachers. We prove that the optimal fixed step size decays nearly inversely with the number of tasks, specifically as $\ln T / T$, for a single teacher and, more generally, i.i.d.~teachers. Under this scaling, we show that the resulting generalization loss is asymptotically equivalent to that achieved by optimally tuned L2 (isotropic) regularization.
Furthermore, we extend our analysis to generalized ridge regression with diagonal regularization matrices, obtaining a closed-form characterization of the generalization loss and identifying necessary and sufficient conditions under which the optimal regularizer is isotropic.
Finally, we identify a regime where the optimal step size remains finite even as the number of tasks tends to infinity, thus revealing a non-trivial tradeoff between task variability and implicit regularization, in contrast to the i.i.d.~teacher scaling. Our theoretical predictions are corroborated by experiments on both linear models and neural networks, providing practical guidance for the design of continual learning systems.
\end{abstract}

\section{Introduction}
In continual learning, the learner is expected to learn sequentially from data, without `forgetting' previous tasks. In recent years, driven by the remarkable success of artificial intelligence (AI) systems \citep{vaswani2017attention,saharia2022photorealistic,rombach2022high}, continual learning has attracted significant attention, as most modern models are static in nature and do not adapt to newly arriving data.
Unlike conventional machine learning models, continual learning involves learning from dynamic data distributions and faces a challenge known as catastrophic forgetting~\citep{mccloskey1989catastrophic,french1999catastrophic}, where adaptation to new distributions often leads to a substantial loss of performance on previously learned ones. 

Many continual learning works aimed to reduce forgetting by \emph{explicitly} penalizing changes in parameter space using per-parameter regularization terms \citep[e.g.,][]{kirkpatrick2017overcoming,zenke2017continual,aljundi2018memory,chaudhry2018riemannian}. Others \citep{lubana2022quadratic,smith2023closer} showed empirically that isotropic L2 regularization results in similar performance. Practically, it is also common to minimize the current task's \emph{unregularized} loss
with a gradient algorithm for a \emph{finite} number of steps 
(in contrast to theoretically learning to convergence).
This cheaper alternative \emph{implicitly} regularizes the model, even in stationary settings \citep{ali2019continuous}. 
To better understand catastrophic forgetting and how to mitigate it, most theoretical papers focus on a simplified setting, such as continual linear regression. This simplified yet analytically tractable setting captures essential features of sequential learning and has proven useful for analyzing diverse continual learning systems \citep{li2023fixed,hiratani2024disentangling,levinstein2026optimal,evron2025continual}.

Existing optimization-focused work \citep{evron2022catastrophic,evron2025continual,levinstein2026optimal,tsipory2026greedy} had to make a limiting assumption of joint realizability—i.e., that training data can be perfectly fitted by a single noiseless ‘teacher’ model. In contrast, we adopt a \emph{statistical} framework, focusing on the generalization loss. Such analysis allows adopting a closed-form characterization of the generalization loss that naturally accommodates multiple teachers and label noise, thereby better capturing real-world scenarios. Moreover, our framework yields \emph{exact} generalization expressions, whereas optimization-focused works only allow worst-case upper and lower bounds. Such expressions better align with empirical experiments, highlighting the explanatory role of this framework \citep{lin2023theory,karpel2026optimal}. Building on this framework, in this paper we investigate several remaining mysteries in continual learning regression. Our contributions are the following:
% In this paper, we extend this framework to examine several remaining mysteries in continual learning regression. 

% \begin{itemize}
%     \item What is the relation between the implicit regularization of using Gradient Descent (GD) for a finite number of steps, and fully optimizing the L2-penalized objective?
%     \item When can per-parameter regularization terms be useful, and when is isotropic regularization enough?
%     \item What is the optimal L2 regularization in settings when the teachers are not i.i.d., but correlated between tasks?
% \end{itemize}

% Specifically, in this setting,
\begin{itemize}
    \item We derive an exact closed-form generalization expression for early-stopped gradient descent that accommodates arbitrary teachers and label noise, and show that the optimal step decays as $\ln T/T$, for a single teacher and, more generally, i.i.d.~teachers. 
    \item We identify a teacher regime under which the optimal step size need not scale with $T$.
    \item We derive the exact generalization loss under task-dependent diagonal regularization in the single-teacher setting, and characterize when isotropic regularization remains optimal.
    \item We empirically validate our theoretical findings on continual linear models and neural networks, using both synthetic and \mnist-based data.
\end{itemize}

\vspace{-0.5em}

\section{Setting: regularized continual learning with implicit or explicit regularization}

\paragraph{Learning scheme.} We study the continual linear regression setting, which has been widely studied in recent years \citep{evron2022catastrophic,evron2025continual,zhao2024statistical,zhu2025global,levinstein2026optimal}.

In our setting the learner is exposed to $T$ sequential tasks, that is, $(\X_1,\y_1), (\X_2,\y_2), \ldots, (\X_T, \y_T)$, where each $\X_t\in \mathbb R^{n\times d}, \y_t\in \mathbb R^n$.
% \footnote{For simplicity we consider a fixed number of $n$ samples for all tasks, though our analysis can extend to varying samples.} 
During task $t$, the learner observes only the data $(\X_t,\y_t)$ associated with that task and updates its predictor accordingly, without having access to data from \emph{any} other task.
To be capable of learning sequentially, many continual learning algorithms incorporate a regularization towards the last iterate. A prominent line of work proposes using the Fisher information to weight this regularization \citep{kirkpatrick2017overcoming,benzing2022unifying}.
However, subsequent studies have observed that a much simpler isotropic regularization performs comparably well in practice \citep{Hsu2018Baselines,lubana2022quadratic}.
Following recent theoretical works \citep{levinstein2026optimal,zhao2025highdimensional} we focus on L2 (isotropic) regularization.
\vspace{-0.5em}
{
\begin{algorithm}[H]
  \caption{Continual linear regression with L2 regularization\label{schm:regularized}}
\begin{algorithmic}
   \STATE {
   \textbf{\algmargin\bfseries Input:} $\{(\X_i,\y_i)\}_{i=1}^T$, regularization strength $\lambda$
   } 
   \STATE {\algmargin Initialize} 
   $\mathbf{w}_0\in \mathbb R^d$
   \STATE { \algmargin
   \textbf{For each} iteration $t=1,\dots,T$: 
   }
   \STATE {
   $\quad \displaystyle
        \mathbf w_t \gets
        {\arg\min}_{\mathbf w}\Bigl\{\frac{1}{2n}\Vert\mathbf X_{t}\mathbf w-\mathbf y_{t}\Vert^2
        +\frac{\lambda}{2} \Vert\mathbf w-\mathbf w_{t-1}\Vert^2\Bigr\}$
    }
   \STATE {\algmargin
   \textbf{Output:} $\mathbf{w}_{T}$
   } 
\end{algorithmic}
\end{algorithm}
\vspace{-1em}
}
Employing explicit regularization may become \emph{impractical} in high-dimensional settings or over long task horizons, as it incurs a high computational cost (as we have to fully minimize the objective in \ref{schm:regularized}). For this reason, we also study a gradient algorithm run for a \emph{finite} number of steps, which naturally induces a form of implicit regularization \citep{ali2019continuous,sonthalia2024regularization}, while maintaining a relative small cost. In line with recent theoretical works \citep{ding2024understanding,levinstein2026optimal}, we will mainly focus on this scheme and its relation to the optimal regularized L2 (isotropic) solution.

\begin{algorithm}[H]
   \caption{Continual linear regression with finite step budgets (Early-stopped GD)}
   \label{schm:early_stopped}
\begin{algorithmic}
    \STATE { 
    \algmargin\bfseries Input:}
    $\{(\X_i,\y_i)\}_{i=1}^{T}$,
    inner step count $N$, and step size $\gamma$
   \STATE {\algmargin Initialize} 
   $\mathbf{w}_0\in \mathbb R^d$
   \STATE {\algmargin For each task $t = 1,\dots,T$:}
   \STATE {\algmargin \codeindent Initialize $\mathbf{w}^{(0)} \leftarrow \mathbf{w}_{t-1}$}
   \STATE {\algmargin \codeindent For $s = 1,\dots,N$:}
   \STATE {\algmargin \codeindent\codeindent
   $\mathbf{w}^{(s)} 
   \leftarrow 
   \mathbf{w}^{(s-1)} - \gamma 
   \nabla 
   {\frac{1}{2n}} \left\Vert \X_{t} \mathbf{w}^{(s-1)} - \y_{t}\right\Vert^2$
   }
   \STATE {\algmargin \codeindent $\mathbf{w}_t \leftarrow \mathbf{w}^{(N)}$}
   \STATE {\algmargin Output $\mathbf{w}_T$}
\end{algorithmic}
\end{algorithm}
\vspace{-1em}

A widely used alternative that has been popular in recent empirical and theoretical works
\citep{kirkpatrick2017overcoming,schwarz2018progress,zhao2024statistical}
considers continual learning through a quadratic regularization that penalizes
deviations from the previous iterate as
$\Vert\mathbf w - \mathbf w_{t-1}\Vert^2_{\mathbf M_t}=(\mathbf w - \mathbf w_{t-1})^\top \mathbf M_t (\mathbf w - \mathbf w_{t-1})$,
where $\mathbf M_t \succeq 0$ is a regularization matrix.
This formulation generalizes the L2 regularization in \cref{schm:regularized}, which is recovered by setting $\mathbf M_t = \lambda \mathbf I$, it also encompasses regularization schemes such as Fisher-information methods.
\vspace{-0.5em}

\begin{algorithm}[H]
  \caption{Generalized L2 continual learning scheme\label{schm:generalized}}
\begin{algorithmic}
   \STATE {\algmargin
   \textbf{Input:} $\{(\X_i,\y_i)\}_{i=1}^T$, regularization matrices $\{\mathbf M_t\}_{t=1}^{T}$
   }   
   \STATE {\algmargin Initialize} 
   $\mathbf{w}_0\in \mathbb R^d$
   \STATE { \algmargin
   \textbf{For each} iteration $t=1,\dots,T$: 
   }
   \STATE {\algmargin
   $\quad \displaystyle
        \mathbf w_t \gets
        {\arg\min}_{\mathbf w}\Bigl\{\frac{1}{2n}\Vert\mathbf X_{t}\mathbf w-\mathbf y_{t}\Vert^2
        +\frac{1}{2}\Vert\mathbf w-\mathbf w_{t-1}\Vert^2_{\mathbf M_t}\Bigr\}$
    }
   \STATE {\algmargin
   \textbf{Output:} $\mathbf{w}_{T}$
   } 
\end{algorithmic}
\end{algorithm}
\vspace{-1em}
In practice, due to memory and computational constraints, continual learning systems typically employ diagonal regularization matrices to selectively penalize deviations from previously learned predictors \citep{kirkpatrick2017overcoming,schwarz2018progress,aljundi2018memory,chaudhry2018riemannian}, rather than maintaining a full $d\times d$ regularization matrix. Indeed, a full matrix can, in principle, eliminate forgetting entirely; see \citet[Proposition 5.5]{evron2023classification}.
%, rendering the problem essentially trivial. 
% As a result, in our work we consider diagonal regularization matrices.
\section{Setup: high dimensional data}
\paragraph{Data model.}
\label{sec:statistical_setting}
In this paper, we consider a high-dimensional random design with independent feature matrices $\X_1,\ldots,\X_T$, where each $\X_t\in \mathbb R^{n \times d}$ has i.i.d.\ entries with zero mean, and variance $\featnoise>0$ with finite $4+\epsilon$ moment for some $\epsilon>0$, and an arbitrary initialization $\mathbf{w}_0\triangleq\mathbf w^\star_0$. Moreover, the problem dimensions grow to infinity at a proportional rate $ \alpha:=\frac{n}{d}\in (0,\infty)$.

For each task $t\in [1, T]\triangleq\left\{1,\dots,T\right\}$, labels are generated from a corresponding teacher vector $\mathbf w^\star_t\in\mathbb{R}^{d}$ using a noisy linear model. That is,
$$
\y_{t}=\X_{t}\mathbf w^\star_t+\z_{t}
,
$$
where $\z_t$ is a label noise vector, with
$\mathbb{E}[\z_t]=0$ and $\text{Cov}(\z_t)=\labelnoise \I$ for some noise variance $\labelnoise \geq 0$, sampled independently of all tasks.
To ensure a well-behaved high-dimensional limit, we assume that the teacher
correlations converge to finite deterministic limits and are uniformly
integrable, i.e.
$$\left(\mathbf w^\star_i-\mathbf w^\star_j\right)^\top\left(\mathbf w^\star_k-\mathbf w^\star_l\right) \xrightarrow[d\to\infty]{} {}r_{i,j,k,l} ,\qquad \forall
i,j,k,l\in[0,T]  .$$

% To facilitate the analysis, we adopt a standard assumption in theoretical continual learning in which all tasks share a single predictor, learned using a noisy linear model \cite{li2023fixed,zhao2025highdimensional}. This assumption remains more permissive than continual learning works which typically assume joint realizability across \emph{training} sets~\cite{levinstein2026optimal,tsipory2026greedy,goldfarb2024theJointEffect,evron2022catastrophic,jung2025convergence}.

\paragraph{Evaluation metrics.} Following recent theoretical work in continual learning \citep{evron2022catastrophic,levinstein2026optimal,evron2025continual}, we evaluate the model error using the mean-squared-error (MSE) loss, that is, $\frac{1}{n}\left\Vert\X_i\mathbf{w} - \y_i\right\Vert^2$, however, under our data model, averaging over independent test data and label noise, gives,
\begin{equation}
\label{eq:expected_loss}
\begin{aligned}
&\frac{1}{n}\mathbb{E}
\big\Vert\X_{i}\mathbf{w} - \y_{i}\big\Vert^2
= 
\frac{1}{n}\mathbb{E}\big\Vert\X_{i}(\mathbf{w}-\mathbf{w}_{i}^\star) - \z_{i}\big\Vert^2 
%\\&
= 
\featnoise  
\norm{\mathbf{w} - \mathbf{w}_{i}^{\star}}^{2}
+
\labelnoise.
\end{aligned}
\end{equation}
Thus, the dependence of the expected loss on the predictor $\mathbf{w}$ is entirely captured by the distance $\left\Vert\mathbf{w}-\mathbf{w}_i^\star\right\Vert$. Consequently, following recent works \citep{lin2023theory,zhao2024statistical,karpel2026optimal}, we measure performance through the distance between the predictor $\mathbf{w}$
and the corresponding teacher $\mathbf{w}_i^\star$.
This reduction allows us to formally define our metrics of interest.

\begin{definition}[Task loss]
\label{def:loss}
For any predictor $\mathbf{w} \in \reals^d$, the loss of the predictor in task $i\in[1, T]$ is defined as the distance between the predictor and the teacher of the associated task, that is,
$$\mathcal{L}_i(\mathbf{w})=\left\Vert \mathbf{w}-\mathbf w^\star_i\right\Vert ^{2}.$$
\end{definition}
\begin{definition}[Generalization error]
\label{def:gen}
At iteration $t$, the generalization error is defined as the average task loss of the predictor $\mathbf{w}_t \in \reals^d$ over all tasks observed up to time $t$, that is,
$$
G_t\triangleq\frac{1}{t}\sum_{i=1}^{t} \mathcal{L}_i(\mathbf{w}_t)=\frac{1}{t}\sum_{i=1}^{t}\left\Vert\mathbf{w}_t-\mathbf w^\star_i\right\Vert^{2}.$$
\end{definition}
\vspace{-1em}

% \paragraph{Model Implications.}
% High dimensionality enables the use of explicit closed-form expressions based on the Marchenko–Pastur law \cite{marchenko1967pastur} and recent refinements thereof \cite{silverstein1995empirical,liao2023rmtml,zhu2025global}.
\vspace{-1em}
\paragraph{Notations.} Throughout the paper, we use $\xrightarrow[]{P}$ for convergence in probability. 
We let $G^{\mathrm{ridge}}_T$, $G^{\mathrm{esgd}}_T$, and $G^{\mathrm{gen}}_T$ denote the generalization loss after $T$ iterations by Algorithms~\ref{schm:regularized}, \ref{schm:early_stopped}, and \ref{schm:generalized}, respectively.
% \paragraph{Marchenko-Pastur law} Consider $\X\in \reals^{n\times d}$ with i.i.d. entries having zero mean, variance $\featnoise$, and consider the sample covariance matrix $\mathbf W = \frac{1}{n}\X^\top\X$. The results of the Marchenko-Pastur law apply when the dimensions increase to infinity while the ratio converges to a constant, i.e., $n,d\to \infty$ and $\frac{n}{d}\to \ratio \in (0,\infty)$. Then, almost surely, the empirical spectral distribution of $\mathbf W$ converges weakly to the Marchenko--Pastur distribution with aspect ratio $\ratio$, whose density is given by
% \begin{align*}
% &f_{\ratio}(x)=\frac{\ratio}{2\pi x\featnoise}\sqrt{(\ratio_+-x)(x-\ratio_-)} \mathbf 1_{[\ratio_-,\ratio_+]}(x)
% \end{align*}
% with $\ratio_{\pm}=\featnoise\left(1\pm \sqrt\frac{1}{\ratio}\right)^2$. If $\ratio<1$, then an additional point mass of size $1-\ratio$ appears at zero.

\section{Early-stopped GD generalization loss and the optimal step-size}
\label{sec:1}
In this section, we derive an \emph{exact} closed-form expression for the generalization loss attained by the early-stopped GD algorithm (\cref{schm:early_stopped}). Through this section, we consider only the step sizes that ensure the convergence of the early-stopped GD algorithm, that is, $\gamma<\frac{2}{{\lambda_{\text{max}}}(\frac{1}{n}\X^\top\X)}$, for generic $\X$ following the statistical setting described in \cref{sec:statistical_setting}. Our key technical result is the following.
\begin{theorem}[Generalization loss]
\label{thm:closedloss}
Recall the high-dimensional setting where the number of samples and the dimension grow at proportional rate $\alpha:=\frac{n}{d}$. Moreover, recall that $\featnoise,\labelnoise$ are the variances of features and labels, respectively. Then, the generalization loss after $T$ tasks of the early-stopped GD algorithm (\cref{schm:early_stopped}) satisfies, as $n, d\to\infty$,
\vspace{-0.2em}
\[
\begin{aligned}
G_{T}^{\mathrm{esgd}}&\xrightarrow{P}\labelnoise c'\frac{1-a'^{T}}{1-a'}+\frac{1}{T}\sum_{i=1}^{T}\left\Vert \mathbf{w}_{T}^{\star}-\mathbf{w}_{i}^{\star}\right\Vert ^{2}+2\big(\mathbf{w}_{T}^{\star}-\bar{\mathbf{w}}^{\star}\big)^{\top}\sum_{i=1}^{T}b'^{T-i+1}\big(\mathbf{w}_{i-1}^{\star}-\mathbf{w}_{i}^{\star}\big)\\&\quad+\sum_{i=1}^{T}\sum_{i'=1}^{T}a'^{T-\max(i,i')+1}b'^{|i-i'|}\big(\mathbf{w}_{i-1}^{\star}-\mathbf{w}_{i}^{\star}\big)^{\top}\big(\mathbf{w}_{i'-1}^{\star}-\mathbf{w}_{i'}^{\star}\big),
\end{aligned}
\]
where we denote $\mathbf w^\star_0=\mathbf{w}_0$ and $\bar {\mathbf w}^\star=\frac{1}{T}\sum_i \mathbf w^\star_i$,

\begin{equation}
\begin{aligned}
&a'=1-2N\featnoise\gamma
+O\big((N\gamma)^{2}\big),\qquad
b'=1-N\featnoise\gamma
+O\big((N\gamma)^{2}\big),
\\&\hspace{10em}
c'
=
\frac{N^{2}}{\ratio}\featnoise\gamma^{2}
+O\big((N\gamma)^{3}\big).
\end{aligned}
\end{equation}

\end{theorem}
The proof of \cref{thm:closedloss}, along with the remaining proofs in this section, is provided in Appendix~\ref{pr:sec1}, where we also give the full expressions for $a'$, $b'$, and $c'$ (Eqs.
(\ref{eq:13}), (\ref{eq:35}), (\ref{eq:36}), respectively). To aid interpretation, we include a numerical validation of the predictions of \cref{thm:closedloss} in Appendix \ref{sec:val1}.

Thus far, the exact generalization prediction in \cref{thm:closedloss} has accommodated \emph{arbitrary} teacher setup. To obtain more meaningful results, following prior works \citep{zhu2025global,karpel2026optimal}, we now consider teachers that are sampled i.i.d. from a common distribution.

\begin{assumption}[i.i.d.~teachers]
\label{assump:iid_teachers}
Considering the statistical setting of \cref{sec:statistical_setting}, the teachers
$\mathbf{w}_{1}^{\star},\dots,\mathbf{w}_{T}^{\star}$ are drawn i.i.d. from a distribution with a mean $\mathbf{w}^{\star}$ and covariance $\mathbf{\Sigma}$ where $\Vert\mathbf{\Sigma}\Vert_*\leq C<\infty$, independently of all other random variables (features and labels).
\end{assumption}

This assumption yet remains more permissive than theoretical works in continual learning which typically assume a single teacher \citep{li2023fixed,zhao2025highdimensional} or joint realizability across \emph{training} sets \citep{evron2022catastrophic,evron2025continual,levinstein2026optimal,tsipory2026greedy}.

A theory-guided selection of the step size $\gamma$
and the number of gradient steps $N$ is crucial in settings where exhaustive tuning is impractical or infeasible, as in continual learning. This motivates the need for principled tuning rules. Such a rule follows from our next result, which we prove in Appendix \ref{pr:opt}, by characterizing the optimal parameters that minimize the generalization loss of \cref{thm:closedloss}.
\begin{theorem}[Optimal step size]
\label{thm:opt}
Under the i.i.d.~teacher setting (\cref{assump:iid_teachers}) with a nontrivial mean teacher $\mathbf w^\star\neq\mathbf{w}_0$\footnote{Having $\mathbf w^\star=\mathbf{w}_0$ would directly imply optimal step size $\gamma^\star\to0$.}, as $n,d\to\infty$ the optimal fixed step size $\gamma^\star$ that minimizes the generalization loss of the early-stopped GD algorithm (\cref{schm:early_stopped}) after $T$ tasks satisfies $\gamma^\star\asymp \frac{\ln T}{NT}$. Specifically, for all sufficiently large $T>T_0=\Theta\big(\frac{1}{\epsilon^2\ln^2(1/\epsilon)}\big)$\footnote{We take the limits sequentially: first we take $n,d\to\infty$, and only then we let $T$ grow.},
the optimal parameters satisfy
\begin{align*}   
\frac{1}{1+\epsilon}\frac{\ln\left(T\right)+\ln\left(\Gamma\right)}{2N\featnoise T}
<\gamma^\star
<\frac{1}{1-\epsilon}\frac{\ln\left(T\right)+\ln\left(\Gamma\right)}{2N\featnoise T},
\end{align*}
where $\Gamma=\frac{4\ratio\featnoise\left\Vert \mathbf{w}^{\star}-\mathbf{w}_0\right\Vert^{2}}{v_{z}+\featnoise\mathrm{Tr}(\boldsymbol{\Sigma})\left(\ratio+1\right)}$. In the degenerate noiseless case of $\mathrm{Tr}(\boldsymbol\Sigma)=\labelnoise=0$, the optimal choice is any $\gamma^\star\in \big(0,2\featnoise^{-1}\left(1+{\alpha}^{-1/2}\right)^{-2}\big)$ and $N\to\infty$.

Moreover, if $\gamma\asymp \frac{\ln T}{NT}$ the resulting generalization loss of the early-stopped GD matches that of optimally tuned L2 regularization up to $O(1/T)$ factor.
\end{theorem}

\paragraph{Experiments on Synthetic Regression Data.}
We empirically validate the theoretical predictions of \cref{thm:opt} using finite dimensional data. For each task $i\in[1,T]$ we sample a teacher \linebreak $\mathbf{w}_i^\star\sim \mathcal N(\mathbf w^\star, \frac{1}{d}\I)$. We use $d=80$ and $n=50$ as the dimensions, and we sample $X_{i,j}\sim \mathcal N(0, 1)$. We fix the number of gradient steps as $N=20$.
\begin{figure}[H]
    \centering
    \begin{subfigure}{.49\columnwidth}
        \centering
        \includegraphics[width=.88\linewidth]{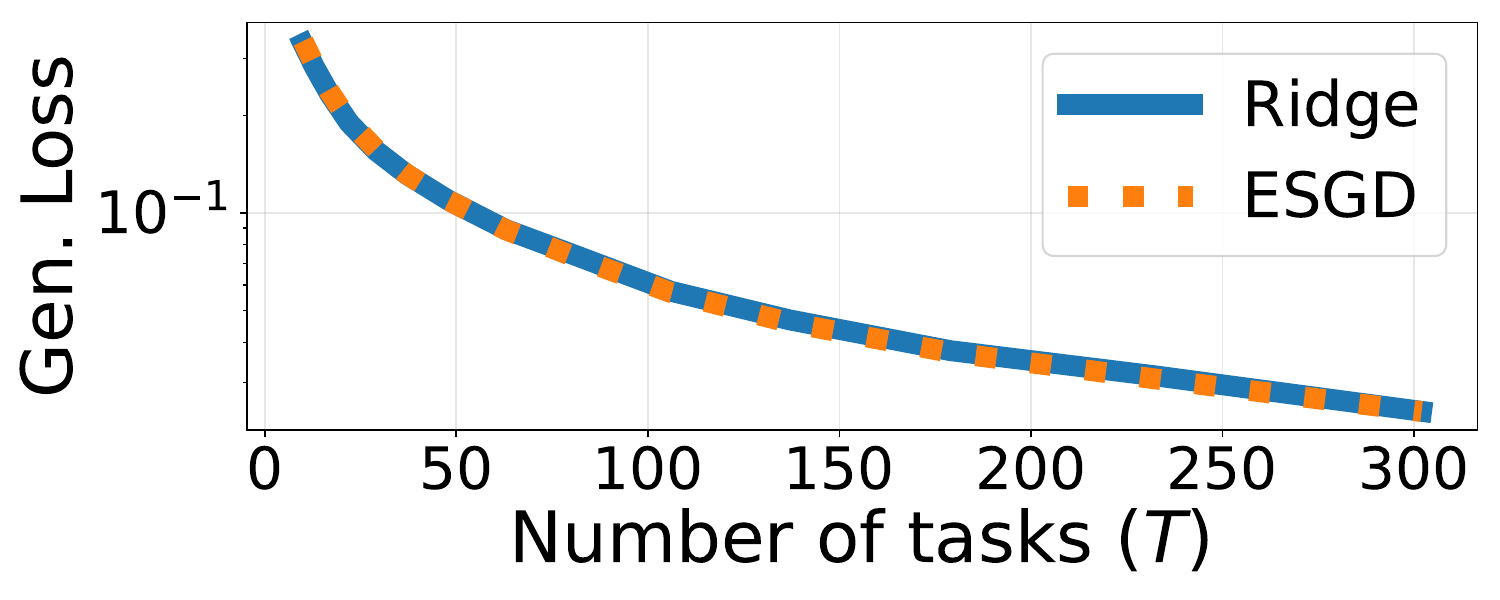}
        \caption{Under optimal tuning, L2 regularization and early-stopped GD coincide as predicted by \cref{thm:opt}.}
    \end{subfigure}
    \hfill
    \begin{subfigure}{.49\columnwidth}
        \centering
        \includegraphics[width=.88\linewidth]{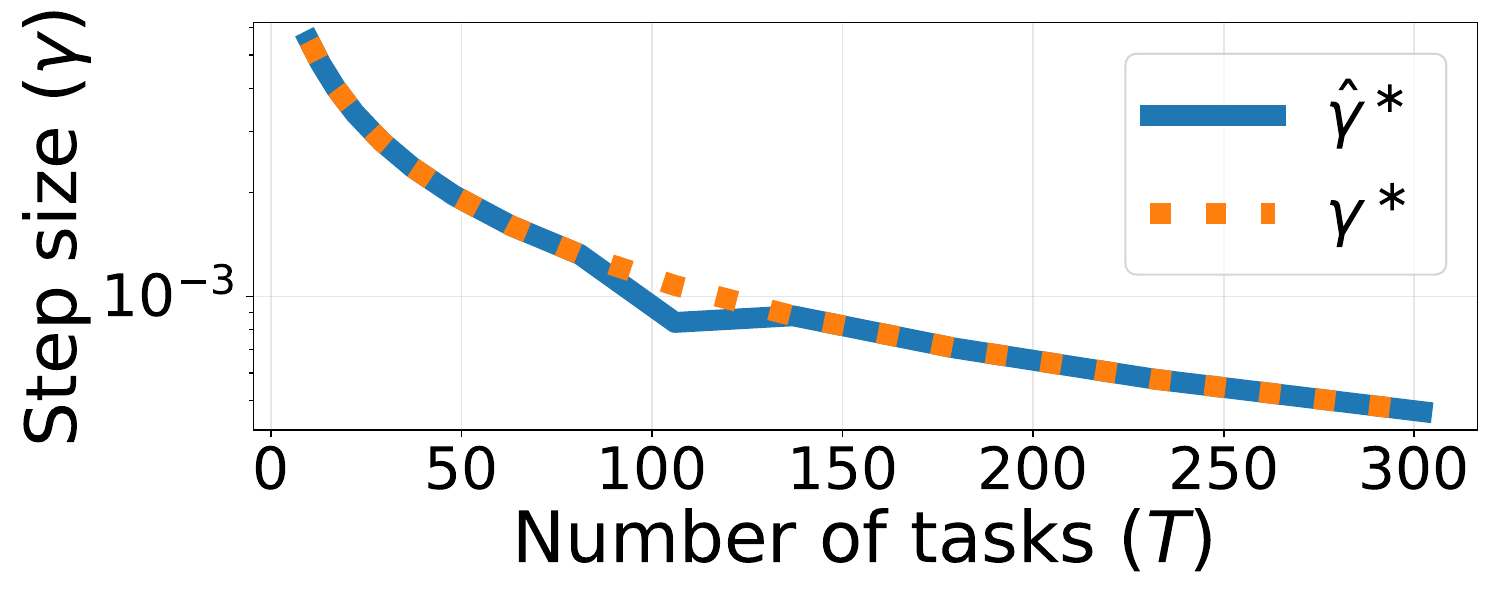}
        \caption{Optimal step size strength decays as predicted by \cref{thm:opt}.}
    \end{subfigure}
    \caption{The plots provide empirical evidence for the \cref{thm:opt}. Here, we set $\labelnoise=1$; $\hat\gamma^\star$ denotes the empirically optimal value, while $\gamma^\star$ is the theoretical prediction by \cref{thm:opt}. We compute the empirical optimal $\hat \gamma^\star$, $\hat \lambda^\star$ with an exhaustive search. Empirical curves are averaged over 50 runs.}
\end{figure}
Up to this point, \cref{thm:opt} yields a practical rule of thumb for choosing the number of gradient steps $N$ and the step size $\gamma$ in which we find that $\gamma^\star \asymp\frac{\ln T}{NT}$. In practice, this allows one to tune the parameter on tasks with smaller task horizons to identify this scaling, and then extrapolate it to larger task horizons, for which direct tuning would be computationally infeasible. For example, for $T = 1000$, one can tune at $T = 10$ and then set $\gamma(1000)=\frac{10\ln(1000)}{1000\ln(10)}\hat\gamma^\star(10)\approx 0.03\hat\gamma^\star(10)$. Moreover, under this choice, the resulting generalization loss remains close to the one attained with optimally tuned L2 regularization while achieving substantially lower computational cost. This result quantifies the close connection between early-stopped GD and optimally tuned isotropic regularization, which has been shown to perform comparably well to standard baselines in continual learning \citep{Hsu2018Baselines,smith2023closer}. At the same time, many continual learning methods introduce more expressive diagonal-based regularization terms to selectively penalize deviations from previously learned predictors \citep{kirkpatrick2017overcoming,schwarz2018progress,aljundi2018memory}, this naturally raises the question of when such simple isotropic regularization is sufficient. To capture this class of methods, we next study ridge regression with a diagonal regularization matrix.

\section{Explicit diagonal regularization: when is it better?}

\label{sec:2}
In this section, we consider the generalized ridge regression algorithm (\cref{schm:generalized}) with arbitrary positive-definite diagonal regularization matrices $\{\mathbf M_t\}_{t=1}^T$ (spectrally bounded\footnote{i.e., there exist constants $0<l\le u$ independent of $n,d$ such that $l\mathbf I\preceq \mathbf M_t\preceq u\mathbf I$ for all $1\leq t\leq T$.}). Importantly, this formulation already captures the core structure of the algorithms that are widely used in practice \citep{kirkpatrick2017overcoming,zenke2017continual,chaudhry2018riemannian}. Our goal is to derive an \emph{exact} generalization analysis, under the following technical assumption. 
\begin{assumption}[Realizability]
\label{assum:1}
Considering the statistical setting of \cref{sec:statistical_setting}, we assume a single teacher across all tasks, that is, $$\mathbf w^\star_1=\mathbf w^\star_2=\ldots=\mathbf w^\star_T=\mathbf w^\star,$$ drawn independently of all random variables (features and labels).
\end{assumption}
The single-teacher assumption is standard in prior theoretical work \citep{li2023fixed,ding2024understanding,zhao2025highdimensional}. In our setting, it is primarily adopted for analytical simplicity rather than as a restriction, and our framework naturally extends to general (task-varying) teachers.
This leads to the following closed-form characterization of the generalization loss.

\begin{theorem}[Generalization loss]
\label{thm:4}
Under \cref{assum:1}, the generalization loss after T tasks of the generalized ridge algorithm (\cref{schm:generalized}), specialized to a fixed regularization matrix $\mathbf M$, satisfies, as $n,d\to\infty$,  
\begin{align*}
&G_{T}^{\mathrm{gen}}\xrightarrow{P}\left(\mathbf{w}_{0}-\mathbf{w}^{\star}\right)^{\top}\left(\bar{\mathbf{Q}}^{2T}+\featnoise\sum_{k=1}^{T}\bar{\mathbf{Q}}^{2k}\tilde{\mu}\tilde{\rho}_{k}\right)\left(\mathbf{w}_{0}-\mathbf{w}^{\star}\right)\\&+\frac{\labelnoise}{n}\mathrm{Tr}\Bigg[\sum_{k=1}^{T}\left[\mathbf{I}-\bar{\mathbf{Q}}\left[\mathbf{I}+\frac{\featnoise^{2}}{n}\mathrm{Tr}\left[\mathbf{M}^{-1}\bar{\mathbf{Q}}^{2}\right]\tilde{\mu}\mathbf{M}^{-1}\right]\right]\left[\bar{\mathbf{Q}}^{2\left(T-k\right)}+\featnoise\sum_{i=k+1}^{T}\bar{\mathbf{Q}}^{2\left(i-k\right)}\tilde{\mu}\tilde{\rho}_{i}\right]\mathbf{M}^{-1}\bar{\mathbf{Q}}\Bigg]
\end{align*}
where $\bar{\mathbf Q}=\big(\mathbf I+\featnoise \tilde m_p(-1)\mathbf M^{-1}\big)^{-1},$
with $\tilde m_p(-1)$ being the unique positive solution of
$$
\frac{1}{\tilde m_p(-1)}=1+\frac{1}{n}\sum_{i=1}^d\frac{\featnoise  M_{i,i}^{-1}}{1+\tilde m_p(-1)\featnoise M_{i,i}^{-1}},
$$
and,\ $\hspace{7em}\begin{aligned}[t]
\tilde{\rho}_{k}&=\frac{\featnoise}{n}\mathrm{Tr}\Big[\mathbf{M}^{-2}\Big(\bar{\mathbf{Q}}^{2(T-k+1)}+\featnoise\tilde{\mu}\sum_{i=k+1}^{T}\tilde{\rho}_{i}\bar{\mathbf{Q}}^{2(i-k+1)}\Big)\Big],\\&\tilde{\mu}=\left[\left[1+\frac{\featnoise}{n}\mathrm{Tr}\left(\mathbf{M}^{-1}\bar{\mathbf{Q}}\right)\right]^{2}-\frac{\featnoise^2}{n}\mathrm{Tr}\left(\mathbf{M}^{-1}\bar{\mathbf{Q}}\right)^{2}\right]^{-1}.
\end{aligned}$
\end{theorem}
The closed-form expression to task dependent diagonal regularization matrices is given in \cref{eq:46}, along with the remaining proofs of this section, which appear in Appendix \ref{pr:sec5}.
To aid interpretation, we include a numerical validation of the predictions of \cref{thm:4} in Appendix \ref{sec:val4}.

We now exhibit a simple setting in which scalar (isotropic) regularization provably fails.

\begin{theorem}[Failure of isotropic ridge under sparse teachers]
\label{thm:sparse}
Assume that $\alpha:=\frac{n}{d}<1$. Let $\mathbf{w}_1^\star-\mathbf{w}_0$ be $k$-sparse with support $S$ so that $\mathbf{w}_1^\star\neq\mathbf{w}_0$, and assume that $\mathbf X_S$ (the submatrix of $\mathbf X$ consisting of the columns indexed by $S$) has full column rank, then as $n,d\to \infty$ after a single task with no noise, it holds almost surely that,
$$
\sup_{\mathbf M\in \mathcal M}\inf_{\lambda>0} \frac{G_1^{\mathrm{ridge}}}{G_1^{\mathrm{gen}}}=\infty.$$
where $\mathcal M=\bigcup_{0<l\leq u<\infty}
\left\{\mathbf M:l\mathbf I_d\preceq\mathbf M\preceq u\mathbf I_d\right\}$ and $u,l$ are constants independent of $n,d$, i.e. $\mathcal M$ is the family of all spectrally bounded matrices (which may depend on the support $S$).
\end{theorem}
The assumption that $\mathbf X_S$ has a full rank column is easy to satisfy: if $k$ is fixed (or $k=o(d)$) and the entries $X_{i,j}$ are drawn from a continuous distribution, then almost surely $\mathbf X_S$ has a full rank column. Yet, it remains unclear under what conditions isotropic regularization remains optimal. Interestingly, in Appendix \ref{sec:nspr} we find that the \emph{necessary and sufficient} condition is that, for each coordinate, the teacher is equally far from the initialization, either deterministically, or, for a random teacher, with probability tending to $1$ as $n,d \to\infty$ (sufficiency is straightforward, necessity is less trivial).

So far, our analysis characterizes how to optimally tune the parameters of early-stopped GD under the i.i.d.~teacher setting, showing that the resulting generalization loss is asymptotically equivalent to that of optimally tuned isotropic regularization. Additionally, in the single teacher setting, our analysis identifies when such isotropic regularization remains optimal within the broader class of diagonal regularization schemes which are widely used in practice. Beyond the i.i.d. setting, a natural question is whether these conclusions persist under more structured, non-stationary task sequences.

\section{Beyond i.i.d.~teachers: optimal step-size need not scale with $T$}
\label{sec:3}
Finding the optimal parameter tuning in early-stopped GD had to make the simplifying assumption of i.i.d.\ teachers (\cref{assump:iid_teachers}). We now ask whether the optimal scaling identified by \cref{thm:opt} which states that $\gamma^\star \asymp \frac{\ln T}{NT}$ persists beyond the i.i.d.\ teacher regime. We address it via a counterexample.

\begin{assumption}
\label{assump:sin_teachers}
We assume a sinusoidal model across the teachers, that is, for all $i\in [1,T]$,
\begin{align*}
\mathbf w^\star_{i}=\mathbf A+\sin\left(2\pi fi\right)\cdot \mathbf B,
\end{align*}
for some $\mathbf A,\mathbf B\in \mathbb R^{d}$ and frequency $f\in \mathbb R$.
\end{assumption}

This type of model is a fundamental building block for representing and analyzing periodic structure in learning \citep{vaswani2017attention}, control systems \citep{aastrom2008feedback}, and signal processing \citep{cooley1965algorithm}. In continual learning, recent works have been studying the effect of repetition.
A few studies have proposed solutions that aim to exploit the repetition in the stream to learn continually \citep{liu2023periodically}. This has been demonstrated to significantly improve performance over standard task-split methods, emphasizing the effectiveness of cyclic strategies in continual learning settings.

Under the sinusoidal model (\cref{assump:sin_teachers}), the sequence of tasks is non-stationary, and the teachers do not concentrate around a single mean. As a result, the generalization measure as defined in \cref{def:gen}, which averages the loss with respect to all past teachers, is no longer informative.

Instead, we evaluate generalization with respect to teachers belonging to the same class (or phase) as the final task. Let $A \subseteq \{1,\dots,T\}$ denote the index set of such teachers. We define the generalization error at time $T$ as
$$
G_T
\triangleq\frac{1}{|A|}\sum_{i \in A}\left\lVert \mathbf{w}_T - \mathbf{w}_i^\star \right\rVert^2.
$$
However, when the task horizon grows to infinity, the generalization error may be asymptotically oscillatory.
We therefore analyze the \emph{worst-case} generalization loss, defined as
\begin{align*}
G_{\infty}
\triangleq
\limsup_{T\to\infty}
\frac{1}{|A_T|}
\sum_{i \in A_T}
\left\lVert \mathbf{w}_T - \mathbf{w}_i^\star \right\rVert^2,
\end{align*}
where $A_T \subseteq \{1,\dots,T\}$ denotes the set of teachers belonging to the same class as the task at time $T$.

\begin{theorem}[Generalization with infinite task horizon]
\label{thm:sin_closed}
Under \cref{assump:sin_teachers}, the generalization loss after $T$ tasks of the early-stopped GD algorithm (\cref{schm:early_stopped}) satisfies, as $n, d\to\infty$,
$$
G_{\infty}
\xrightarrow[]{P}
C_{0}
+
\sqrt{C_{1}^{2}+C_{2}^{2}}
\sup_{T \ge 1}
\cos\left(4\pi f T + \varphi\right)
$$
\end{theorem}

The proof, together with the explicit expressions for $C_{0}$, $C_{1}$, $C_{2}$ and $\varphi$, appears in Appendix~\ref{pr:sin} (see \cref{eq:40}).
We further note that when $f$ is irrational, then $\sup_{T \ge 1}\cos\left(4\pi f T + \varphi\right)=1$.

We empirically evaluate the optimal step size under a sinusoidal model.

\begin{figure}[H]
  \centering
  \includegraphics[width=0.94\linewidth]{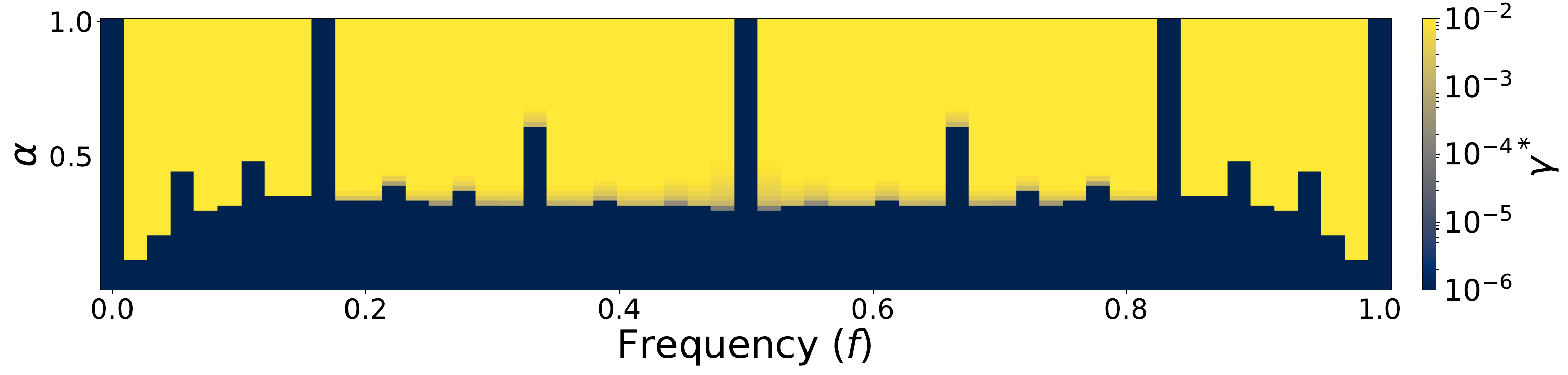}
  \caption{\textbf{The optimal step size remains finite as $T\to\infty$.} We plot the optimal step size $\gamma^\star$ of the objective $G_\infty$ as a function of the ratio ($\ratio$) and the frequency ($f$), fixing $\left\Vert\mathbf B \right\Vert=1$, $\labelnoise=1$, $\featnoise=1$, $d=100$ and $N=5$. We compute $\gamma^\star$ with an exhaustive search.}
  \label{fig:sin}
\end{figure}
\vspace{-1em}

Under the i.i.d.~teacher setting, our theoretical analysis (\cref{thm:opt}) shows that the optimal step size scales as $\ln T/T$, so we would expect $\gamma^\star \to 0$ in this regime. By contrast, as illustrated in \cref{fig:sin}, for the sinusoidal teacher model the optimal step size also remains finite as $T \to \infty$.
% To better understand the trends observed in the figure, we now discuss two representative regimes and their implications for the optimal regularization.

% \paragraph{Special case.}
% For $f \in \{0,\tfrac{1}{2},1\}$, the periodic teacher collapses to a 
% \emph{single-teacher} model. As identified in~\citet[Theorem 6]{karpel2026optimal},
% in a noisy regime (i.e., $\labelnoise > 0$), the optimal regularization becomes $\lambda^\star \to \infty$ as $T\to \infty$.

% \paragraph{More information implies weaker regularization.}
% If $\ratio \to 0$ the current task holds only limited information, therefore, higher regularization is needed.
% As $\ratio$ increases, data is aligned with the current teacher, and the optimal regularization weakens, as observed in \cref{fig:sin}.

\begin{figure}[!t]
\label{fig:exp}
\centering
% ----------------- Top row -----------------
\begin{subfigure}[t]{0.25\linewidth}
  \centering
  \includegraphics[width=\textwidth,keepaspectratio]{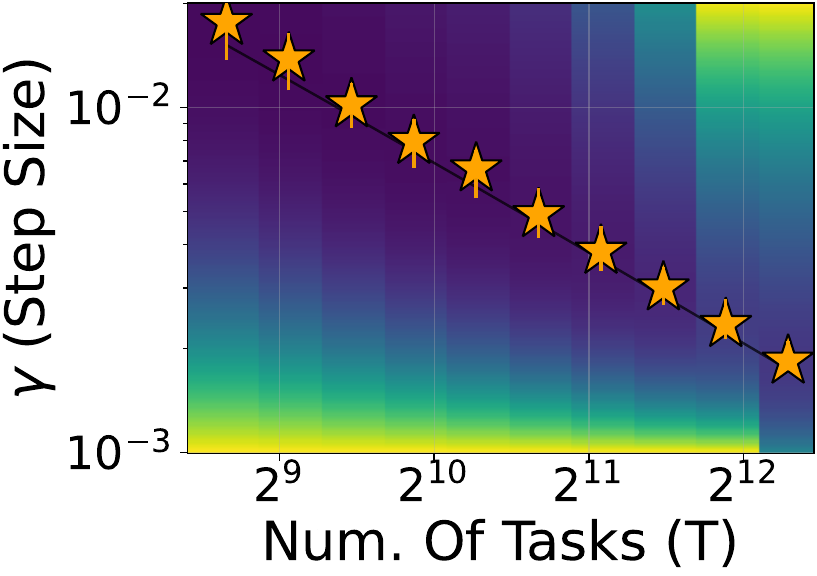}
  \caption{Single teacher (Linear)}
\label{fig:1}

\end{subfigure}\hfill
\begin{subfigure}[t]{0.23\linewidth}
  \centering
  \includegraphics[width=\textwidth,keepaspectratio]{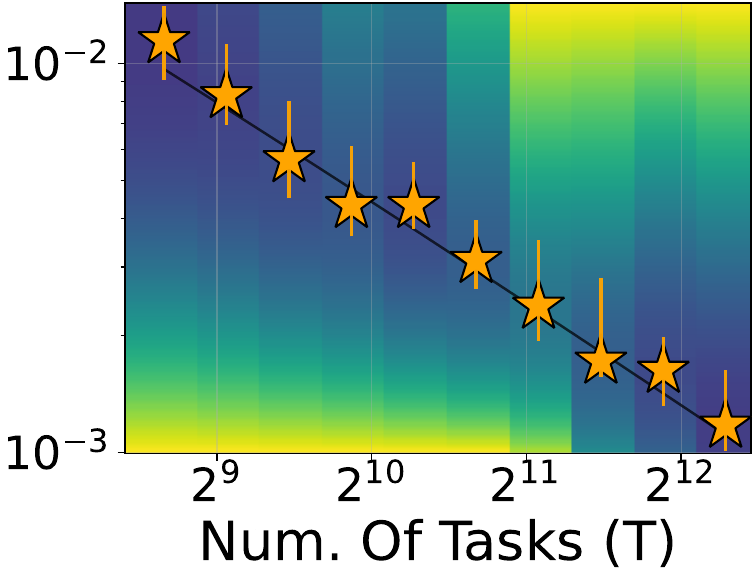}
  \caption{i.i.d.~teachers (Linear)}
\label{fig:2}

\end{subfigure}
% ----------------- Bottom row -----------------
\begin{subfigure}[t]{0.23\linewidth}
  \centering
  \includegraphics[width=\textwidth,keepaspectratio]{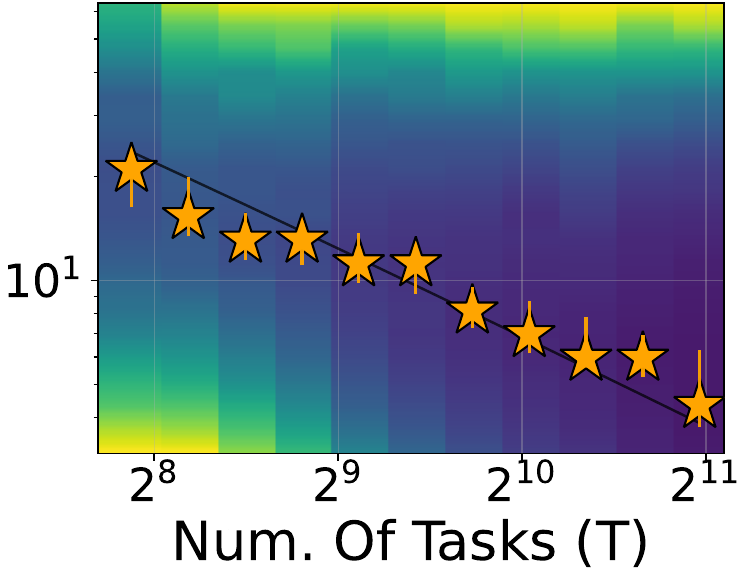}
  \caption{Single teacher (NN)}
\label{fig:3}

\end{subfigure}\hfill
\begin{subfigure}[t]{0.26\linewidth}
  \centering
  \includegraphics[width=\textwidth,keepaspectratio]{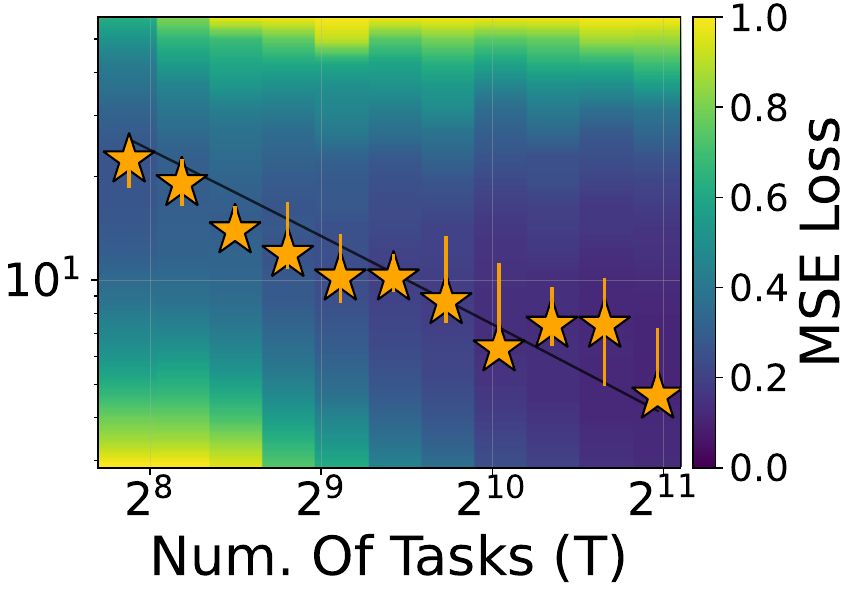}
  \caption{i.i.d.~teachers (NN)}
\label{fig:4}

\end{subfigure}
\caption{\textbf{Optimal step size scaling in linear models and neural networks (NN).} We plot the effect of step size on the mean squared error (MSE) after learning $T$ tasks. Curves are averaged over 5 and 15 random seeds for the linear model and neural network, respectively. The star marks the minimizer of the mean curve; vertical bars indicate the 95\% credible interval for the minimizer, and the black line indicates a fitted $\ln T/T$ scaling. The empirical MSE is normalized by the maximum MSE across all step sizes for the specific task horizon $T$. }

\end{figure}

\section{Empirical experiments on \mnist-data}
\label{sec:5}

While our analysis assumes i.i.d. features and a linear model, we empirically validate our theoretical predictions on a finite-dimensional \mnist-data, rather than high-dimensional i.i.d. features, where the assumptions of \cref{sec:statistical_setting} do not hold. We evaluate both linear models and ReLU networks. In all cases, we observe that the scaling $\gamma^\star \asymp \ln T/T$ as predicted by \cref{thm:opt} is preserved. 

\textbf{Experimental Setup.} We convert the \mnist data set labels into a binary classification problem where the digits 0-4 are converted to class 0 and 5-9 are converted to class 1. Moreover, we standardize the data to zero mean and unit variance.

\textbf{Teacher models.} In the experimental setup, we consider two types of teacher models.
\begin{enumerate}
    \item \textbf{A single teacher.} We train a single teacher using the same architecture as the evaluated model---in the linear regression, the single teacher is trained over a subset of 500 samples (to fit a perfect linear teacher), and in the neural networks the teacher is trained until convergence over all \mnist data. In this case, each task suffers from label noise.
    
    \item \textbf{i.i.d. noisy teachers.} Each task draws a teacher by adding Gaussian noise to the single teacher's weights. In addition to teacher noise, the labels also suffer from label noise.
\end{enumerate}

\textbf{Noisy Labels.} Each task is affected by additive Gaussian noise with variance $\labelnoise$ applied to the binary labels.
Specifically, we consider noise variances of $\labelnoise=0.1$ and, more challengingly, $\labelnoise=0.4$ for the linear and neural-network experiments, respectively.

\textbf{Evaluation Metric.} We evaluate generalization performance using the Mean Squared Error (MSE) computed over an independent test set.

\textbf{Linear Regression Experiments.}
We employ the early-stopped GD algorithm (\cref{schm:early_stopped}), trained by performing $N=10$ gradient steps on the mean-squared error loss, to predict \emph{noisy} labels from a single teacher (\cref{fig:1}) or multiple i.i.d.~teachers (\cref{fig:2}). Evidently, the predicted $\ln T / T$ scaling is preserved, consistent with \cref{thm:opt}, even in the absence of i.i.d. features.

\textbf{Neural Networks Experiments.}
We also consider a two-layer ReLU neural network ($784 \to 200\to 1$), with a sigmoid output for the teacher and a linear output for the student (to predict noisy labels), trained by performing $N=10$ gradient steps per task on the mean-squared error loss. Similarly, we repeat the previous experiments under both single-teacher (\cref{fig:3}) and multiple-teachers (\cref{fig:4}) settings. Notably, we again observe the same $\ln T / T$ scaling, which persists even in the substantially more complex NN model, matching the predicted scaling from \cref{thm:opt}.
\subsection{Diagonal regularization experiments}

We additionally experiment with the generalized L2 algorithm (\cref{schm:generalized}) in the linear regression setting. We consider two artificial teachers: a \emph{sparse} teacher with $100$ randomly selected coordinates equal to one and the remaining entries zero, and an \emph{isotropic} teacher with all coordinates equal to one.

We use the optimal diagonal regularizer obtained by optimizing the theoretical objective in \cref{thm:4}. Throughout this experiment, we set $\labelnoise = 0.4$, initialize $\mathbf w_0=\0$ and narrow the \mnist training data to $100$ samples for each task.
As observed, in the sparse-teacher setting the optimal diagonal regularizer substantially outperforms the scalar regularizer (\cref{fig:6}), in agreement with the prediction of \cref{thm:sparse}. In contrast, for an isotropic teacher (\cref{fig:5}), the scalar regularizer is optimal, as anticipated by the condition found in \cref{sec:2}, which requires the teacher vector to have coordinates equally far from the initialization.

\begin{figure}[H]
    \centering
    \begin{subfigure}[t]{0.415\linewidth}
        \centering
        \includegraphics[width=\linewidth]{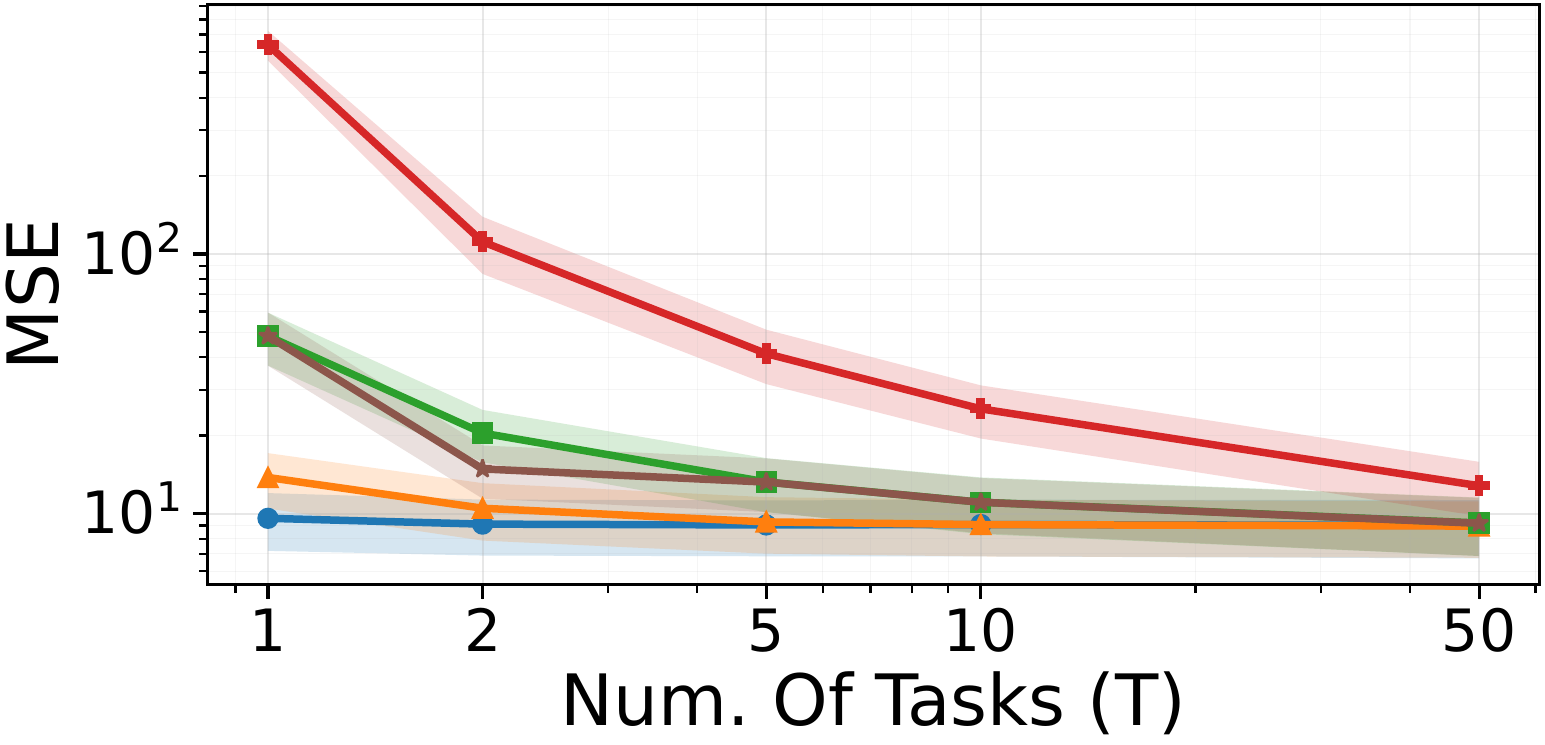}
        \caption{Isotropic teacher}
        \label{fig:m_reg_left}
        \label{fig:5}
    \end{subfigure}
    \hfill
    \begin{subfigure}[t]{0.5\linewidth} 
        \centering
        \includegraphics[width=\linewidth]{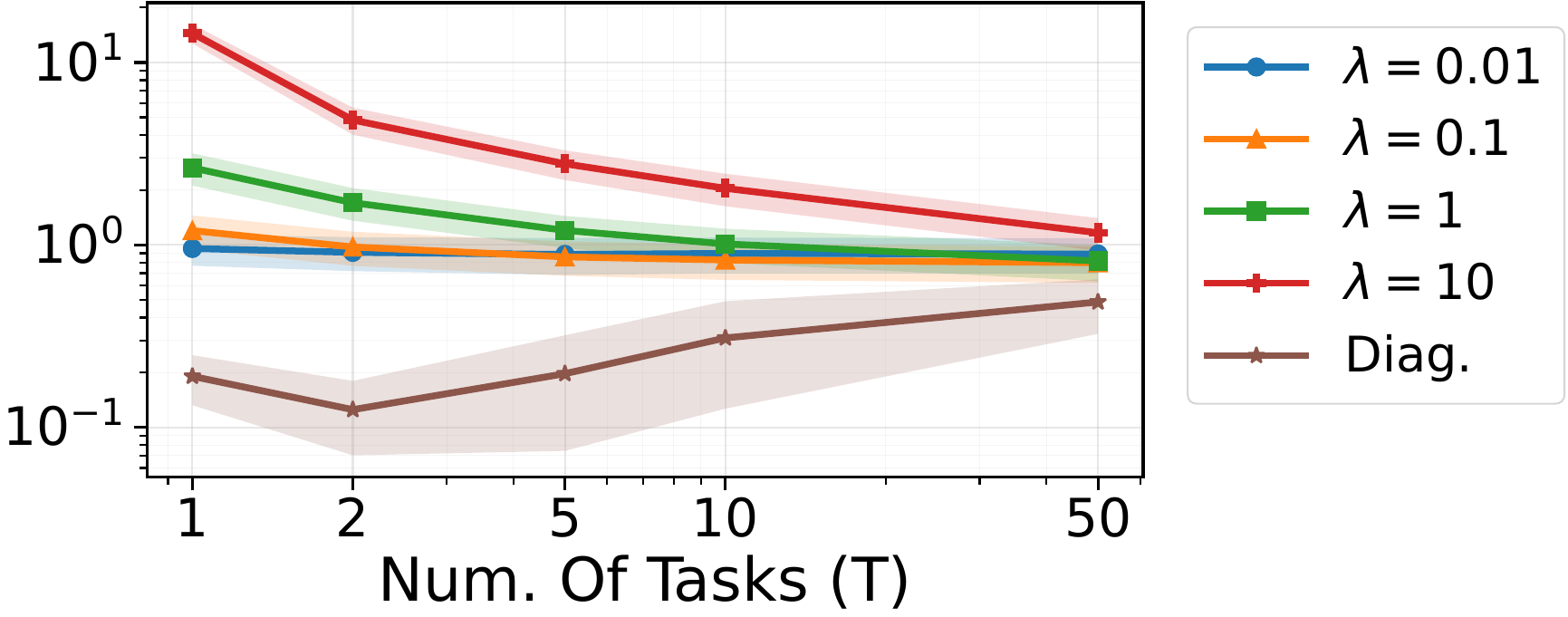}
        \caption{Sparse teacher}
        \label{fig:m_reg_right}
    \label{fig:6}
    \end{subfigure}

    \caption{\textbf{Optimal isotropic regularization conditions.} We plot the mean MSE with varying regularization strengths and the optimal diagonal regularizer obtained by optimizing the theoretical objective in \cref{thm:4}. Error bars denote the standard deviation across 50 random seeds.}
    \label{fig:m_reg_both}
\end{figure}
\vspace{-1em}

\section{Related work and discussion}
\label{sec:4}

\vspace{-0.5em}

\textbf{Summary.} In this paper, we studied continual linear regression under a simplified yet analytically tractable setting, assuming isotropic feature covariance, i.i.d. features, and either i.i.d. or single teacher. Despite these simplifying assumptions, our analysis yields meaningful theoretical insights that appear to persist beyond the linear regime. In particular, we identify a simple scaling law, $\gamma^\star \asymp \ln T/T$, which is supported both theoretically and empirically. This scaling provides a practical rule of thumb and significantly reduces the need for extensive hyperparameter tuning. Notably, we observe that this behavior extends beyond linear models, holding empirically even for neural networks. Next, we compare our results with existing work and discuss potential future directions.

\textbf{Implicit regularization.} \citet{jung2025convergence} study continual learning on multiple linear \emph{classification} tasks, where each task is trained with gradient descent for a fixed budget of iterations per task. They consider both cyclic and random task ordering, for which they proved the convergence of the training loss to zero. For cyclic ordering, they provided the non-asymptotic rate convergence, whereas for random task ordering, they provided the asymptotic loss convergence.
\citet{levinstein2026optimal} studied the optimal scaling of ridge regression and early-stopped GD in continual linear regression with arbitrary fixed features and a single noiseless teacher. Their analysis provides near-optimal \emph{worst-case} convergence guarantees. In particular, they consider a fixed step-size $\gamma \in (0, 1/R^2)$ (where $R$ is the data radius) and choose the iteration budget as $N=\frac{\ln \left(1-\frac{1}{\ln T}\right)}{\ln\left(1-\gamma R^2\right)}$, though it still remains unclear whether other values could perform better. Interestingly, in a similar single teacher noiseless setting our analysis dictates that the optimal number of gradient steps is $N\to\infty$. Only in the presence of label noise or teacher noise a finite optimal $N$ emerge. Indeed, in our setting, early-stopping slows convergence, yet it mitigates the asymptotic error arising from label noise or variation in the teachers.

\textbf{Regularization matrices.} \citet{zhao2024statistical,zhu2025global} construct a task-dependent regularization matrix from the empirical covariance at each task, which requires storing a $d \times d$ matrix. However, maintaining such a matrix is computationally expensive. In principle, this can eliminate forgetting; see \citet[Proposition 5.5]{evron2023classification}.
Moreover, their analysis is restricted to a specific design of regularization matrices. In contrast, we study arbitrary diagonal regularization matrix which is considered a typical choice in continual learning \citep{kirkpatrick2017overcoming,schwarz2018progress}.

% \paragraph{Statistical analysis and data model.} 
% So far, prior optimization-focused work \citep{evron2022catastrophic,evron2025continual,levinstein2026optimal} had to make a limiting assumption of joint realizability—i.e., that training data can be perfectly fitted by a single noiseless ‘teacher’ model. In contrast, our analysis derives a closed-form expression for the generalization loss of early-stopped gradient descent that naturally accommodates multiple teachers and label noise, which can more closely characterize real-world scenarios, and covers a wider family of setups. Our statistical analysis has an average-case nature—allowing \emph{exact} generalization expressions, while the optimization-focused body of work accommodates worst-case analysis that only allows upper and lower bounds. 

\textbf{Optimal scaling.} \citet{karpel2026optimal} showed that, under isotropic data and i.i.d.~teachers, the optimal L2 regularization minimizing the expected generalization loss in the overparameterized regime scales nearly linearly with the number of tasks. In this paper, we significantly strengthen and extend their results by considering both overparameterized and underparameterized regimes, and providing almost-sure or high-probability guarantees. Furthermore, our analysis shows that even when each task is trained for many gradient steps, the predicted scaling of the optimal L2 regularization emerges only after sufficiently many tasks (see \cref{fig:Gen_T}, in Appendix \ref{sec:exl2}). Finally, we show that isotropic regularization can fail altogether even under a simple setting and we derive necessary and sufficient conditions for its optimality within the broader class of diagonal regularizers that are practically used.

\textbf{Future directions.} Several natural directions remain open. Future work may relax the isotropy and independence assumptions by considering non-isotropic covariances, temporally correlated features, or non-i.i.d.~teacher sequences. Additional extensions include classification problems (e.g., logistic regression), nonlinear models, and more general non-stationary environments.

\subsubsection*{Acknowledgments}
The research of DS was funded by the European Union (ERC, A-B-C-Deep, 101039436). Views and opinions expressed are however those of the author only and do not necessarily reflect those of the European Union or the European Research Council Executive Agency (ERCEA). Neither the European Union nor the granting authority can be held responsible for them. The research of DS was also supported by the Council for Higher Education in Israel under the Moonshot Project.

% \subsection*{AI use statement}

% In this work, we used generative AI tools for polishing the writing within some of the proofs write-up.
% We have not used generative AI tools for helping develop theoretical models or conceptual frameworks, formulating mathematical claims, providing critical ingredients for proving mathematical claims, proposing or refining hypotheses, designing or providing feedback on research methodology or experiments, implementing methods, cleaning and generating synthetic data sets, cleaning and reformatting dataset, assisting with translation, supporting qualitative and thematic data analysis, are not applicable to this work.
% Additionally, we used generative AI tools for creating or editing software code. We have reviewed all AI-assisted work. Any LLM-generated code or text was verified and tested for correctness by the authors. We take responsibility for the final content of this work,
% including text, claims or artifacts produced with the aid of generative AI.

% \subsection*{Reproducibility statement}

% The required details to reproduce the main experimental results of the paper appear in \cref{sec:5}. For theoretical results, we provide clear explanations of any assumptions and a complete
% proof for each one of the claims, see \cref{sec:statistical_setting,sec:1,sec:2,sec:3}.

\bibliography{ICLR/iclr2027_conference}

@inproceedings{ali2019continuous,
  title={A continuous-time view of early stopping for least squares regression},
  author={Ali, Alnur and Kolter, J Zico and Tibshirani, Ryan J},
  booktitle={The 22nd international conference on artificial intelligence and statistics},
  pages={1370--1378},
  year={2019},
  organization={PMLR}
}

@article{kirkpatrick2017overcoming,
  title={Overcoming catastrophic forgetting in neural networks},
  author={Kirkpatrick, James and Pascanu, Razvan and Rabinowitz, Neil and Veness, Joel and Desjardins, Guillaume and Rusu, Andrei A and Milan, Kieran and Quan, John and Ramalho, Tiago and Grabska-Barwinska, Agnieszka and others},
  journal={Proceedings of the national academy of sciences},
  volume={114},
  number={13},
  pages={3521--3526},
  year={2017},
  publisher={National Acad Sciences}
}

@incollection{mccloskey1989catastrophic,
  title={Catastrophic interference in connectionist networks: The sequential learning problem},
  author={McCloskey, Michael and Cohen, Neal J},
  booktitle={Psychology of learning and motivation},
  volume={24},
  pages={109--165},
  year={1989},
  publisher={Elsevier}
}

@inproceedings{lubana2022quadratic,
  title={How do quadratic regularizers prevent catastrophic forgetting: The role of interpolation},
  author={Lubana, Ekdeep Singh and Trivedi, Puja and Koutra, Danai and Dick, Robert},
  booktitle={Conference on Lifelong Learning Agents},
  pages={819--837},
  year={2022},
  organization={PMLR}
}

@inproceedings{benzing2022unifying,
  title={Unifying importance based regularisation methods for continual learning},
  author={Benzing, Frederik},
  booktitle={International Conference on Artificial Intelligence and Statistics},
  pages={2372--2396},
  year={2022},
  organization={PMLR}
}

@inproceedings{evron2022catastrophic,
  title={How catastrophic can catastrophic forgetting be in linear regression?},
  author={Evron, Itay and Moroshko, Edward and Ward, Rachel and Srebro, Nathan and Soudry, Daniel},
  booktitle={Conference on Learning Theory (COLT)},
  pages={4028--4079},
  year={2022},
  organization={PMLR}
}

@inproceedings{zenke2017continual,
  title={Continual learning through synaptic intelligence},
  author={Zenke, Friedemann and Poole, Ben and Ganguli, Surya},
  booktitle={International Conference on Machine Learning},
  pages={3987--3995},
  year={2017},
  organization={PMLR}
}

@inproceedings{aljundi2018memory,
  title={Memory aware synapses: Learning what (not) to forget},
  author={Aljundi, Rahaf and Babiloni, Francesca and Elhoseiny, Mohamed and Rohrbach, Marcus and Tuytelaars, Tinne},
  booktitle={Proceedings of the European Conference on Computer Vision (ECCV)},
  pages={139--154},
  year={2018}
}

@InProceedings{evron2023classification,
  title = 	 {Continual Learning in Linear Classification on Separable Data},
  author =       {Evron, Itay and Moroshko, Edward and Buzaglo, Gon and Khriesh, Maroun and Marjieh, Badea and Srebro, Nathan and Soudry, Daniel},
  booktitle = 	 {Proceedings of the 40th International Conference on Machine Learning},
  pages = 	 {9440--9484},
  year = 	 {2023},
  volume = 	 {202},
  series = 	 {Proceedings of Machine Learning Research},
  month = 	 {23--29 Jul},
  publisher =    {PMLR},
}

@inproceedings{jung2025convergence,
title={Convergence and Implicit Bias of Gradient Descent on Continual Linear Classification},
author={Jung, Hyunji and Cho, Hanseul and Yun, Chulhee},
booktitle={The Thirteenth International Conference on Learning Representations},
year={2025},
}

@InProceedings{lin2023theory,
  title = 	 {Theory on Forgetting and Generalization of Continual Learning},
  author =       {Lin, Sen and Ju, Peizhong and Liang, Yingbin and Shroff, Ness},
  booktitle = 	 {Proceedings of the 40th International Conference on Machine Learning},
  pages = 	 {21078--21100},
  year = 	 {2023},
  volume = 	 {202},
  series = 	 {Proceedings of Machine Learning Research},
  month = 	 {23--29 Jul},
  publisher =    {PMLR},
}

@inproceedings{ding2024understanding,
title={Understanding Forgetting in Continual Learning with Linear Regression},
author={Meng Ding and Kaiyi Ji and Di Wang and Jinhui Xu},
booktitle={Forty-first International Conference on Machine Learning},
year={2024}
}

@inproceedings{zhao2024statistical,
title={A Statistical Theory of Regularization-Based Continual Learning},
author={Xuyang Zhao and Huiyuan Wang and Weiran Huang and Wei Lin},
booktitle={Forty-first International Conference on Machine Learning},
year={2024}
}

@inproceedings{hiratani2024disentangling,
title={Disentangling and mitigating the impact of task similarity for continual learning},
author={Naoki Hiratani},
booktitle={The Thirty-eighth Annual Conference on Neural Information Processing Systems},
year={2024}
}

@InProceedings{li2023fixed,
  title = 	 {Fixed Design Analysis of Regularization-Based Continual Learning},
  author =       {Li, Haoran and Wu, Jingfeng and Braverman, Vladimir},
  booktitle = 	 {Proceedings of The 2nd Conference on Lifelong Learning Agents},
  pages = 	 {513--533},
  year = 	 {2023},
  editor = 	 {Chandar, Sarath and Pascanu, Razvan and Sedghi, Hanie and Precup, Doina},
  volume = 	 {232},
  series = 	 {Proceedings of Machine Learning Research},
  month = 	 {22--25 Aug},
  publisher =    {PMLR},
}

@inproceedings{evron2025continual,
    title={From Continual Learning to {SGD} and Back: Better Rates for Continual Linear Models},
    author={Itay Evron and Ran Levinstein and Matan Schliserman and Uri Sherman and Tomer Koren and Daniel Soudry and Nathan Srebro},
    booktitle={37th International Conference on Algorithmic Learning Theory},
    year={2026},
}

@inproceedings{smith2023closer,
  title={A closer look at rehearsal-free continual learning},
  author={Smith, James Seale and Tian, Junjiao and Halbe, Shaunak and Hsu, Yen-Chang and Kira, Zsolt},
  booktitle={2023 IEEE/CVF Conference on Computer Vision and Pattern Recognition Workshops (CVPRW)},
  pages={2410--2420},
  year={2023},
  organization={IEEE}
}

@inproceedings{chaudhry2018riemannian,
  title={Riemannian walk for incremental learning: Understanding forgetting and intransigence},
  author={Chaudhry, Arslan and Dokania, Puneet K and Ajanthan, Thalaiyasingam and Torr, Philip HS},
  booktitle={Proceedings of the European conference on computer vision (ECCV)},
  pages={532--547},
  year={2018}
}

@article{french1999catastrophic,
  title={Catastrophic forgetting in connectionist networks},
  author={French, Robert M},
  journal={Trends in cognitive sciences},
  volume={3},
  number={4},
  pages={128--135},
  year={1999},
  publisher={Elsevier}
}

@article{cooley1965algorithm,
  title={An algorithm for the machine calculation of complex Fourier series},
  author={Cooley, James W and Tukey, John W},
  journal={Mathematics of computation},
  volume={19},
  number={90},
  pages={297--301},
  year={1965},
  publisher={JSTOR}
}

@book{aastrom2008feedback,
    author = {Astrom, Karl Johan and Murray, Richard M.},
    title = {Feedback Systems: An Introduction for Scientists and Engineers},
    year = {2008},
    publisher = {Princeton University Press},
}

@article{vaswani2017attention,
  title={Attention is all you need},
  author={Vaswani, Ashish and Shazeer, Noam and Parmar, Niki and Uszkoreit, Jakob and Jones, Llion and Gomez, Aidan N and Kaiser, {\L}ukasz and Polosukhin, Illia},
  journal={Advances in neural information processing systems},
  volume={30},
  year={2017}
}

@INPROCEEDINGS{liu2023periodically,
    author = { Liu, Qipeng and Lin, Luojun and Shen, Zhifeng and Yang, Zhifeng },
    booktitle = { 2023 IEEE/CVF International Conference on Computer Vision (ICCV) },
    title = {{ Periodically Exchange Teacher-Student for Source-Free Object Detection }},
    year = {2023},
    volume = {},
    ISSN = {},
    pages = {6391-6401},
    publisher = {IEEE Computer Society},
}

@article{Hsu2018Baselines,
  title={Re-evaluating continual learning scenarios: A categorization and case for strong baselines},
  author={Hsu, Yen-Chang and Liu, Yen-Cheng and Ramasamy, Anita and Kira, Zsolt},
  journal={arXiv preprint arXiv:1810.12488},
  year={2018}
}

@inproceedings{levinstein2026optimal,
    title={Optimal Rates in Continual Linear Regression via Increasing Regularization},
    author={Ran Levinstein and Amit Attia and Matan Schliserman and Uri Sherman and Daniel Soudry and Tomer Koren and Itay Evron},
    booktitle={The Thirty-ninth Annual Conference on Neural Information Processing Systems},
    year={2025}
}

@article{zhu2025global,
  title={Global convergence of continual learning on non-IID data},
  author={Zhu, Fei and Liu, Yujing and Liu, Wenzhuo and Zhang, Zhaoxiang},
  journal={arXiv preprint arXiv:2503.18511},
  year={2025}
}

@article{zhao2025highdimensional,
  title={High-dimensional Asymptotics of Generalization Performance in Continual Ridge Regression},
  author={Zhao, Yihan and Su, Wenqing and Yang, Ying},
  journal={arXiv preprint arXiv:2508.15494},
  year={2025}
}

@book{liao2023rmtml,
  title={Random matrix methods for machine learning},
  author={Couillet, Romain and Liao, Zhenyu},
  year={2022},
  publisher={Cambridge University Press}
}

@inproceedings{tsipory2026greedy,
    title={Are Greedy Task Orderings Better Than Random in Continual Linear Regression?},
    author={Matan Tsipory and Ran Levinstein and Itay Evron and Mark Kong and Deanna Needell and Daniel Soudry},
    booktitle={The Thirty-ninth Annual Conference on Neural Information Processing Systems},
    year={2025}
}

@article{sonthalia2024regularization,
  title={On regularization via early stopping for least squares regression},
  author={Sonthalia, Rishi and Lok, Jackie and Rebrova, Elizaveta},
  journal={arXiv preprint arXiv:2406.04425},
  year={2024}
}

@InProceedings{schwarz2018progress,
  title = {Progress \& Compress: A Scalable Framework for Continual Learning},
  author =       {Schwarz, Jonathan and Czarnecki, Wojciech and Luketina, Jelena and Grabska-Barwinska, Agnieszka and Teh, Yee Whye and Pascanu, Razvan and Hadsell, Raia},
  booktitle = 	 {Proceedings of the 35th International Conference on Machine Learning},
  pages = 	 {4528--4537},
  year = 	 {2018},
}

@article{silverstein1995empirical,
  title={On the empirical distribution of eigenvalues of a class of large dimensional random matrices},
  author={Silverstein, Jack W and Bai, Zhi Dong},
  journal={Journal of Multivariate analysis},
  volume={54},
  number={2},
  pages={175--192},
  year={1995},
  publisher={Elsevier}
}

@InProceedings{karpel2026optimal,
  title = 	 {Optimal L2 Regularization in High-dimensional Continual Linear Regression},
  author =       {Karpel, Gilad and Moroshko, Edward and Levinstein, Ran and Meir, Ron and Soudry, Daniel and Evron, Itay},
  booktitle = 	 {Proceedings of The 37th International Conference on Algorithmic Learning Theory},
  year = 	 {2026}
}

@inproceedings{rombach2022high,
  title={High-resolution image synthesis with latent diffusion models},
  author={Rombach, Robin and Blattmann, Andreas and Lorenz, Dominik and Esser, Patrick and Ommer, Bj{\"o}rn},
  booktitle={Proceedings of the IEEE/CVF conference on computer vision and pattern recognition},
  pages={10684--10695},
  year={2022}
}

@article{saharia2022photorealistic,
  title={Photorealistic text-to-image diffusion models with deep language understanding},
  author={Saharia, Chitwan and Chan, William and Saxena, Saurabh and Li, Lala and Whang, Jay and Denton, Emily L and Ghasemipour, Kamyar and Gontijo Lopes, Raphael and Karagol Ayan, Burcu and Salimans, Tim and others},
  journal={Advances in neural information processing systems},
  volume={35},
  pages={36479--36494},
  year={2022}
}

@article{patil2022mitigating,
  title={Mitigating multiple descents: A model-agnostic framework for risk monotonization},
  author={Patil, Pratik and Kuchibhotla, Arun Kumar and Wei, Yuting and Rinaldo, Alessandro},
  journal={arXiv preprint arXiv:2205.12937},
  year={2022}
}
\bibliographystyle{ICLR/iclr2027_conference}

\appendix
\newpage
\section{Numerical validations of the exact generalization predictions}
In this Appendix, we perform a numerical validations of the \emph{exact} generalization prediction produced by both \cref{thm:closedloss,thm:4}.

\subsection{Numerical validation of the prediction in \cref{thm:closedloss}}
\label{sec:val1}
We verify the theoretical prediction from \cref{thm:closedloss}.

In this experiment, we set $d=100$ and $v_x=1$, with $N=3$ and $\gamma=0.2$. We examine both i.i.d. Gaussian-perturbed teachers (where the perturbation has zero mean and standard deviation of $1/2$) and sinusoidal teachers as defined in \cref{assump:sin_teachers} (using $\mathbf B=\mathbf 1_d$ and $f=1/10$). In both configurations, the base teacher is initialized from a Bernoulli distribution with parameter $p=0.7$. All reported results are averaged over 50 random seeds.

\begin{figure}[h!]
    \centering
    \includegraphics[width=1\linewidth]{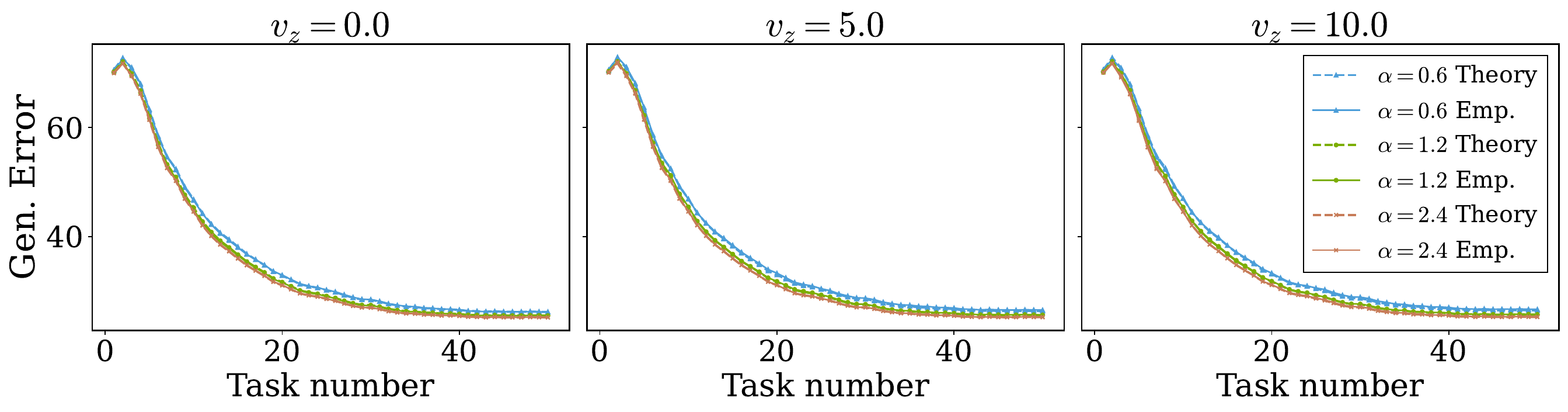}
    \caption{Continual linear regression with early-stopped GD (i.i.d.~teachers).}
\label{fig:11}
\end{figure}
\begin{figure}[h!]
    \centering
    \includegraphics[width=1\linewidth]{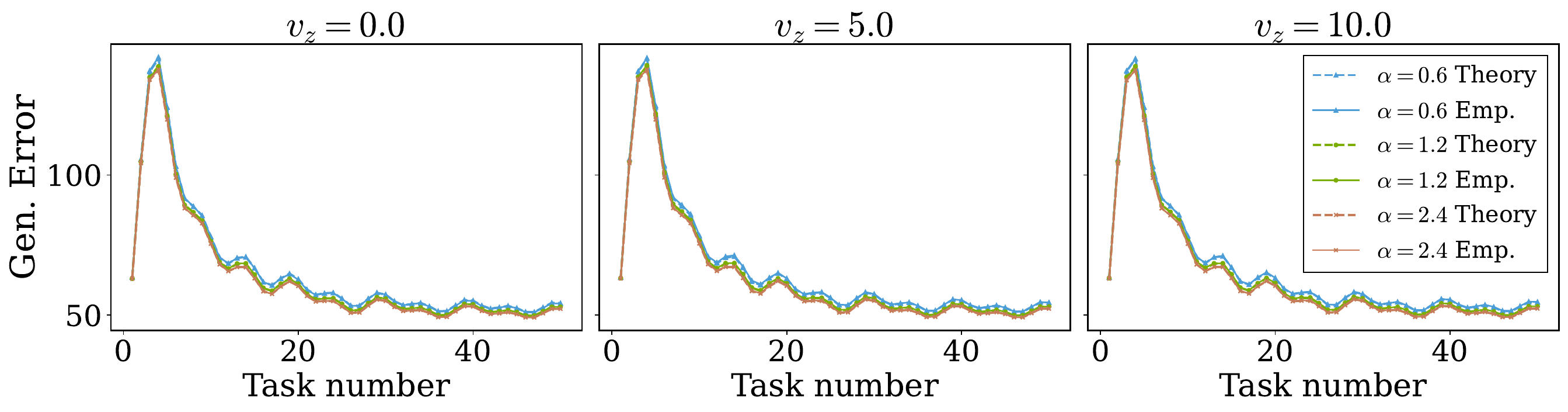}
    \caption{Continual linear regression with early-stopped GD (sinusoidal teachers).}
\label{fig:12}
\end{figure}
In both \cref{fig:11,fig:12} we observe a near identical matching between the theoretical and empirical generalization loss.

\newpage
\subsection{Numerical validation of the prediction in \cref{thm:4}}
\label{sec:val4}
We verify the theoretical prediction from \cref{thm:4}.

In this experiment, we set $d=100$, $v_x=1$ and use two diagonal regularizers. Ascending diagonal with entries $\mathrm{diag}(1/d,2/d,\dots,1)$ and descending diagonal with entries $\mathrm{diag}(1,(d-1)/d,\dots,1/d)$\footnote{Since $d$ is finite-dimensional, in this situation the smallest eigenvalue is bounded from below by $1/100$. In contrast, when $d\to \infty$, such diagonal regularization matrices lie outside the class of matrices examined in the paper, as the smallest eigenvalue is not bounded below by any positive scalar.}, in both configurations, the teacher is defined by setting the first $d/10$ components to 1 and all remaining components to 0. All results are averaged over 50 random seeds.

\begin{figure}[h!]
    \centering
    \includegraphics[width=1\linewidth]{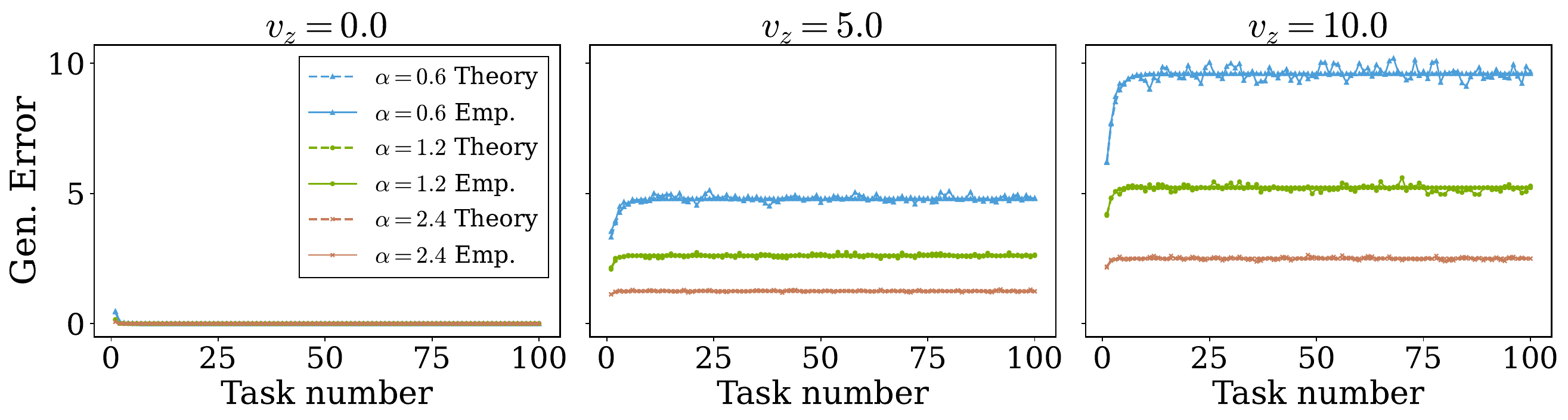}
    \caption{Continual Generalized L2 Regularization (ascending entries).}
\label{fig:13}
\end{figure}
\begin{figure}[h!]
    \centering
    \includegraphics[width=1\linewidth]{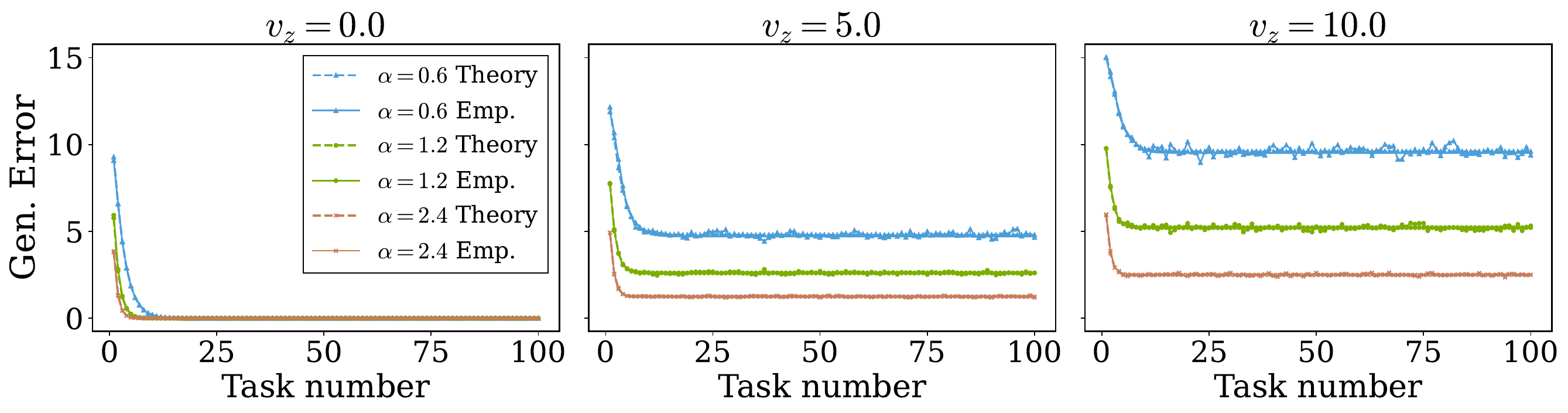}
    \caption{Continual Generalized L2 Regularization (descending entries).}
\label{fig:14}
\end{figure}
In both \cref{fig:13,fig:14} we observe a near identical matching between the theoretical and empirical generalization loss.

\newpage
\section{Additional Continual Linear Regression Experiments}
\subsection{Explicit Regularization Via an early-stopped GD On The Regularized Loss}
\label{sec:exl2}
We experiment with gradient descent on \mnist-based data (\cref{sec:5}) run for $N$ steps of over the \emph{regularized} objective
$\frac{1}{2n}\left\Vert\X\mathbf{w}-\y\right\Vert^2+\frac{\lambda}{2}\left\Vert\mathbf{w}-\mathbf{w}_{t-1}\right\Vert^2$. We visualize the empirical optimal regularizer over different numbers of steps and number of tasks.
Evidently, even for large $N$ (corresponding to near-convergent optimization within each task), the resulting $\lambda^\star(T)$ does not exhibit the predicted near linear scaling as predicted by \citet{karpel2026optimal}; such behavior only begins to emerge after a large number of tasks.

\begin{figure}[H]
    \centering
    \begin{subfigure}[t]{.44\columnwidth}
        \centering
        \includegraphics[width=\linewidth]{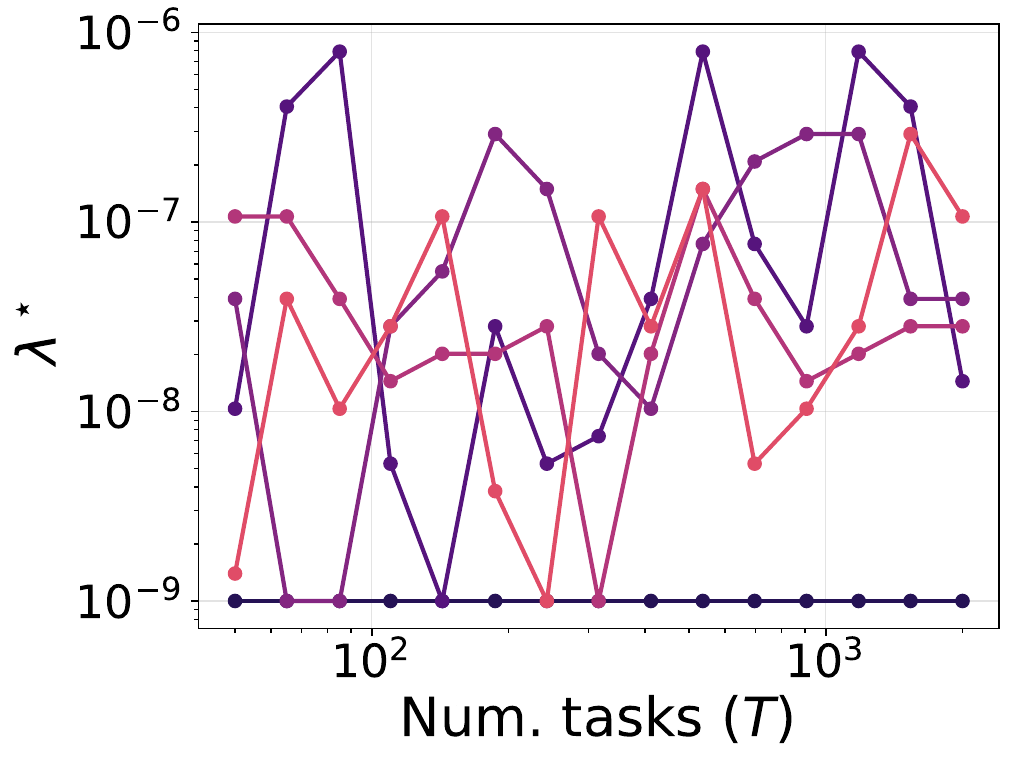}
        \caption{Effect of small steps ($N$) on the optimal regularizer.}
    \end{subfigure}
    \hfill
    \begin{subfigure}[t]{.55\columnwidth}
        \centering
        \includegraphics[width=\linewidth]{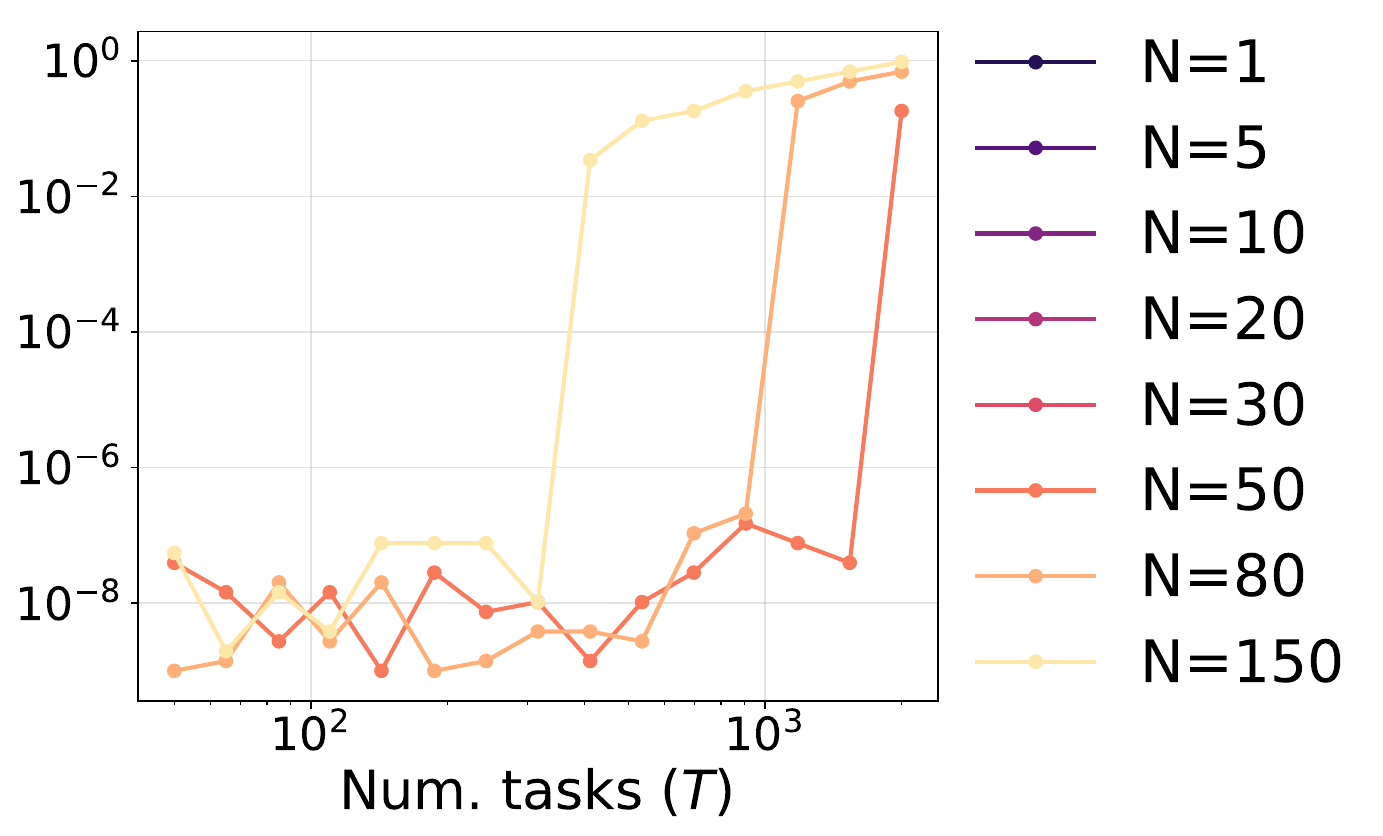}
        \caption{Effect of large steps ($N$) on the optimal regularizer.}
    \end{subfigure}
    \caption{
    \textbf{Explicit regularization does not replicate implicit regularization in continual learning.}
    We plot the optimal regularization parameter $\lambda^\star$ as a function of the number of tasks $T$ for different numbers of gradient steps $N$. We compute $\lambda^\star$ with an exhaustive search. 
    }
    \label{fig:Gen_T}
\end{figure}
\newpage
\section{Proofs for \cref{sec:1}: Early-Stopped GD}
\label{pr:sec1}
\subsection{Auxiliary Lemmas and definitions}
We begin by providing a few Lemmas, the following Lemma provides a closed-form solution to the iterates produced by early-stopped GD algorithm (\cref{schm:early_stopped}).
\begin{lemma}
\label{lem:1}
At task $t$ the predictor update under \cref{schm:early_stopped} after $k$ gradient steps is given by,
$$
\mathbf{w}^{k}_{t}=\left(\mathbf{I}-\gamma\frac{1}{n}\mathbf{X}_{t}^{\top}\mathbf{X}_{t}\right)^{k}\mathbf{w}_{t-1}+\gamma\frac{1}{n}\sum_{i=0}^{k-1}\left(\mathbf{I}-\gamma\frac{1}{n}\mathbf{X}_{t}^{\top}\mathbf{X}_{t}\right)^{i}\mathbf{X}_{t}^{\top}\mathbf{y}_{t}.
$$
\begin{proof}
Let $t\in [1, T]$, and let $k\in [1,N]$, then,
\begin{align}
    \mathbf{w}_{t}^{k}-\mathbf{w}_{t}^{\star}&=\mathbf{w}_{t}^{k-1}-\gamma\nabla\frac{1}{2n}\left\Vert \mathbf{X}_{t}\mathbf{w}_{t}^{k-1}-\mathbf{y}_{t}\right\Vert^2 -\mathbf{w}_{t}^{\star}\notag\\&=\mathbf{w}_{t}^{k-1}-\gamma\frac{1}{n}\mathbf{X}_{t}^{\top}\left(\mathbf{X}_{t}\mathbf{w}_{t}^{k-1}-\mathbf{y}_{t}\right)-\mathbf{w}_{t}^{\star}\notag\\&=\left(\mathbf{I}-\gamma\frac{1}{n}\mathbf{X}_{t}^{\top}\mathbf{X}_{t}\right)\mathbf{w}_{t}^{k-1}+\gamma\frac{1}{n}\mathbf{X}_{t}^{\top}\mathbf{y}_{t}-\left[\left(\mathbf{I}-\gamma\frac{1}{n}\mathbf{X}_{t}^{\top}\mathbf{X}_{t}\right)\mathbf{w}_{t}^{\star}+\gamma\frac{1}{n}\mathbf{X}_{t}^{\top}\mathbf{X}_{t}\mathbf{w}_{t}^{\star}\right]\notag\\&=\left(\mathbf{I}-\gamma\frac{1}{n}\mathbf{X}_{t}^{\top}\mathbf{X}_{t}\right)\left(\mathbf{w}_{t}^{k-1}-\mathbf{w}_{t}^{\star}\right)+\gamma\frac{1}{n}\mathbf{X}_{t}^{\top}\left(\mathbf{y}_{t}-\mathbf{X}_{t}\mathbf{w}_{t}^{\star}\right)\notag\\&=\ldots\notag\\&=\left(\mathbf{I}-\gamma\frac{1}{n}\mathbf{X}_{t}^{\top}\mathbf{X}_{t}\right)^{k}\left(\mathbf{w}_{t-1}-\mathbf{w}_{t}^{\star}\right)+\gamma\frac{1}{n}\sum_{i=0}^{k-1}\left(\mathbf{I}-\gamma\frac{1}{n}\mathbf{X}_{t}^{\top}\mathbf{X}_{t}\right)^{i}\mathbf{X}_{t}^{\top}\left(\mathbf{y}_{t}-\mathbf{X}_{t}\mathbf{w}_{t}^{\star}\right)\notag\\&=\left(\mathbf{I}-\gamma\frac{1}{n}\mathbf{X}_{t}^{\top}\mathbf{X}_{t}\right)^{k}\left(\mathbf{w}_{t-1}-\mathbf{w}_{t}^{\star}\right)+\gamma\frac{1}{n}\sum_{i=0}^{k-1}\left(\mathbf{I}-\gamma\frac{1}{n}\mathbf{X}_{t}^{\top}\mathbf{X}_{t}\right)^{i}\mathbf{X}_{t}^{\top}\mathbf{y}_{t}\notag\\&\hspace{12em}-\underbrace{\gamma\frac{1}{n}\sum_{i=0}^{k-1}\left(\mathbf{I}-\gamma\frac{1}{n}\mathbf{X}_{t}^{\top}\mathbf{X}_{t}\right)^{i}\mathbf{X}_{t}^{\top}\mathbf{X}_{t}\mathbf{w}_{t}^{\star}}_{\text{*}}\label{eq:1},
\end{align}
we continue with developing the last term in \cref{eq:1} (denoted by *), 
\begin{equation}
\label{eq:2}
\begin{aligned}
    \gamma\frac{1}{n}\sum_{i=0}^{k-1}\left(\mathbf{I}-\gamma\frac{1}{n}\mathbf{X}_{t}^{\top}\mathbf{X}_{t}\right)^{i}\mathbf{X}_{t}^{\top}\mathbf{X}_{t}\mathbf{w}_{t}^{\star}=&\gamma\frac{1}{n}\sum_{s=0}^{k-1}\sum_{p=0}^{s}{\binom{s}{p}}\left(-\gamma\frac{1}{n}\mathbf{X}_{t}^{\top}\mathbf{X}_{t}\right)^{p}\mathbf{X}_{t}^{\top}\mathbf{X}_{t}\mathbf{w}_{t}^{\star}\\=&\gamma\frac{1}{n}\mathbf{X}_{t}^{\top}\mathbf{X}_{t}\sum_{s=0}^{k-1}\sum_{p=0}^{s}{\binom{s}{p}}\left(-\gamma\frac{1}{n}\mathbf{X}_{t}^{\top}\mathbf{X}_{t}\right)^{p}\mathbf{w}_{t}^{\star}\\=&\gamma\frac{1}{n}\mathbf{X}_{t}^{\top}\mathbf{X}_{t}\sum_{s=0}^{k-1}\left(\mathbf{I}-\gamma\frac{1}{n}\mathbf{X}_{t}^{\top}\mathbf{X}_{t}\right)^{s}\mathbf{w}_{t}^{\star}\\=&\left(\mathbf{I}-\left(\mathbf{I}-\gamma\frac{1}{n}\mathbf{X}_{t}^{\top}\mathbf{X}_{t}\right)^{k}\right)\mathbf{w}_{t}^{\star},
\end{aligned}
\end{equation}

Then by \cref{eq:1,eq:2},
\begin{align*}
    &\mathbf{w}_{t}^{k}-\mathbf{w}_{t}^{\star}=\left(\mathbf{I}-\gamma\frac{1}{n}\mathbf{X}_{t}^{\top}\mathbf{X}_{t}\right)^{k}\left(\mathbf{w}_{t-1}-\mathbf{w}_{t}^{\star}\right)\\&\hspace{6em}+\gamma\frac{1}{n}\sum_{i=0}^{k-1}\left(\mathbf{I}-\gamma\frac{1}{n}\mathbf{X}_{t}^{\top}\mathbf{X}_{t}\right)^{i}\mathbf{X}_{t}^{\top}\mathbf{y}_{t}-\left(\mathbf{I}-\left(\mathbf{I}-\gamma\frac{1}{n}\mathbf{X}_{t}^{\top}\mathbf{X}_{t}\right)^{k}\right)\mathbf{w}_{t}^{\star}\\&=\left(\mathbf{I}-\gamma\frac{1}{n}\mathbf{X}_{t}^{\top}\mathbf{X}_{t}\right)^{k}\mathbf{w}_{t-1}+\gamma\frac{1}{n}\sum_{i=0}^{k-1}\left(\mathbf{I}-\gamma\frac{1}{n}\mathbf{X}_{t}^{\top}\mathbf{X}_{t}\right)^{i}\mathbf{X}_{t}^{\top}\mathbf{y}_{t}-\mathbf{w}_{t}^{\star}.
\end{align*}
The lemma is an immediate consequence of the equality stated above.
\end{proof}
\end{lemma}
The following Lemma provides the closed-from solution to the iterates produced by ridge-regression algorithm (\cref{schm:regularized}).
\begin{lemma}
\label{lem:2}
At each iteration $t$ the predictor update under \cref{schm:regularized} is given by 
$$\mathbf{w}_{t}=\left(\frac{1}{n}\X_{t}^{\top}\X_{t}+\lambda \I\right)^{-1}\left(\frac{1}{n}\X_{t}^{\top}\mathbf{y}_{t}+\lambda \mathbf{w}_{t-1}\right).$$
\end{lemma}
\begin{proof}
    
Let
$$
J_t(\mathbf{w})=\frac{1}{2n}\left\Vert\X_t\mathbf{w}-\y_t\right\Vert^2+\frac{\lambda}{2} \left\Vert\mathbf{w}-\mathbf{w}_{t-1}\right\Vert^2
.
$$
this function is convex, and its gradient is given by,
$$
\nabla_\mathbf{w} J_t(\mathbf{w})
= \frac{1}{n}\X_t^\top(\X_t\mathbf{w}-\y_t)+\lambda (\mathbf{w}-\mathbf{w}_{t-1})
= \left(\frac{1}{n}\X_t^\top\X_t+\lambda \I\right)\mathbf{w}-\left(\frac{1}{n}\X_t^\top\y_t+\lambda \mathbf{w}_{t-1}\right).
$$
At an optimal point $\mathbf{w}_t$, the first-order optimality (stationarity) condition must hold, hence,

$$
\left(\frac{1}{n}\X_t^\top\X_t+\lambda \mathbf I\right)\mathbf{w}_t=\frac{1}{n}\X_t^\top\y_t+\lambda \mathbf{w}_{t-1}.
$$
If $\lambda>0$, then $\frac{1}{n}\X_t^\top\X_t+\lambda \I$ is symmetric positive definite and invertible, so
$$
\mathbf{w}_t
= \left(\frac{1}{n}\X_t^\top\X_t+\lambda \mathbf I\right)^{-1}\left(\frac{1}{n}\X_t^\top\y_t+\lambda \mathbf{w}_{t-1}\right).
$$
\end{proof}
\begin{lemma}
\label{lem:3}
    For any $\mathbf u, \mathbf v, \mathbf w\in \mathbb R^d$ and $\eta>0$ we have $$\left\Vert \mathbf u+\mathbf v+\mathbf w\right\Vert ^{2}\leq \left\Vert \mathbf{v}\right\Vert ^{2}\left(1+\eta\right)+2\left(\left\Vert \mathbf{u}\right\Vert ^{2}+\left\Vert \mathbf{w}\right\Vert ^{2}\right)\left(1+\frac{1}{\eta}\right).$$
\end{lemma}

\begin{proof}
\begin{align*}
   \left\Vert \mathbf{u}+\mathbf{v}+\mathbf{w}\right\Vert ^{2}&=\left\Vert \mathbf{v}\right\Vert ^{2}+2\langle\mathbf{v},\mathbf{u}+\mathbf{w}\rangle+\left\Vert \mathbf{u}+\mathbf{w}\right\Vert ^{2}\\&\stackrel{\explain{1}}{\leq}\left\Vert \mathbf{v}\right\Vert ^{2}+2\left\Vert \mathbf{v}\right\Vert \left\Vert \mathbf{u}+\mathbf{w}\right\Vert +\left\Vert \mathbf{u}+\mathbf{w}\right\Vert ^{2}\\&\stackrel{\explain{2}}{\leq}\left\Vert \mathbf{v}\right\Vert ^{2}+\eta\left\Vert \mathbf{v}\right\Vert ^{2}+\frac{1}{\eta}\left\Vert \mathbf{u}+\mathbf{w}\right\Vert ^{2}+\left\Vert \mathbf{u}+\mathbf{w}\right\Vert ^{2}\\&=\left(1+\eta\right)\left\Vert \mathbf{v}\right\Vert ^{2}+\left(1+\frac{1}{\eta}\right)\left\Vert \mathbf{u}+\mathbf{w}\right\Vert ^{2}\\&\leq\left(1+\eta\right)\left\Vert \mathbf{v}\right\Vert ^{2}+2\left(1+\frac{1}{\eta}\right)\left(\left\Vert \mathbf{u}\right\Vert ^{2}+\left\Vert \mathbf{w}\right\Vert ^{2}\right), 
\end{align*}
where [1] follows by Cauchy-Schwarz inequality, and [2] follows by Young inequality for all $\eta>0$.
\end{proof}
\begin{lemma}
\label{lem:fix}
Let $\X\in\reals^{n\times d}$ satisfy the high-dimensional setting of \cref{sec:statistical_setting}, and define the sample covariance $\mathbf S=\frac1n \mathbf X^\top\mathbf X$ with the Marchenko-Pastur distribution $\mu_{\mathrm{MP},\frac{1}{\ratio}}$. Let $f$ be a fixed polynomial, and let $\mathbf a,\mathbf b\in\reals^{d}$ be deterministic (or independent) vectors with bounded norm, i.e., $\Vert\mathbf a\Vert=\Vert\mathbf b\Vert= O(1)$ (where for random quantities it holds almost surely).
Then as $n,d \to \infty$, we have almost surely that, 
\begin{equation}\label{eq:expectation_isotropic_poly_mp}
\mathbf a^\top f(\mathbf S)\mathbf b \to
\Big(\int f(\lambda)\dd\mu_{\mathrm{MP},\frac{1}{\ratio}}(\lambda)\Big)\mathbf a^\top\mathbf  b.
\end{equation}
\end{lemma}

\begin{proof}
Let
$$
g_d(z) := \mathbf a^\top (\mathbf S_d - z\mathbf I)^{-1} \mathbf b.
$$
For each $d$, the function $g_d$ is analytic on the domain $D=\mathbb{C} \setminus \mathrm{supp}(\mu_{\mathrm{MP},1/\alpha}),$
and by spectral confinement, almost surely for all sufficiently large $d$, $\mathrm{spec}(\mathbf S_d) \subset \mathrm{supp}(\mu_{\mathrm{MP},1/\alpha}) + o(1)$,
so that $g_d$ is analytic on $D$ for all large $d$.
Moreover, for any compact set $K \subset D$, there exists $c_K > 0$ such that
$\mathrm{dist}(z,\mathrm{spec}(\mathbf S_d)) \ge c_K$ for all $z \in K$ and all large $d$,
almost surely. Hence, $|g_d(z)|\le \Vert\mathbf a\Vert\Vert\mathbf b\Vert\Vert(\mathbf S_d - z\mathbf I)^{-1}\Vert\le \frac{\Vert\mathbf a\Vert\Vert\mathbf b\Vert}{c_K},$ so that $(g_d)$ is uniformly bounded on compact subsets of $D$ almost surely.

Let $E \subset D$ be a countable dense subset. By the isotropic resolvent law (see, e.g. \citet[Theorem 2.6]{liao2023rmtml}),
for each fixed $z \in E$,
$$
g_d(z)=\mathbf a^\top (\mathbf S_d - z\mathbf I)^{-1} \mathbf b\to m_{1/\alpha}(z)\mathbf a^\top \mathbf b,
$$
almost surely. By taking a countable intersection over $z \in E$, this convergence holds
simultaneously for all $z \in E$ on an event of probability one.

Therefore, by Vitali's convergence theorem, it follows that
$
g_d(z) \to m_{1/\alpha}(z)\mathbf a^\top \mathbf b
,$ uniformly on compact subsets of $D$,
almost surely.
Let $\Gamma \subset D$ be a contour enclosing $\mathrm{supp}(\mu_{\mathrm{MP},1/\alpha})$ with positive margin.
By the holomorphic functional calculus,
$$
\mathbf a^\top f(\mathbf S_d)\mathbf b=-\frac{1}{2\pi i}\oint_\Gamma f(z)g_d(z)\dd z.
$$
By uniform convergence of $g_d$ on $\Gamma$, we may pass the limit inside the integral:
$$
\mathbf a^\top f(\mathbf S_d)\mathbf b\to-\frac{1}{2\pi i}\oint_\Gamma f(z)m_{1/\alpha}(z)\dd z\mathbf a^\top \mathbf b,
$$
almost surely.
Finally, using the Stieltjes transform representation $m_{1/\alpha}(z)=\int \frac{1}{\lambda - z}\dd\mu_{\mathrm{MP},1/\alpha}(\lambda),$
and Cauchy's integral formula, we obtain $-\frac{1}{2\pi i}\oint_\Gamma f(z)m_{1/\alpha}(z)\dd z=\int f(\lambda)\dd\mu_{\mathrm{MP},1/\alpha}(\lambda).$
This completes the proof.
\end{proof}

We aim to establish the tight relationship between ridge regression and early-stopped gradient descent. To do so, we partially rely on the results developed in \cite{karpel2026optimal}. Note that our algorithmic formulation of ridge regression (\cref{schm:regularized}) differs slightly from the one used in \cite{karpel2026optimal}; nonetheless, the following lemma establishes the connection between the two formulations.

\begin{lemma}
\label{lem:updates}
Let $t\in [1, T]$ and consider the update $$\mathbf w_t \gets
{\argmin}_{\mathbf w}\Bigl\{\frac{1}{n}\Vert\mathbf X_{t}\mathbf w-\mathbf y_{t}\Vert^2+\lambda\Vert\mathbf w-\mathbf w_{t-1}\Vert^2\Bigr\}$$
and the update
$$\mathbf w_t' \gets{\argmin}_{\mathbf w}\Bigl\{\Vert\mathbf X_{t}\mathbf w-\mathbf y_{t}\Vert^2+\tilde\lambda d\Vert\mathbf w-\mathbf w_{t-1}\Vert^2\Bigr\}$$ 
then if $\tilde \lambda = \frac{n}{d}\lambda$ so $\mathbf w_t'=\mathbf w_t$.
\end{lemma}
\begin{proof}
Using \cref{lem:2} we have, 
\begin{align*} 
&\mathbf{w}_{t}=\left(\frac{1}{n}\X_{t}^{\top}\X_{t}+\lambda\mathbf{I}\right)^{-1}\left(\frac{1}{n}\X_{t}^{\top}\y_{t}+\lambda\mathbf{w}_{t-1}\right)\\&=\left(\X_{t}^{\top}\X_{t}+\lambda n\mathbf{I}\right)^{-1}\left(\X_{t}^{\top}\y_{t}+\lambda n\mathbf{w}_{t-1}\right)=\left(\X_{t}^{\top}\X_{t}+\tilde{\lambda}d\mathbf{I}\right)^{-1}\left(\X_{t}^{\top}\y_{t}+\tilde{\lambda}d\mathbf{w}_{t-1}\right)=\mathbf{w}_{t}'.
\end{align*}
which is the exact update solution for $\mathbf{w}_{t}'$ after redefining $\tilde \lambda = \frac{n}{d}\lambda$.
\end{proof}
By \cref{lem:updates}, if we adopt the algorithmic formulation of ridge regression from \cite{karpel2026optimal} and simply replace $\tilde \lambda$ with $\tilde\lambda := \frac{n}{d}\lambda = \ratio \lambda$, where $\lambda$ denotes the regularization parameter used in ridge-regression (\cref{schm:regularized}), then the iterates coincide at every step. Consequently, the expression in \cref{cor:closed} also remains unchanged under our algorithmic formulation (after substituting $\tilde\lambda$).

\begin{lemma}
\label{lem:differentiable}
Let $h$ be analytic at $\infty$, and define $h(x)=g(1/x)$, where $g$ is analytic in a neighborhood of $0$.
Let the Taylor expansion of $g$ at $0$ be
\[
g(x)=P_k(x)+R(x),
\qquad
R(x)=O(x^{k+1}),
\]
where $P_k$ is a polynomial of degree at most $k$.
Then, as $x\to\infty$,
\[
\frac{\dd}{\dd x}h(x)
=
-\frac{1}{x^{2}} P_k'\left(\frac{1}{x}\right)
+
O\left(\frac{1}{x^{k+2}}\right).
\]
\end{lemma}

\begin{proof}
Since $g$ is analytic at $0$, there exists $\rho>0$ such that for all $|x|<\rho$,
\[
g(x)
=
\sum_{n=0}^{k}\frac{g^{(n)}(0)}{n!}x^{n}
+
\sum_{n=k+1}^{\infty}\frac{g^{(n)}(0)}{n!}x^{n}.
\]
Power series may be differentiated term wise within their radius of convergence, and the resulting series converges uniformly on any closed interval $|x|\le r<\rho$.
Differentiating, we obtain
\[
g'(x)
=
\sum_{n=0}^{k}\frac{g^{(n+1)}(0)}{n!}x^{n}
+
\sum_{n=k+1}^{\infty}\frac{g^{(n)}(0)}{(n-1)!}x^{n-1},
\qquad |x|\le r.
\]

For $|x|<r$, write $|x|=r\theta$ with $\theta\in[0,1]$. Then
\begin{align*}
\left|
g'(x)-P_k'(x)
\right|
&=
\sum_{n=k+1}^{\infty}\frac{|g^{(n)}(0)|}{(n-1)!}|x|^{n-1}
=
\sum_{n=k+1}^{\infty}\frac{|g^{(n)}(0)|}{(n-1)!}(r\theta)^{n-1}
\\
&\le
\theta^{k}
\sum_{n=k+1}^{\infty}\frac{|g^{(n)}(0)|}{(n-1)!}r^{n-1}
=
C\theta^{k}
=
\frac{C}{r^{k}}|x|^{k},
\end{align*}
where the series is convergent since $r<\rho$.
Thus,
$$
g'(x)
=
P_k'(x)+O(x^{k}),
\qquad |x|\le r.
$$

Finally, using $h(x)=g(1/x)$ and the chain rule, for $|x|\ge \frac{1}{r},$
\begin{align*}
\frac{\dd}{\dd x}h(x)
&=
g'\left(\frac{1}{x}\right)\left(-\frac{1}{x^{2}}\right)
=
\left[
P_k'\left(\frac{1}{x}\right)
+
O\left(\frac{1}{x^{k}}\right)
\right]
\left(-\frac{1}{x^{2}}\right)
=
-\frac{1}{x^{2}}P_k'\left(\frac{1}{x}\right)
+
O\left(\frac{1}{x^{k+2}}\right).
\end{align*}
\end{proof}
\begin{lemma}
\label{lem:small}
    Let $f,g:\mathbb R\to \mathbb R$, be analytic functions at $\infty$, assume that as $x\to\infty$, $f(x)=\sum_{s=0}^k a_sx^{-s}+O(x^{-k-1})$ and $g(x)=\sum_{s=0}^k a_sx^{-s}+O(x^{-k-1})$ with some non-zero coefficient which is not $a_0$, then $\left|\frac{\dd}{\dd x}f(x)-\frac{\dd}{\dd x}g(x)\right|=o\left(\frac{\dd}{\dd x}g(x)\right)$.
\end{lemma}
\begin{proof}
Since $f,g$ are analytic function, we use \cref{lem:differentiable} to obtain, 
\begin{align*}
    \left|\frac{\dd}{\dd x}f(x)-\frac{\dd}{\dd x}g(x)\right|=O(x^{-k-2}).
\end{align*}
Since $\frac{\dd}{\dd x}g(x)=\Omega(x^{-k-1})$ we get $\left|\frac{\dd}{\dd x}f(x)-\frac{\dd}{\dd x}g(x)\right|=o\left(\frac{\dd}{\dd x}g(x)\right)$.
\end{proof}

\newpage

\subsection{Proof of \cref{thm:closedloss}: closed-form expression for the generalization loss from the early-stopped GD algorithm}
\label{pr:thm1}
\begin{recall}[\cref{thm:closedloss}]
Recall the high-dimensional setting where the number of samples and the dimension grow at proportional rate $\alpha:=\frac{n}{d}$. Moreover, recall that $\featnoise,\labelnoise$ are the variances of features and labels, respectively. Then, the generalization loss after $T$ tasks of the early-stopped GD algorithm (\cref{schm:early_stopped}) satisfies, as $n, d\to\infty$,
\vspace{-0.2em}
\[
\begin{aligned}
G^{\mathrm{esgd}}_T 
  &\xrightarrow[]{P} \labelnoise c'\frac{1-a'^{T}}{1-a'}
     +\frac{1}{T}\sum_{i=1}^{T}
     \left\Vert \mathbf{w}_{T}^{\star}-\mathbf{w}_{i}^{\star}\right\Vert^{2}
  %    \\[6pt]
  % &
  % \quad 
  + 2\big(\mathbf{w}_{T}^{\star}-\bar{\mathbf{w}}^{\star}\big)^{\top}
     \sum_{i=1}^{T} b'^{T-i+1}
     \big(\mathbf{w}_{i-1}^{\star}-\mathbf{w}_{i}^{\star}\big)\\[4pt]
  &
  \quad + \sum_{i=1}^{T}\sum_{i'=1}^{T}
     a'^{T-\max(i,i')+1}b'^{|i-i'|}
     \big(\mathbf{w}_{i-1}^{\star}-\mathbf{w}_{i}^{\star}\big)^{\top}
     \big(\mathbf{w}_{i'-1}^{\star}-\mathbf{w}_{i'}^{\star}\big)
     ,
\end{aligned}
\]
where we denote $\mathbf w^\star_0=\mathbf{w}_0$ and $\bar {\mathbf w}^\star=\frac{1}{T}\sum_i \mathbf w^\star_i$,
% \begin{equation}
% \begin{aligned}
% a' &= 1 - 2N\gamma\featnoise + O\big((N\gamma)^2\big),\\
% b' &= 1 - N\gamma\featnoise + O\big((N\gamma)^2\big),\\
% c' &= \frac{(N\gamma)^2}{\ratio}\featnoise+ O\big((N\gamma)^3\big).
% \end{aligned}
% \end{equation}
\begin{equation}
\begin{aligned}
&a'=1-2N\featnoise\gamma
+O\big((N\gamma)^{2}\big),\qquad
b'=1-N\featnoise\gamma
+O\big((N\gamma)^{2}\big),
\\&\hspace{10em}
c'
=
\frac{N^{2}}{\ratio}\featnoise\gamma^{2}
+O\big((N\gamma)^{3}\big).
\end{aligned}
\end{equation}
\end{recall}

Recall \cref{lem:1},
\[
\mathbf{w}^{\mathrm{esgd}}_{T}
=
\left(\mathbf{I}-\gamma\frac{1}{n}\mathbf{X}_{T}^{\top}\mathbf{X}_{T}\right)^{N}\mathbf{w}^{\mathrm{esgd}}_{T-1}
+\gamma\frac{1}{n}\sum_{s=0}^{N-1}
\left(\mathbf{I}-\gamma\frac{1}{n}\mathbf{X}_{T}^{\top}\mathbf{X}_{T}\right)^{s}\mathbf{X}_{T}^{\top}\mathbf{y}_{T},
\]
and let $\mathbf P_{k}=\left(\mathbf{I}-\gamma\frac{1}{n}\mathbf{X}_{k}^{\top}\mathbf{X}_{k}\right)$, with the convention $\mathbf{w}_0^\star=\mathbf{w}^{\mathrm{esgd}}_0$.

Then,
\begin{align*}
\mathbf{w}^{\mathrm{esgd}}_{T}
&=
\left(\mathbf{I}-\gamma\frac{1}{n}\mathbf{X}_{T}^{\top}\mathbf{X}_{T}\right)^{N}\mathbf{w}^{\mathrm{esgd}}_{T-1}
+\gamma\frac{1}{n}\sum_{s=0}^{N-1}
\left(\mathbf{I}-\gamma\frac{1}{n}\mathbf{X}_{T}^{\top}\mathbf{X}_{T}\right)^{s}\mathbf{X}_{T}^{\top}\mathbf{y}_{T}\\
&=
\mathbf P_T^{N}\mathbf{w}^{\mathrm{esgd}}_{T-1}
+\gamma\frac{1}{n}\sum_{s=0}^{N-1}\mathbf P_T^{s}\mathbf{X}_{T}^{\top}\left(\mathbf{X}_{T}\mathbf{w}_{T}^{\star}+\mathbf{z}_{T}\right)\\
&=
\mathbf P_T^{N}\mathbf{w}^{\mathrm{esgd}}_{T-1}
+\gamma\frac{1}{n}\sum_{s=0}^{N-1}\mathbf P_T^{s}\mathbf{X}_{T}^{\top}\mathbf{X}_{T}\mathbf{w}_{T}^{\star}
+\gamma\frac{1}{n}\sum_{s=0}^{N-1}\mathbf P_T^{s}\mathbf{X}_{T}^{\top}\mathbf{z}_{T}.
\end{align*}
Using $\mathbf I-\mathbf P_T=\gamma\frac{1}{n}\mathbf X_T^\top\mathbf X_T$, we obtain
\[
\gamma\frac{1}{n}\sum_{s=0}^{N-1}\mathbf P_T^{s}\mathbf{X}_{T}^{\top}\mathbf{X}_{T}
=
\sum_{s=0}^{N-1}\mathbf P_T^{s}(\mathbf I-\mathbf P_T)
=
\mathbf I-\mathbf P_T^{N}.
\]
Therefore,
\begin{align}
\label{eq:41}
\mathbf{w}^{\mathrm{esgd}}_{T}
&=
\mathbf P_T^{N}\mathbf{w}^{\mathrm{esgd}}_{T-1}
+\left(\mathbf I-\mathbf P_T^{N}\right)\mathbf w_T^\star
+\gamma\frac{1}{n}\sum_{s=0}^{N-1}\mathbf P_T^{s}\mathbf{X}_{T}^{\top}\mathbf{z}_{T}.
\end{align}

Since $G_{T}^{\mathrm{esgd}}=\frac{1}{T}\sum_{i=1}^T\Vert\mathbf{w}_{T}^{\mathrm{esgd}}-\mathbf{w}_i^\star\Vert^2$,
we therefore develop the corresponding expression for $\mathbf{w}^{\mathrm{esgd}}_{T}-\mathbf{w}_{i}^{\star}$, where $i\in[1,T]$. Using \cref{eq:41},
\begin{align*}
\mathbf{w}^{\mathrm{esgd}}_{T}-\mathbf{w}_{i}^{\star}
&=
\mathbf P_T^{N}\left(\mathbf{w}^{\mathrm{esgd}}_{T-1}-\mathbf{w}_{i}^{\star}\right)
+\left(\mathbf I-\mathbf P_T^{N}\right)\left(\mathbf w_T^\star-\mathbf w_i^\star\right)
+\gamma\frac{1}{n}\sum_{s=0}^{N-1}\mathbf P_T^{s}\mathbf{X}_{T}^{\top}\mathbf{z}_{T}.
\end{align*}
To rewrite this term note that for all $k\in [1,T]$, 
\begin{align*}
\left(\mathbf I-\mathbf P_k^{N}\right)\left(\mathbf w_k^\star-\mathbf w_i^\star\right)
&=
\Bigl[\left(\mathbf w_k^\star-\mathbf w_i^\star\right)-\mathbf P_k^{N}\left(\mathbf w_{k-1}^\star-\mathbf w_i^\star\right)\Bigr]
+\mathbf P_k^{N}\left(\mathbf w_{k-1}^\star-\mathbf w_k^\star\right).
\end{align*}
Unrolling the recursion and using the telescoping of the bracketed term yields
\begin{align*}
\mathbf{w}^{\mathrm{esgd}}_{T}-\mathbf{w}_{i}^{\star}
&=
\sum_{k=1}^{T}\left[\prod_{m=T}^{k+1}\mathbf P_m^{N}\right]\boldsymbol\gamma\frac{1}{n}\sum_{s=0}^{N-1}\mathbf P_k^{s}\mathbf{X}_{k}^{\top}\mathbf{z}_{k}
+\sum_{k=1}^{T}\left[\prod_{m=T}^{k}\mathbf P_m^{N}\right]\left(\mathbf w_{k-1}^\star-\mathbf w_k^\star\right)
+\left(\mathbf w_T^\star-\mathbf w_i^\star\right),
\end{align*}

so,
\begin{equation}
\label{eq:18}
\begin{aligned}
    &\left\Vert \mathbf{w}_{T}^{\mathrm{esgd}}-\mathbf{w}_{i}^{\star}\right\Vert ^{2}\\&=\bigg(\sum_{k=1}^{T}\left[\prod_{m=T}^{k+1}\mathbf{P}_{m}^{N}\right]\boldsymbol{\gamma}\frac{1}{n}\sum_{s=0}^{N-1}\mathbf{P}_{k}^{s}\mathbf{X}_{k}^{\top}\mathbf{z}_{k}+\sum_{k=1}^{T}\left[\prod_{m=T}^{k}\mathbf{P}_{m}^{N}\right]\left(\mathbf{w}_{k-1}^{\star}-\mathbf{w}_{k}^{\star}\right)+\left(\mathbf{w}_{T}^{\star}-\mathbf{w}_{i}^{\star}\right)\bigg)^{\top}&\\&\hspace{4em}\Bigg(\sum_{k=1}^{T}\left[\prod_{m=T}^{k+1}\mathbf{P}_{m}^{N}\right]\boldsymbol{\gamma}\frac{1}{n}\sum_{s=0}^{N-1}\mathbf{P}_{k}^{s}\mathbf{X}_{k}^{\top}\mathbf{z}_{k}+\sum_{k=1}^{T}\left[\prod_{m=T}^{k}\mathbf{P}_{m}^{N}\right]\left(\mathbf{w}_{k-1}^{\star}-\mathbf{w}_{k}^{\star}\right)+\left(\mathbf{w}_{T}^{\star}-\mathbf{w}_{i}^{\star}\right)\Bigg)\\&\stackrel{\explain{1}}{\xrightarrow[]{P}}\underbrace{\sum_{j=1}^{T}\gamma\frac{1}{n}\sum_{s=0}^{N-1}\mathbf{z}_{j}^{\top}\mathbf{X}_{j}\mathbf{P}_{j}^{s}\left(\prod_{k=j+1}^{T}\mathbf{P}_{k}^{N}\right)^{\top}\sum_{j'=1}^{T}\gamma\frac{1}{n}\left(\prod_{k=j'+1}^{T}\mathbf{P}_{k}^{N}\right)\sum_{s=0}^{N-1}\mathbf{P}_{j'}^{s}\mathbf{X}_{j'}^{\top}\mathbf{z}_{j'}}_{\text{Term 1}}\\&+\underbrace{\left(\sum_{k=1}^{T}\left[\prod_{m=k}^{T}\mathbf{P}_{m}^{N}\right]\left(\mathbf{w}_{k-1}^{\star}-\mathbf{w}_{k}^{\star}\right)+\left(\mathbf{w}_{T}^{\star}-\mathbf{w}_{i}^{\star}\right)\right)^{\top}\left(\sum_{k=1}^{T}\left[\prod_{m=k}^{T}\mathbf{P}_{m}^{N}\right]\left(\mathbf{w}_{k-1}^{\star}-\mathbf{w}_{k}^{\star}\right)+\left(\mathbf{w}_{T}^{\star}-\mathbf{w}_{i}^{\star}\right)\right)}_{\text{Term 2}},\\
\end{aligned}
\end{equation}
where [1] follows since $\mathbb E[\mathbf z_k]=\mathbf 0$ for all $k\in[1, T]$, and the noise variables are sampled independently of all tasks, so by the weak law of large numbers cross terms vanish as $n\to\infty$.

Let $\boldsymbol \xi_r^l=\left(\prod_{k=r+1}^{l}\mathbf{P}_{k}^{N}\right)\left[\sum_{s=0}^{N-1}\mathbf{P}_{r}^{s}\mathbf{X}_{r}^{\top}\right]\mathbf z_r$.
Then, the first term becomes,
\begin{align}
&\sum_{j=1}^{T}\left[\gamma\frac{1}{n}\sum_{s=0}^{N-1}\mathbf{z}_{j}^{\top}\mathbf{X}_{j}\mathbf{P}_{j}^{s}\right]\left(\prod_{k=j+1}^{T}\mathbf{P}_{k}^{N}\right)^{\top}\sum_{j'=1}^{T}\left(\prod_{k=j'+1}^{T}\mathbf{P}_{k}^{N}\right)\left[\gamma\frac{1}{n}\sum_{s=0}^{N-1}\mathbf{P}_{j'}^{s}\mathbf{X}_{j'}^{\top}\mathbf{z}_{j'}\right]\notag\\&=\sum_{j=1}^{T}\sum_{j'=1}^{T}\left[\gamma\frac{1}{n}\sum_{s=0}^{N-1}\mathbf{z}_{j}^{\top}\mathbf{X}_{j}\mathbf{P}_{j}^{s}\right]\left(\prod_{k=j+1}^{T}\mathbf{P}_{k}^{N}\right)^{\top}\left(\prod_{k=j'+1}^{T}\mathbf{P}_{k}^{N}\right)\left[\gamma\frac{1}{n}\sum_{s=0}^{N-1}\mathbf{P}_{j'}^{s}\mathbf{X}_{j'}^{\top}\mathbf{z}_{j'}\right]\notag\\&\stackrel{\explain{1}}{\xrightarrow{P}}\sum_{j=1}^{T}\mathbf{z}_{j}^{\top}\left[\gamma\frac{1}{n}\sum_{s=0}^{N-1}\mathbf{X}_{j}\mathbf{P}_{j}^{s}\right]\left(\prod_{k=j+1}^{T}\mathbf{P}_{k}^{N}\right)^{\top}\left(\prod_{k=j+1}^{T}\mathbf{P}_{k}^{N}\right)\left[\gamma\frac{1}{n}\sum_{s=0}^{N-1}\mathbf{P}_{j}^{s}\mathbf{X}_{j}^{\top}\right]\mathbf{z}_{j}\notag\\&=\gamma^{2}\sum_{j=1}^{T}\left[\frac{{\mathbf{\boldsymbol{\xi}}_{j}^{(T-1)}}^{\top}}{n}\mathbf{P}_{T}^{2N}\frac{\mathbf{\boldsymbol{\xi}}_{j}^{(T-1)}}{n}\right]\notag\\&\stackrel{\explain{2}}{\xrightarrow{P}}\frac{\gamma^{2}}{n^{2}}\sum_{j=1}^{T}a'{\mathbf{\boldsymbol{\xi}}_{j}^{(T-1)}}^{\top}\boldsymbol{\xi}_{j}^{(T-1)}\notag\\&=\frac{\gamma^{2}}{n^{2}}\sum_{j=1}^{T}a'\mathbf{z}_{j}^{\top}\left[\sum_{s=0}^{N-1}\mathbf{X}_{j}\mathbf{P}_{j}^{s}\right]\left(\prod_{k=j+1}^{T-1}\mathbf{P}_{k}^{N}\right)^{\top}\left(\prod_{k=j+1}^{T-1}\mathbf{P}_{k}^{N}\right)\left[\sum_{s=0}^{N-1}\mathbf{P}_{j}^{s}\mathbf{X}_{j}^{\top}\right]\mathbf{z}_{j}\notag\\&=\ldots\notag\xrightarrow{P}\frac{\gamma^{2}}{n^{2}}\sum_{j=1}^{T}a'^{T-j}\mathbf{z}_{j}^{\top}\sum_{s=0}^{N-1}\mathbf{X}_{j}\mathbf{P}_{j}^{s}\sum_{s=0}^{N-1}\mathbf{P}_{j}^{s}\mathbf{X}_{j}^{\top}\mathbf{z}_{j}&\notag\\&=\frac{\gamma^{2}}{n^{2}}
\sum_{j=1}^{T}a'^{T-j}
\mathbf z_j^\top \mathbf X_j
\left(\sum_{s,s'=0}^{N-1}\mathbf P_j^{s+s'}\right)
\mathbf X_j^\top \mathbf z_j.\label{eq:16}
\end{align}
where [1] follows since under the statistical design (\cref{sec:statistical_setting}) in which
$\mathbb E[\mathbf z_k]=\mathbf 0$ for all $k\in[1, T]$, and the noise variables are sampled independently of all tasks, by the weak law of large numbers cross terms vanish as $n\to\infty$. And [2] follows from \cref{lem:fix}, by the independence of the random variable $\boldsymbol \xi_j^{t-1}$ with $\mathbf P_t$ and since $\left\Vert\boldsymbol\xi_j^t\right\Vert^2/n<\infty$ almost surely for all $t\in [1, T]$. Also, note that we substituted $a'$ as the $2N$-th moment of the Marchenko-Pastur law.

Using the closed-form expression for the $k$-th moment of the Marchenko-Pastur law, we have (for all $j\in [1,T]$),
\begin{align}  &\frac{\gamma^{2}}{n^{2}}\mathbf{z}_{j}^{\top}\mathbf{X}_{j}\left(\sum_{s,s'=0}^{N-1}\mathbf{P}_{j}^{s+s'}\right)\mathbf{X}_{j}^{\top}\mathbf{z}_{j}=\frac{\gamma^{2}}{n^{2}}\mathrm{Tr}\left(\mathbf{z}_{j}^{\top}\mathbf{X}_{j}\sum_{s=0}^{2N-2}m_{s}\mathbf{P}_{j}^{s}\mathbf{X}_{j}^{\top}\mathbf{z}_{j}\right)\notag&\\&\stackrel{\explain{1}}{\xrightarrow{P}}v_{z}\frac{\gamma^{2}}{n^{2}}\mathrm{Tr}\left(\mathbf{X}_{j}^{\top}\mathbf{X}_{j}\sum_{s=0}^{2N-2}m_{s}\mathbf{P}_{j}^{s}\right)=\frac{1}{n}\sum_{s=0}^{2N-2}m_{s}\sum_{l=0}^{s}{\binom{s}{l}}\left(-\gamma\right)^{l+2}\frac{1}{n}\mathrm{Tr}\Bigg[\left(\frac{1}{n}\mathbf{X}_{j}^{\top}\mathbf{X}_{j}\right)^{l+1}\Bigg]\notag\\&\stackrel{\explain{2}}{\xrightarrow{P}}\labelnoise\sum_{s=0}^{2N-2}m_{s}\sum_{l=0}^{s}{\binom{s}{l}}\left(-\gamma\right)^{l+2}\sum_{r=0}^{l}\frac{\featnoise^{l+1}}{r+1}{\binom{l+1}{r}}{\binom{l}{r}}\frac{1}{\ratio^{r+1}}:=\labelnoise c'\left(\gamma,N\right):=\labelnoise c',\label{eq:36}
\end{align}
where,
$$
m_s \coloneqq 
\#\Big\{(u,v)\in\{0,\ldots,N-1\}^2 : u+v = s \Big\}
=
\begin{cases}
s+1, & 0 \le s \le N-1,\\[4pt]
2N-1-s, & N \le s \le 2N-2,\\[4pt]
0, & \text{otherwise}.
\end{cases},$$
and [1] follows by the weak law of large numbers. Moreover, [2] follows by \cref{lem:fix} with the closed-form expression. Similarly, 
\begin{align}
\label{eq:13}
a'=a'\left(\gamma, N\right)=1+\sum_{k=0}^{2N}{\binom{2N}{k}}\left(-\gamma\right)^{k}\sum_{r=0}^{k-1}\frac{\featnoise^{k}}{r+1}{\binom{k}{r}}{\binom{k-1}{r}}\frac{1}{\ratio^r} 
\end{align}
Then the first term in \cref{eq:16,eq:36} becomes,
\begin{align*}
&\sum_{j=1}^{T}a'^{T-j}\frac{\gamma^{2}}{n}\mathbf{z}_{j}^{\top}\sum_{s,s'=0}^{N-1}\mathbf{P}_{j}^{s+s'}\frac{1}{n}\mathbf{X}_{j}^{\top}\mathbf{X}_{j}\mathbf{z}_{j}=\labelnoise c'\sum_{j=1}^{T}a'^{T-j}=\labelnoise c'\frac{1-a'^{T}}{1-a'}
\end{align*}

And the second term in \cref{eq:18} becomes,
\begin{equation}
 \label{eq:27}   
\begin{aligned}
    &\left(\sum_{k=1}^{T}\left(\mathbf{w}_{k-1}^{\star}-\mathbf{w}_{k}^{\star}\right)^{\top}\left[\prod_{m=T}^{k}\mathbf{P}_{m}^{N}\right]+\left(\mathbf{w}_{T}^{\star}-\mathbf{w}_{i}^{\star}\right)^{\top}\right)\left(\sum_{k=1}^{T}\left[\prod_{m=k}^{T}\mathbf{P}_{m}^{N}\right]\left(\mathbf{w}_{k-1}^{\star}-\mathbf{w}_{k}^{\star}\right)+\left(\mathbf{w}_{T}^{\star}-\mathbf{w}_{i}^{\star}\right)\right)&\\&=\sum_{k=1}^{T}\left(\mathbf{w}_{k-1}^{\star}-\mathbf{w}_{k}^{\star}\right)^{\top}\left[\prod_{m=T}^{k}\mathbf{P}_{m}^{N}\right]\sum_{k=1}^{T}\left[\prod_{m=k}^{T}\mathbf{P}_{m}^{N}\right]\left(\mathbf{w}_{k-1}^{\star}-\mathbf{w}_{k}^{\star}\right)\\&\hspace{10em}+2\sum_{k=1}^{T}\left(\mathbf{w}_{k-1}^{\star}-\mathbf{w}_{k}^{\star}\right)^{\top}\left[\prod_{m=T}^{k}\mathbf{P}_{m}^{N}\right]\left(\mathbf{w}_{T}^{\star}-\mathbf{w}_{i}^{\star}\right)+\left\Vert \mathbf{w}_{T}^{\star}-\mathbf{w}_{i}^{\star}\right\Vert ^{2}.
\end{aligned}
\end{equation}

The first term in \cref{eq:27} becomes,
\begin{align*}
    &\sum_{k=1}^{T}\sum_{k'=1}^{T}\left(\mathbf{w}_{k-1}^{\star}-\mathbf{w}_{k}^{\star}\right)^{\top}\left[\prod_{m=k}^{T}\mathbf{P}_{m}^{N}\prod_{m=T}^{k'}\mathbf{P}_{m}^{N}\right]\left(\mathbf{w}_{k'-1}^{\star}-\mathbf{w}_{k'}^{\star}\right)\\&=\sum_{k=1}^{T}\sum_{k'=1}^{T}\left(\mathbf{w}_{k-1}^{\star}-\mathbf{w}_{k}^{\star}\right)^{\top}\left[\prod_{m=\left(k\lor k'\right)}^{T}\mathbf{P}_{m}^{N}\prod_{m=T}^{\left(k\lor k'\right)}\mathbf{P}_{m}^{N}\prod_{m=\left(k\lor k'\right)-1}^{\left(k'\land k\right)}\mathbf{P}_{m}^{N}\right]\left(\mathbf{w}_{k'-1}^{\star}-\mathbf{w}_{k'}^{\star}\right)\\&\stackrel{\explain{1}}{\xrightarrow[]{P}}\sum_{k=1}^{T}\sum_{k'=1}^{T}a'^{T\left(k\lor k'\right)+1}b'^{\left|k-k'\right|}\left(\mathbf{w}_{k'-1}^{\star}-\mathbf{w}_{k'}^{\star}\right)^\top \left(\mathbf{w}_{k-1}^{\star}-\mathbf{w}_{k}^{\star}\right), 
\end{align*}
To justify [1] we similarly define $\boldsymbol \xi_r^i=\prod_{m=i}^{r}\mathbf{P}_{m}^{N}\left(\mathbf{w}_{r-1}^{\star}-\mathbf{w}_{r}^{\star}\right)$ and $\boldsymbol {\tilde \xi}_r^i=\prod_{m=r}^{i}\mathbf{P}_{m}^{N}\left(\mathbf{w}_{r-1}^{\star}-\mathbf{w}_{r}^{\star}\right)$ (where products over empty index sets are defined to be the identity). Then for all $k,k'\in [1, T]$ we have,
\begin{equation}
\label{eq:28}
\begin{aligned}
    &\left(\mathbf{w}_{k-1}^{\star}-\mathbf{w}_{k}^{\star}\right)^{\top}\left[\prod_{m=\left(k\lor k'\right)}^{T}\mathbf{P}_{m}^{N}\prod_{m=T}^{\left(k\lor k'\right)}\mathbf{P}_{m}^{N}\prod_{m=\left(k\lor k'\right)-1}^{\left(k'\land k\right)}\mathbf{P}_{m}^{N}\right]\left(\mathbf{w}_{k'-1}^{\star}-\mathbf{w}_{k'}^{\star}\right)&\\&=\left(\boldsymbol{\xi}_{T-1}^{\left(k\lor k'\right)}\right)^{\top}\mathbf{P}_{T}^{2N}\tilde{\boldsymbol{\xi}}_{T-1}^{\left(k\land k'\right)}\stackrel{[1]}{\xrightarrow{P}}a'\left(\boldsymbol{\xi}_{T-1}^{\left(k\lor k'\right)}\right)^{\top}\tilde{\boldsymbol{\xi}}_{T-1}^{\left(k\land k'\right)}=a'\left(\boldsymbol{\xi}_{T-2}^{\left(k\lor k'\right)}\right)^{\top}\mathbf{P}_{T-1}^{2N}\tilde{\boldsymbol{\xi}}_{T-2}^{\left(k\land k'\right)}\\&=\ldots\\&\xrightarrow{P}a'^{T-\left(k\lor k'\right)+1}\left(\mathbf{w}_{k-1}^{\star}-\mathbf{w}_{k}^{\star}\right)^{\top}\prod_{m=\left(k\lor k'\right)-1}^{\left(k'\land k\right)}\mathbf{P}_{m}^{N}\left(\mathbf{w}_{k'-1}^{\star}-\mathbf{w}_{k'}^{\star}\right)\\&=a'^{T-\left(k\lor k'\right)+1}\left[\left(\boldsymbol{\xi}_{\left(k\lor k'\right)-1}^{\left(k\lor k'\right)}\right)^{\top}\mathbf{P}_{\left(k\lor k'\right)-1}^{N}\tilde{\boldsymbol{\xi}}_{\left(k\lor k'\right)-1}^{\left(k\land k'\right)}\right]\\&\stackrel{[1]}{\xrightarrow{P}}a'^{T-\left(k\lor k'\right)+1}b'\left(\boldsymbol{\xi}_{\left(k\lor k'\right)-1}^{\left(k\lor k'\right)}\right)^{\top}\tilde{\boldsymbol{\xi}}_{\left(k\lor k'\right)-1}^{\left(k\land k'\right)}\\&=\ldots\\&\xrightarrow{P}a'^{T-\left(k\lor k'\right)+1}b'^{\left|k-k'\right|}\left(\mathbf{w}_{k-1}^{\star}-\mathbf{w}_{k}^{\star}\right)^{\top}\left(\mathbf{w}_{k'-1}^{\star}-\mathbf{w}_{k'}^{\star}\right),
\end{aligned}
\end{equation}
Here, [1] follows from \cref{lem:fix} and from the independence of $\boldsymbol \xi_{t-1}^{\left(k\lor k'\right)}$ and $\tilde{\boldsymbol \xi}_{t-1}^{\left(k\land k'\right)}$ from $\mathbf P_t$, together with the facts that $\left\Vert\boldsymbol \xi_{t-1}^{\left(k\lor k'\right)}\right\Vert\in O(1)$ and $\left\Vert\tilde{\boldsymbol \xi}_{t-1}^{\left(k\land k'\right)}\right\Vert \in O(1)$ almost surely, which follows since $\left\Vert\mathbf w^\star_{k-1}-\mathbf w^\star_k\right\Vert \in O(1)$ (for all $k\in [1, T]$, see \cref{sec:statistical_setting}) and $\left\Vert \mathbf P_t\right\Vert_2 < 1$ for all $t \in [1, T]$ almost surely, as $\gamma <\frac{2}{\featnoise\left(1+\sqrt\frac{1}{\ratio}\right)^2}$.
Furthermore, using the closed expression for the $N$-th moment in the Marchenko-Pastur law, \begin{align}
\label{eq:35}
b'=b'\left(\gamma, N\right)=1+\sum_{k=0}^{N}{\binom{N}{k}}\left(-\gamma\right)^{k}\sum_{r=0}^{k-1}\frac{\featnoise^{k}}{r+1}{\binom{k}{r}}{\binom{k-1}{r}}\frac{1}{\ratio^r}
\end{align}
Since $\left(\mathbf{w}_{k-1}^{\star}-\mathbf{w}_{k}^{\star}\right)^{\top}\left[\prod_{m=T}^{k}\mathbf{P}_{m}^{N}\right]\left(\mathbf{w}_{T}^{\star}-\mathbf{w}_{i}^{\star}\right)=b'^{T-k+1}\left(\mathbf{w}_{k-1}^{\star}-\mathbf{w}_{k}^{\star}\right)^{\top}\left(\mathbf{w}_{T}^{\star}-\mathbf{w}_{i}^{\star}\right)$ which similarly follows by the identity in \cref{eq:28}. Then, substituting back to \cref{eq:18}, yields,
\begin{align}
\begin{split}
    &\left\Vert \mathbf{w}^{\mathrm{esgd}}_{T}-\mathbf{w}_{i}^{\star}\right\Vert ^{2}\xrightarrow[]{P}\labelnoise c'\frac{1-a'^{T}}{1-a'}+\sum_{k=1}^{T}\sum_{k'=1}^{T}a'^{T-(k \lor k')+1}b'^{\left|k-k'\right|}\left(\mathbf{w}_{k-1}^{\star}-\mathbf{w}_{k}^{\star}\right)^{\top}\left(\mathbf{w}_{k'-1}^{\star}-\mathbf{w}_{k'}^{\star}\right)
    \\&\hspace{12em}+2\sum_{k=1}^{T}b'^{T-k+1}\left(\mathbf{w}_{T}^{\star}-\mathbf{w}_{i}^{\star}\right)^{\top}\left(\mathbf{w}_{k-1}^{\star}-\mathbf{w}_{k}^{\star}\right)+\left\Vert \mathbf{w}_{T}^{\star}-\mathbf{w}_{i}^{\star}\right\Vert ^{2}.
\label{eq:mainresult1}
\end{split}
\end{align}
And therefore,
\[
\begin{aligned}
&G^{\mathrm{esgd}}_T = \frac{1}{T}\sum_{i=1}^{T}\left\Vert\mathbf{w}_T-\mathbf{w}_i^\star\right\Vert^2
   \\& \xrightarrow[]{P} \labelnoise c'\frac{1-a'^{T}}{1-a'}
     +\frac{1}{T}\sum_{i=1}^{T}
     \left\Vert \mathbf{w}_{T}^{\star}-\mathbf{w}_{i}^{\star}\right\Vert^{2}
  + 2\big(\mathbf{w}_{T}^{\star}-\bar{\mathbf{w}}^{\star}\big)^{\top}
     \sum_{i=1}^{T} b'^{T-i+1}
     \big(\mathbf{w}_{i-1}^{\star}-\mathbf{w}_{i}^{\star}\big)
  \\& \hspace{4em} +\sum_{i=1}^{T}\sum_{i'=1}^{T}
     a'^{T-\max(i,i')+1}b'^{|i-i'|}
     \big(\mathbf{w}_{i-1}^{\star}-\mathbf{w}_{i}^{\star}\big)^{\top}
     \big(\mathbf{w}_{i'-1}^{\star}-\mathbf{w}_{i'}^{\star}\big)
     ,
\end{aligned}
\]
Here, $\bar{\mathbf{w}}^{\star}$ denotes the average teacher, i.e., $\bar{\mathbf{w}}^{\star}=\frac{1}{T}\sum_{i=1}^T\mathbf{w}_i^\star$.

By taking the Taylor expansion around $N\gamma\to 0$ for $a',b',c'$, we get,
\begin{equation}
\label{eq:33}
\begin{aligned}
&a'(\gamma,N)=1-2N\featnoise\gamma
+N(2N-1)\featnoise^{2}\left(1+\frac{1}{\ratio}\right)\gamma^{2}
+O\big((N\gamma)^{3}\big),
\\&
b'(\gamma,N)=1-N\featnoise\gamma
+\frac{N(N-1)}{2}\featnoise^{2}\left(1+\frac{1}{\ratio}\right)\gamma^{2}
+O\big((N\gamma)^{3}\big),
\\&
c'(\gamma,N)
=
\frac{N^{2}}{\ratio}\featnoise\gamma^{2}
+O\big((N\gamma)^{3}\big).
\end{aligned}
\end{equation}
\newpage
Before deriving the result of the optimal step size in early-stopped GD we simplify the expression of the generalization loss given in \cref{thm:closedloss} under the assumption of the i.i.d.~teachers (\cref{assump:iid_teachers}).

\begin{corollary}
\label{cor:closed}
Under the i.i.d. setting  (\cref{assump:iid_teachers}), and for all $T>1$, the generalization loss of \cref{schm:regularized} becomes, as $n,d\to\infty$
\begin{align}
G^{\mathrm{ridge}}_{T}\xrightarrow[]{P}\labelnoise{c'}\frac{1-a'^{T}}{1-a'}+a'^{T}\left\Vert \mathbf{w}^{\star}-\mathbf{w}_0\right\Vert ^{2}+2\trace(\boldsymbol{\Sigma})\prn{\frac{1-b'}{1-a'}+\frac{2b'^{T}-1}{T}+\frac{a'^{T}\prn{-a' +2b' -1}}{2(1-a')}}.
\end{align}
where 
\[
\begin{aligned}
  &a'=1+\sum_{k=0}^{2N}{\binom{2N}{k}}\left(-\gamma\right)^{k}\sum_{r=0}^{k-1}\frac{\featnoise^{k}}{r+1}{\binom{k}{r}}{\binom{k-1}{r}}\frac{1}{\ratio^r},\\&b'=1+\sum_{k=0}^{N}{\binom{N}{k}}\left(-\gamma\right)^{k}\sum_{r=0}^{k-1}\frac{\featnoise^{k}}{r+1}{\binom{k}{r}}{\binom{k-1}{r}}\frac{1}{\ratio^r},\\c'&=\sum_{s=0}^{2N-2}m_{s}\sum_{l=0}^{s}{\binom{s}{l}}\left(-\gamma\right)^{l+2}\sum_{r=0}^{l}\frac{\featnoise^{l+1}}{r+1}{\binom{l+1}{r}}{\binom{l}{r}}\frac{1}{\ratio^{r+1}}.
\end{aligned}
\]
\end{corollary}
\begin{proof}

Under the i.i.d.~teacher assumption \cref{assump:iid_teachers}, each task-specific teacher can be written as
$$
\mathbf{w}_i^\star = \mathbf{w}^\star + \boldsymbol{\xi}_i, \qquad i \in [T],
$$
where the random vectors $\boldsymbol{\xi}_i$ are i.i.d.\ with mean $\0_d$ and covariance
$\mathrm{Cov}(\boldsymbol{\xi}_i)=\boldsymbol{\Sigma}$.
Substituting this representation into the general expression derived in
\cref{eq:mainresult1}, yields the following
explicit decomposition of the squared error,
\begin{align*}
\left\Vert \mathbf{w}^{\mathrm{esgd}}_{T}-\mathbf{w}_{i}^{\star}\right\Vert^{2}\xrightarrow[]{P}&a'^{T}\left\Vert \mathbf{w}^{\star}-\mathbf{w}_{0}+\boldsymbol{\xi}_{1}\right\Vert^{2}+2b'^{T}\left(\boldsymbol{\xi}_{i}-\boldsymbol{\xi}_{T}\right)^{\top}\left(\mathbf{w}^{\star}-\mathbf{w}_{0}+\boldsymbol{\xi}_{1}\right)\\&\qquad
+\labelnoise c'\frac{1-a'^{T}}{1-a'}+2\sum_{k=2}^{T}a'^{T-k+1}b'^{k-1}\left(\boldsymbol{\xi}_{k}-\boldsymbol{\xi}_{k-1}\right)^{\top}\left(\mathbf{w}^{\star}-\mathbf{w}_{0}+\boldsymbol{\xi}_{1}\right)\\&\qquad+\sum_{k=2}^{T}\sum_{k'=2}^{T}a'^{T-(k \lor k')+1}b'^{(k \lor k')-(k \land k')}\left(\boldsymbol{\xi}_{k-1}-\boldsymbol{\xi}_{k}\right)^{\top}\left(\boldsymbol{\xi}_{k'-1}-\boldsymbol{\xi}_{k'}\right)\\&\qquad+2\sum_{k=2}^{T}b'^{T-k+1}\left(\boldsymbol{\xi}_{T}-\boldsymbol{\xi}_{i}\right)^{\top}\left(\boldsymbol{\xi}_{k-1}-\boldsymbol{\xi}_{k}\right)+\left\Vert \boldsymbol{\xi}_{T}-\boldsymbol{\xi}_{i}\right\Vert ^{2}.&
\end{align*}
Additionally, for all $i, k\in [T]$ so that $i\neq k$ we have,
\begin{equation}
\begin{aligned}
&\boldsymbol{\xi}_{i}^{\top}\boldsymbol{\xi}_{i}=\sum_{j=1}^{d}{\xi}_{i,j}^{2}\xrightarrow[]{P}\trace\left(\boldsymbol{\Sigma}\right), \quad \boldsymbol{\xi}_{i}^{\top}\boldsymbol{\xi}_{k}=\sum_{j=1}^{d}{\xi}_{i,j}{\xi}_{k,j}\xrightarrow[]{P}0, \\& \quad\boldsymbol{\xi}_{i}^{\top}\mathbf w^\star=\sum_{j=1}^{d}{\xi}_{i,j}w^\star_j\xrightarrow[]{P}0,\quad\boldsymbol{\xi}_{i}^{\top}\mathbf w_0=\sum_{j=1}^{d}{\xi}_{i,j}{w_0}_j\xrightarrow[]{P}0,
\end{aligned}
\label{eq:xi:id}
\end{equation}

where the convergence in probability in each one of the cases follows by the uniform integrability of teacher correlations and from independence across coordinates with $\E[\boldsymbol{\xi}_i]=\E[\boldsymbol{\xi}_k]=0$, and $\mathrm{Cov}(\boldsymbol{\xi}_i)=\mathbf \Sigma.$
Using the identities in \cref{eq:xi:id}, we evaluate the resulting expression
by separating the analysis into three cases.

In the special case, when $i=T$:
\begin{align}
&\left\Vert \mathbf{w}^{\mathrm{esgd}}_{T}-\mathbf{w}_{i}^{\star}\right\Vert^{2}\notag
\xrightarrow[]{P}\left(a'^{T}-2a'^{T-1}b'\right)\trace\left(\boldsymbol{\Sigma}\right)+a'^{T}\left\Vert \mathbf{w}^{\star}-\mathbf{w}_{0}\right\Vert ^{2}
+\\&\hspace{12em}c'\labelnoise \frac{1-a'^{T}}{1-a'}+2\trace\left(\boldsymbol{\Sigma}\right)\left(\sum_{k=2}^{T}a'^{T-k+1}-b'\sum_{k=3}^{T}a'^{T-k+1}\right)\notag\\&=\left(a'^{T}-2a'^{T-1}b'\right)\trace\left(\boldsymbol{\Sigma}\right)+a'^{T}\left\Vert \mathbf{w}^{\star}-\mathbf{w}_{0}\right\Vert ^{2}\notag\\&\hspace{3em}+c'\labelnoise\frac{1-a'^{T}}{1-a'}+2\trace\left(\boldsymbol{\Sigma}\right)\left(\frac{a'\left(1-a'^{T-1}\right)}{1-a'}-b'\frac{a'\left(1-a'^{T-2}\right)}{1-a'}\right).
\label{eq:i=t}
\end{align}

When $1<i<T$:
\begin{align}
&\left\Vert\mathbf{w}^{\mathrm{esgd}}_{T}-\mathbf{w}_{i}^{\star}\right\Vert^{2}\xrightarrow[]{P}\left(a'^{T}-2a'^{T-1}b'\right)\trace\left(\boldsymbol{\Sigma}\right)+a'^{T}\left\Vert \mathbf{w}^{\star}-\mathbf{w}_{0}\right\Vert ^{2}+c'\labelnoise\frac{1-a'^{T}}{1-a'}\notag\\&\hspace{2em}+2\trace\left(\boldsymbol{\Sigma}\right)\left(\frac{a'\left(1-a^{T-1}\right)}{1-a'}-b'\frac{a'\left(1-a^{T-2}\right)}{1-a'}+1\right)+2\left(b'^{T-i+1}-b'^{T-i}-b'\right)\trace\left(\boldsymbol{\Sigma}\right).
\label{eq:i>1}
\end{align}

When $i=1$:
\begin{align}
&\left\Vert \mathbf{w}^{\mathrm{esgd}}_{T}-\mathbf{w}_{i}^{\star}\right\Vert^{2}\xrightarrow[]{P}\left(a'^{T}-2a'^{T-1}b'\right)\trace\left(\boldsymbol{\Sigma}\right)+a'^{T}\Vert\mathbf{w}^{\star}-\mathbf{w}_0\Vert^{2}+2b'^{T}\trace\left(\boldsymbol{\Sigma}\right)+c'\labelnoise\frac{1-a'^{T}}{1-a'}\notag\\&\hspace{2em}+2\trace\left(\boldsymbol{\Sigma}\right)\left(\frac{a'\left(1-a'^{T-1}\right)}{1-a'}-b'\frac{a'\left(1-a'^{T-2}\right)}{1-a'}+1\right)+2\bigl(b'^{T}-b'^{T-1}-b'\bigr)\trace\left(\boldsymbol{\Sigma}\right).
\label{eq:t>1}
\end{align}

Using \cref{eq:i=t,eq:i>1,eq:t>1} when summing over all cases,
\begin{align}
&\E\left[G_{T}\right]=\frac{1}{T}\sum_{i=1}^{T}\E\left\Vert \mathbf{w}^{\mathrm{esgd}}_{T}-\mathbf{w}_{i}^{\star}\right\Vert ^{2}\xrightarrow[]{P}&\notag\\&\left(a'^{T}-2a'^{T-1}b'\right)\trace\left(\boldsymbol{\Sigma}\right)+a^{T}\left\Vert \mathbf{w}^{\star}-\mathbf{w}_{0}\right\Vert^{2}+2\trace\left(\boldsymbol{\Sigma}\right)\left(\frac{a\left(1-a'^{T-1}\right)}{1-a'}-b'\frac{a'\left(1-a'^{T-2}\right)}{1-a'}\right)\notag+\\&c\labelnoise\frac{1-a'^{T}}{1-a'}+\frac{2}{T}\left(2b'^{T}+T-1-Tb'\right)\trace\left(\boldsymbol{\Sigma}\right)\label{eq:finite:t}\\=&c'\labelnoise\frac{1-a'^{T}}{1-a'}+a'^{T}\left\Vert \mathbf{w}^{\star}-\mathbf{w}_{0}\right\Vert ^{2}+\notag\\&2\trace(\boldsymbol{\Sigma})\prn{1-b'+\frac{2b'^{T}-1}{T}+\frac{a'}{1-a'}\left(1-a'^{T-1}-b\left(1-a'^{T-2}\right)\right)+\frac{a^{T-1}\left(a'-2b\right)}{2}}
\notag\\
=&c'\labelnoise\frac{1-a^{T}}{1-a}+a^{T}\left\Vert \mathbf{w}^{\star}-\mathbf{w}_{0}\right\Vert ^{2}+\notag\\&2\trace(\boldsymbol{\Sigma})\prn{1-b'+\frac{2b'^{T}-1}{T}+\frac{a'}{1-a'}\prn{1-b'+a'^{T-2}\prn{b'-a'}}+\frac{a'^{T-1}\left(a'-2b'\right)}{2}}%
\notag\\
=&c\labelnoise\frac{1-a'^{T}}{1-a'}+a'^{T}\left\Vert \mathbf{w}^{\star}-\mathbf{w}_{0}\right\Vert ^{2}+\notag\\&2\trace(\boldsymbol{\Sigma})\prn{\frac{(1-b')(1-a')}{1-a'}+\frac{a'(1-b')}{1-a'}+\frac{2b'^{T}-1}{T}+\frac{a'^{T-1}}{1-a'}\prn{b'-a'}+\frac{a'^{T-1}\left(a'-2b'\right)}{2}}
\notag\\
=&c\labelnoise\frac{1-a'^{T}}{1-a'}+a'^{T}\left\Vert \mathbf{w}^{\star}-\mathbf{w}_{0}\right\Vert^{2}+2\trace(\boldsymbol{\Sigma})\prn{\frac{1-b'}{1-a'}+\frac{2b'^{T}-1}{T}+\frac{a'^{T}\left(-a'+2b'-1\right)}{2\left(1-a'\right)}}.
\notag
\end{align}
\end{proof}
\newpage
\subsection{Proof of \cref{thm:opt}: Optimal Step Size in Early-Stopped GD} \label{pr:opt}
\begin{recall}[\cref{thm:opt}]
Under the i.i.d.~teacher setting (\cref{assump:iid_teachers}) with a nontrivial mean teacher $\mathbf w^\star\neq\mathbf{w}_0$\footnote{Having $\mathbf w^\star=\mathbf{w}_0$ would directly imply optimal step size $\gamma^\star\to0$.}, as $n,d\to\infty$ the optimal fixed step size $\gamma^\star$ that minimizes the generalization loss of the early-stopped GD algorithm (\cref{schm:early_stopped}) after $T$ tasks satisfies $\gamma^\star\asymp \frac{\ln T}{NT}$. Specifically, for all sufficiently large $T>T_0=\Theta\big(\frac{1}{\epsilon^2\ln^2(1/\epsilon)}\big)$,
the optimal parameters satisfy
\begin{align*}   
\frac{1}{1+\epsilon}\frac{\ln\left(T\right)+\ln\left(\Gamma\right)}{2N\featnoise T}
<\gamma^\star
<\frac{1}{1-\epsilon}\frac{\ln\left(T\right)+\ln\left(\Gamma\right)}{2N\featnoise T},
\end{align*}
where $\Gamma=\frac{4\ratio\featnoise\left\Vert \mathbf{w}^{\star}-\mathbf{w}_0\right\Vert^{2}}{v_{z}+\featnoise\mathrm{Tr}(\boldsymbol{\Sigma})\left(\ratio+1\right)}$. In the degenerate noiseless case of $\mathrm{Tr}(\boldsymbol\Sigma)=\labelnoise=0$, the optimal choice is any $\gamma^\star\in \big(0,2\featnoise^{-1}\left(1+{\alpha}^{-1/2}\right)^{-2}\big)$ and $N\to\infty$.

Moreover, if $\gamma\asymp \frac{\ln T}{NT}$ the resulting generalization loss of the early-stopped GD matches that of optimally tuned L2 regularization up to $O(1/T)$ factor.
\end{recall}
\begin{proof}
The proof of the main theorem proceeds as follows. We first show that no optimal regularizer can be obtained in the regime where $T$ is 'significantly' larger than $\lambda$. Next, by employing exact Laurent analysis in the regime where $T$ increases more slowly than $\lambda$, we derive the approximately optimal value for the step size. Finally, in a suitable regime, we find a close connection between the ridge objective and the objective of early-stopped gradient descent.

Recall \cref{thm:closedloss},
\begin{align*}
    &G^{\mathrm{esgd}}_T 
     \xrightarrow{P} \labelnoise c'\frac{1-a'^{T}}{1-a'}
     +\frac{1}{T}\sum_{i=1}^{T}
     \left\Vert \mathbf{w}_{T}^{\star}-\mathbf{w}_{i}^{\star}\right\Vert^{2}
  %    \\[6pt]
  % &
  % \quad 
  + 2\big(\mathbf{w}_{T}^{\star}-\bar{\mathbf{w}}^{\star}\big)^{\top}
     \sum_{i=1}^{T} b'^{T-i+1}
     \big(\mathbf{w}_{i-1}^{\star}-\mathbf{w}_{i}^{\star}\big)\\[4pt]
  &
  \quad +\sum_{i=1}^{T}\sum_{i'=1}^{T}
     a'^{T-\max(i,i')+1}b'^{|i-i'|}
     \big(\mathbf{w}_{i-1}^{\star}-\mathbf{w}_{i}^{\star}\big)^{\top}
     \big(\mathbf{w}_{i'-1}^{\star}-\mathbf{w}_{i'}^{\star}\big)
     ,
\end{align*}
For convenience, throughout the proof we use \(G_T^{\mathrm{esgd}},G_T^{\mathrm{ridge}}\) to denote their corresponding asymptotic deterministic equivalent, to which they converge in probability as \(n,d \to \infty\).

To further simplify this expression in the i.i.d.~teacher setting we use \cref{cor:closed}.
We conclude that under early-stopped GD (\cref{schm:early_stopped}) the expected generalization after $T>1$ steps satisfies,
\begin{align}
\label{eq:37}
G^{\mathrm{esgd}}_{T}=\labelnoise{c'}\frac{1-a'^{T}}{1-a'}+a'^{T}\left\Vert \mathbf{w}^{\star}-\mathbf{w}_{0}\right\Vert ^{2}+2\trace(\boldsymbol{\Sigma})\prn{\frac{1-b'}{1-a'}+\frac{2b'^{T}-1}{T}+\frac{a'^{T}\prn{-a' +2b' -1}}{2(1-a')}}.
\end{align}
By taking the Taylor expansion around $N\gamma\to 0$ for $a',b',c'$, as in \cref{eq:33}, we obtain,
\begin{equation}
\label{eq:19}
\begin{aligned}
&a'(\gamma,N)=1-2N\featnoise\gamma
+N(2N-1)\featnoise^{2}\left(1+\frac{1}{\ratio}\right)\gamma^{2}
+O\big((N\gamma)^{3}\big),
\\&
b'(\gamma,N)=1-N\featnoise\gamma
+\frac{N(N-1)}{2}\featnoise^{2}\left(1+\frac{1}{\ratio}\right)\gamma^{2}
+O\big((N\gamma)^{3}\big),
\\&
c'(\gamma,N)
=
\frac{N^{2}}{\ratio}\featnoise\gamma^{2}
+O\big((N\gamma)^{3}\big).
\end{aligned}
\end{equation}

We first show that the optimal step size $\gamma$ and number of gradient steps $N$, can not be in any region where $\frac{1}{N\gamma}:=\lambda \in o(T^\rho)$, for all $\rho<1$.
First, we consider the limit of the objective (\cref{eq:37}), i.e.
\begin{align*}
\lim\limits_{T\to\infty}G^{\mathrm{esgd}}_T=c'\labelnoise\frac{1}{1-a'}  +2\trace(\boldsymbol{\Sigma})\prn{\frac{1-b'}{1-a'}},
\end{align*}
Taking the derivative of the limiting expression,  \begin{align}
\label{eq:56}\frac{\dd}{\dd\lambda}\lim\limits_{T\to\infty}G^{\mathrm{esgd}}_{T}&= \underbrace{\frac{\dd}{\dd\lambda}\left(c'\right)\labelnoise\frac{1}{1-a'}+\frac{\dd}{\dd\lambda}\frac{1}{1-a'}c'\labelnoise}_{\text{term 1}}+\underbrace{2\trace\left(\mSigma\right)\left(\frac{1-b'}{(1-a')^{2}}\frac{\dd}{\dd\lambda}a'-\frac{1}{1-a'}\frac{\dd}{\dd\lambda}b'\right)}_{\text{term 2}}.\end{align}

Using the Taylor expansion in \cref{eq:19}, and since $a'$, $b'$, $c'$ are analytic at $\gamma=0$ (all are polynomial), then using \cref{lem:differentiable} and the relation $\lambda = \frac{1}{N\gamma}$, term 1 becomes,
\begin{align}
&\frac{\dd}{\dd\lambda}\left(c'\right)\labelnoise\frac{1}{1-a'}+\frac{\dd}{\dd\lambda}\left(\frac{1}{1-a'}\right)c'\labelnoise\notag&\\&=\left(\frac{-2\featnoise}{\lambda^{3}\ratio}+O(\lambda^{-4})\right)\frac{\labelnoise}{\frac{2\featnoise}{\lambda}+O(\lambda^{-2})}+\frac{\dd a'}{\dd\lambda}\frac{\dd}{\dd a'}\left(\frac{1}{1-a'}\right)c'\labelnoise\notag\\&=\left(\frac{-2\featnoise}{\lambda^{3}\ratio}+O(\lambda^{-4})\right)\frac{\labelnoise}{\frac{2\featnoise}{\lambda}+O(\lambda^{-2})}+\left(\frac{2\featnoise}{\lambda^{2}}+O(\lambda^{-3})\right)\frac{1}{\left(\frac{2\featnoise}{\lambda}+O(\lambda^{-2})\right)^{2}}\left(\frac{\featnoise}{\lambda^{2}\ratio}+O(\lambda^{-3})\right)\labelnoise\notag\\&=\left(\frac{-\labelnoise}{\lambda^{2}\ratio}+O(\lambda^{-3})\right)+\left(\frac{1}{2\featnoise}+O(\lambda^{-2})\right)\left(\frac{\featnoise}{\lambda^{2}\ratio}+O(\lambda^{-3})\right)\labelnoise\notag\\&=\frac{-\labelnoise}{\lambda^{2}\ratio}+\frac{\labelnoise}{2\lambda^{2}\ratio}+O\left(\lambda^{-3}\right)=\frac{-\labelnoise}{2\lambda^{2}\ratio}+O\left(\lambda^{-3}\right).\label{eq:20}
\end{align}
And term 2 becomes,
\begin{align}
&2\trace(\mSigma)\Bigg(\frac{1-b'}{(1-a')^{2}}\frac{\dd a'}{\dd\lambda}-\left(\frac{1}{1-a'}\right)\frac{\dd b'}{\dd\lambda}\Bigg)\notag\\&=2\trace(\mSigma)\Bigg(\left[\frac{\featnoise}{\lambda}-\frac{N-1}{2N}\featnoise^{2}\left(1+\frac{1}{\ratio}\right)\frac{1}{\lambda^{2}}+O(\lambda^{-3})\right]\notag\\&\qquad\left[\frac{\lambda^{2}}{4\featnoise^{2}}+\frac{2N-1}{4N}\left(1+\frac{1}{\ratio}\right)\frac{\lambda}{\featnoise}+O(1)\right]\left[\frac{2\featnoise}{\lambda^{2}}-\frac{2(2N-1)}{N}\featnoise^{2}\left(1+\frac{1}{\ratio}\right)\frac{1}{\lambda^{3}}+O(\lambda^{-4})\right]\notag\\&\qquad-\left[\frac{\lambda}{2\featnoise}+\frac{2N-1}{4N}\left(1+\frac{1}{\ratio}\right)+O(\lambda^{-1})\right]\left[\frac{\featnoise}{\lambda^{2}}-\frac{N-1}{N}\featnoise^{2}\left(1+\frac{1}{\ratio}\right)\frac{1}{\lambda^{3}}+O(\lambda^{-4})\right]\Bigg)\notag\\&=2\trace(\mSigma)\notag\\&\quad\Bigg(\left[\frac{1}{2\lambda}+\frac{1}{\lambda^{2}}\left(-\frac{N-1}{4N}\featnoise\left(1+\frac{1}{\ratio}\right)+\frac{2N-1}{4N}\featnoise\left(1+\frac{1}{\ratio}\right)-\frac{2N-1}{2N}\featnoise\left(1+\frac{1}{\ratio}\right)\right)\right]\notag\\&\hspace{4em}-\left[\frac{1}{2\lambda}+\frac{1}{\lambda^{2}}\left(\frac{2N-1}{2N}\featnoise\left(1+\frac{1}{\ratio}\right)-\frac{N-1}{2N}\featnoise\left(1+\frac{1}{\ratio}\right)\right)\right]+O(\lambda^{-3})\Bigg)\notag\\&=2\trace(\mSigma)\notag\\&\Bigg(\cancel{\frac{1}{2\lambda}}-\cancel{\frac{1}{2\lambda}}+\frac{1}{\lambda^{2}}\left(-\frac{N-1}{4N}+\frac{2N-1}{2N}-\frac{2N-1}{2N}-\frac{2N-1}{4N}+\frac{N-1}{2N}\right)\featnoise\left(1+\frac{1}{\ratio}\right)+O(\lambda^{-3})\Bigg)\notag\\&=2\trace(\mSigma)\left(-\frac{1}{4}\featnoise\left(1+\frac{1}{\ratio}\right)\frac{1}{\lambda^{2}}+O(\lambda^{-3})\right)=-\frac{\trace(\mSigma)}{2}\featnoise\left(1+\frac{1}{\ratio}\right)\frac{1}{\lambda^{2}}+O(\lambda^{-3}).\label{eq:29}
\end{align}
Therefore, the derivative is strictly negative (for all sufficiently large $\lambda$) and is of order $\Theta(\lambda^{-2})$.
We aim to show that the distance between $\frac{\dd}{\dd \lambda}G^{\mathrm{esgd}}_T$ is exponentially small in the aforementioned regime, and thereby showing that $\frac{\dd}{\dd \lambda}G^{\mathrm{esgd}}_T$ is of the same order as $\frac{\dd}{\dd \lambda}\lim\limits_{T\to\infty}G^{\mathrm{esgd}}_T$.
First note that,
\begin{equation}
\begin{aligned}
&G^{\mathrm{esgd}}_{T}-\lim_{T\to\infty}G^{\mathrm{esgd}}_{T}\\&=\labelnoise{c'}\frac{-a'^{T}}{1-a'}+a'^{T}\left\Vert \mathbf{w}^{\star}-\mathbf{w}_{0}\right\Vert ^{2}+2\trace(\boldsymbol{\Sigma})\prn{\frac{2b'^{T}-1}{T}+\frac{a'^{T}\prn{-a'+2b'-1}}{2(1-a')}}.
\end{aligned}
\label{eq:38}
\end{equation}
Then by taking the derivative,
\begin{align*}
&\frac{\dd}{\dd\lambda}\left[G^{\mathrm{esgd}}_T-\lim\limits_{T\to\infty}G^{\mathrm{esgd}}_T\right]=-\frac{\dd}{\dd\lambda}\frac{a'^{T}}{1-a'}\labelnoise c'-\frac{a'^{T}}{1-a'}\labelnoise\frac{\dd}{\dd\lambda}c'+\frac{\dd a}{\dd\lambda}Ta'^{T-1}\left\Vert \mathbf{w}^{\star}-\mathbf{w}_{0}\right\Vert ^{2}\\&=\frac{Ta'^{T-1}\left(1-a'\right)+a'^{T}}{\left(1-a'\right)^{2}}\frac{\dd a'}{\dd\lambda}\labelnoise c'-\frac{a'^{T}}{1-a'}\labelnoise\frac{\dd c'}{\dd\lambda}+\frac{\dd a'}{\dd\lambda}Ta'^{T-1}\left\Vert \mathbf{w}^{\star}-\mathbf{w}_{0}\right\Vert ^{2}\end{align*}\begin{align*}&\hspace{4em}+2\trace(\boldsymbol{\Sigma})\Bigg(2b'^{T-1}\frac{\dd b'}{\dd\lambda}+\frac{1}{2(1-a')}\Bigg[Ta'^{T-1}\frac{\dd a'}{\dd\lambda}\big(-a'+2b'-1\big)+a'^{T}\left(-\frac{\dd a'}{\dd\lambda}+2\frac{\dd b'}{\dd\lambda}\right)\Bigg]\\&\hspace{8em}+\frac{a'^{T}\big(-a'+2b'-1\big)}{2(1-a')^{2}}\frac{\dd a'}{\dd\lambda}\Bigg)
\end{align*}
Since $a',b',c'$ and their corresponding derivatives are bounded (for all sufficiently large $\lambda$) by \cref{eq:19,lem:differentiable}, and $a',b'=1-\Theta(\frac{1}{\lambda})$ by \cref{eq:19}, then $a'^T,b'^T=\left(1-\Theta\left(\frac{1}{\lambda}\right)\right)^T\leq \exp\left(-\beta\frac{T}{\lambda}\right)$, for some $\beta>0$ (uniformly). So, 
\begin{align*}
\frac{\dd}{\dd\lambda}\left[G^{\mathrm{esgd}}_T-\lim\limits_{T\to\infty}G^{\mathrm{esgd}}_T\right]\leq \exp\left(-\beta\frac{T}{\lambda}\right)\cdot \text{Poly}(\lambda^{-1}, T).
\end{align*}
where $\mathrm{Poly}(\lambda^{-1},T)$ denotes a polynomial in $(\lambda^{-1},T)$ (with bounded degree).
Assume $\lambda=O(T^\rho)$ for some $\rho<1$.
Then $\exp(-\beta T/\lambda)\mathrm{Poly}(\lambda^{-1},T)=o(\lambda^{-2})$. Since
$\frac{\dd}{\dd\lambda}\lim\limits_{T\to\infty}G^{\mathrm{esgd}}_T=\Omega(\lambda^{-2}),$ by \cref{eq:20},
it follows that
$\frac{\dd}{\dd\lambda}G^{\mathrm{esgd}}_T=\frac{\dd}{\dd\lambda}\lim\limits_{T\to\infty}G^{\mathrm{esgd}}_T+o(\lambda^{-2}),$
and in particular the finite-$T$ derivative has the same leading-order behavior. 

Now we show, that if $\frac{1}{N\gamma}\in \omega\left(T^{\frac{1}{2}}\right)$, then the difference between the objectives' derivative is not 'large', i.e. $\left|\frac{\dd}{\dd \lambda}G^{\mathrm{esgd}}_T-\frac{\dd}{\dd \lambda}G^{\mathrm{ridge}}_T \right|=o\left(\frac{\dd}{\dd \lambda} G^{\mathrm{ridge}}_T \right)$ (almost everywhere). We develop the corresponding Laurent expansion using the relation of $\frac{1}{N \gamma}=\lambda$ and \citet[Lemma D.18]{karpel2026optimal},
\begin{equation}
\label{eq:42}
\begin{aligned}
&G_{T}^{\mathrm{esgd}}-\lim\limits_{T\to\infty}G_{T}^{\mathrm{esgd}}\\&=a'^{T}\left\Vert\mathbf{w}^{\star}-\mathbf{w}_{0}\right\Vert ^{2}+\labelnoise c'\frac{-a'^{T}}{1-a'}+2\trace(\boldsymbol{\Sigma})\prn{\frac{2b'^{T}-1}{T}+\frac{a'^{T}\prn{-a'+2b'-1}}{2(1-a')}}&\\&=\exp\left(-\frac{2Tv_{x}}{\lambda}\right)\left(1+O\left(T\lambda^{-2}\right)\right)\Vert\mathbf{w}^{\star}-\mathbf{w}_0\Vert^{2}\\&\hspace{2em}-v_{z}\left(\frac{\lambda}{2v_{x}}+O\left(1\right)\right)\left(\frac{\featnoise}{\ratio\lambda^{2}}+O\left(\lambda^{-3}\right)\right)\exp\left(-\frac{2Tv_{x}}{\lambda}\right)\left(1+O\left(T\lambda^{-2}\right)\right)\\&\hspace{4em}+\frac{2\mathrm{Tr}(\boldsymbol{\Sigma})}{T}\left(2\exp\left(-\frac{Tv_{x}}{\lambda}\right)\left(1+O\left(T\lambda^{-2}\right)\right)-1\right)\\&\hspace{6em}+\mathrm{Tr}(\boldsymbol{\Sigma})\left(\exp\left(-\frac{2Tv_{x}}{\lambda}\right)\left(1+O\left(T\lambda^{-2}\right)\right)\right)O\left(\lambda^{-1}\right).
\end{aligned}
\end{equation}

An analogous statement holds for $G^{\mathrm{ridge}}_T -\lim\limits_{T\to\infty} G^{\mathrm{ridge}}_T $.
In this case, one replaces $a',b',c'$ by $a,b,c$, respectively, because the corresponding coefficients admit identical expansions\footnote{While their expression is stated for the expectation (including expectation on the features and labels), their analysis can be naturally extended to yield the corresponding convergence-in-probability statement under the same asymptotic regime.}. Furthermore, the limiting objective's derivative attains the following expansion $$\frac{\dd}{\dd\lambda}\lim\limits_{T\to\infty}G^{\mathrm{ridge}}_T =-\frac{\featnoise}{2}\left(\frac{\labelnoise}{\featnoise}+\trace(\mSigma)(\ratio+1)\right)\lambda^{-2}
+O(\lambda^{-3}),$$ by~\citet[Lemma E.3]{karpel2026optimal}, which coincides with the expansion we had for $\frac{\dd}{\dd\lambda}\lim\limits_{T\to\infty}G^{\mathrm{esgd}}_T$ (under the substitution of $\lambda\alpha$, see \cref{lem:updates,eq:20}). Therefore, $G^{\mathrm{ridge}}_T $ admits the same expansion as $G^{\mathrm{esgd}}_T$, under the substitution of $N\gamma=\frac{1}{\lambda}$.

So far our analysis accommodated large enough $\lambda$ (or equivalently, small enough $N\gamma$), to deal with the case of $\gamma\in\Theta (1)$ but is small enough to ensure convergence (i.e. $\gamma\in \big(0,2\featnoise^{-1}\left(1+{\alpha}^{-1/2}\right)^{-2}\big)$), it is enough to show that these points can not be the optimal of the objective $\lim\limits_{T\to\infty}G^{\mathrm{esgd}}_T$, that is, there exists a point which attains a strictly smaller error, then for all sufficiently large $T$, it will remain non-optimal for the finite $T$ objective $G^{\mathrm{esgd}}_T$, since the two objectives can be made arbitrary close (see \cref{eq:38}). Let $\gamma \in \Theta(1)$ where $\gamma\in \big(0,2\featnoise^{-1}\left(1+{\alpha}^{-1/2}\right)^{-2}\big)$.

Recall that $a'$ and $b'$ are the corresponding $2N$'th and $N$'th moment of the MP law. Therefore, by Jensen's inequality $a'\geq b'^2$, and,
\begin{align*}
    \frac{1-b'}{1-a'}\geq \frac{1-b'}{1-b'^2}=\frac{1}{1+b'}.
\end{align*}
This yields, 
\begin{align*}
    \lim\limits_{T\to \infty}G^{\mathrm{esgd}}_T \geq \labelnoise c'\frac{1}{1-a'}+\frac{2}{1+b'}\trace(\mathbf \Sigma)>\trace(\mathbf \Sigma).
\end{align*}
This follows since $a', b'<1$ in the range $\gamma\in \big(0,2\featnoise^{-1}\left(1+{\alpha}^{-1/2}\right)^{-2}\big)$ (and $\labelnoise, c'>0$ always). 

Denote by $\Delta>0$ as $\lim\limits_{T\to \infty}\E G^{\text{esgd}}_T-\trace(\mathbf \Sigma)=\Delta$.
As $\lambda \to \infty$ we have $\lim\limits_{T\to \infty}G^{\mathrm{esgd}}_T=\trace\left(\mathbf \Sigma\right)$ (as $\frac{1-b'}{1-a'}\to\frac{1}{2}$, and $\frac{c'}{1-a'}\to 0$ which follows immediately by the Taylor approximations of $a',b',c'$), take $\tilde \lambda$ large enough (or $N\tilde \gamma$ small enough) such that $\left|\lim\limits_{T\to \infty}G^{\mathrm{esgd}}_T-\trace\left(\mathbf \Sigma\right)\right| < \frac{\Delta}{2}$, then 
$$
G^{\mathrm{esgd}}_T(\gamma)-G^{\mathrm{esgd}}_T(\tilde\gamma)> \left(\lim\limits_{T\to \infty}G^{\mathrm{esgd}}_T(\gamma)-\trace\left(\mathbf \Sigma\right)\right)-\left(\lim\limits_{T\to \infty}G^{\mathrm{esgd}}_T(\tilde\gamma)-\trace\left(\mathbf \Sigma\right)\right)\geq \frac{\Delta}{2}>0
$$

therefore there exists $\lambda$ sufficiently large that attains a smaller error.
If both $\mathrm{Tr}(\boldsymbol\Sigma)=0$ and $\labelnoise=0$, then by \cref{thm:closedloss},
$$
G_T^{\mathrm{esgd}}
= \bigl\Vert\mathbf{w}_T - \mathbf{w}^\star\bigr\Vert^{2}
= (a')^{T}\bigl\Vert\mathbf{w}^{\star}-\mathbf{w}_{0}\bigr\Vert^{2}.
$$
When $N\to\infty$ and $\gamma$ is finite (but ensures the convergence of early-stopped GD algorithm, i.e. any $\gamma\in \big(0,2\featnoise^{-1}\left(1+{\alpha}^{-1/2}\right)^{-2}\big)$) then $a'=\lim_{N\to\infty}\frac{1}{d}\operatorname{Tr}\left(\mathbf I-\gamma\frac{1}{n}\mathbf X^\top\mathbf X\right)^{2N}=(1-\ratio)_+$
since for every nonzero eigenvalue $\lambda_i$ of $\frac1n\mathbf X^\top\mathbf X$, we have $(1-\gamma\lambda_i)^{2N}\to 0$, while zero eigenvalues remain equal to one. Hence $a'$ attains its minimal possible value, corresponding to the irreducible nullspace component.

Otherwise, the optimal step size
$\frac{1}{N\gamma^\star}=\Theta(T/\ln T)$. Recall \cref{eq:38},
\begin{align*}
G_{T}^{\mathrm{esgd}}-\lim\limits_{T\to\infty}G_{T}^{\mathrm{esgd}}
&=
\underbrace{\frac{2\mathrm{Tr}\left({\mSigma}\right)}{T}\left(2(b')^{T}-1\right)}_{\text{term 1}}
+\underbrace{\mathrm{Tr}\left({\mSigma}\right)(a')^T\left(1-2\frac{1-b'}{1-a'}\right)}_{\text{term 2}}\\
&\quad
+\underbrace{(a')^{T}\left\Vert \mathbf{w}^{\star}-\mathbf{w}_{0}\right\Vert ^{2}+\labelnoise c'\frac{-(a')^{T}}{1-a'}}_{\text{term 3}}.
\end{align*}

We develop a Laurent expansion for each term. Term 1 becomes,
\begin{align}
\frac{2\mathrm{Tr}(\boldsymbol{\mSigma})}{T}\left(2(b')^{T}-1\right)
&=
\frac{2\mathrm{Tr}(\boldsymbol{\mSigma})}{T}\left(
2\exp\left(-\frac{v_x T}{\lambda}\right)
\left(1+O(T\lambda^{-2})\right)
-1\right).
\label{eq:/t:new}
\end{align}

Term 2 is,
\begin{align}
&\mathrm{Tr}\left(\boldsymbol{\mSigma}\right)(a')^{T}\left[1-2\left(\frac{1-b'}{1-a'}\right)\right]\notag\\
&=
\mathrm{Tr}\left(\boldsymbol{\mSigma}\right)
\left(
\exp\left(-\frac{2v_xT}{\lambda}\right)
\left(1+O\left(T\lambda^{-2}\right)\right)
\right)
O\left(\lambda^{-1}\right).
\label{eq:at:new}
\end{align}

Term 3 is,
\begin{align}
&(a')^{T}\Vert\mathbf{w}^{\star}-\mathbf{w}_{0}\Vert^{2}+v_zc'\frac{-(a')^{T}}{1-a'}\notag&\\&=\exp\left(-\frac{2v_{x}T}{\lambda}+O\left(T\lambda^{-2}\right)\right)\Vert\mathbf{w}^{\star}-\mathbf{w}_{0}\Vert^{2}\notag\\&\quad-\frac{\lambda v_{z}}{4\alpha v_{x}^{2}}\left(\frac{2\alpha}{\left(\frac{\lambda\alpha}{v_{x}}+\alpha+1\right)^{2}}+O\left(\lambda^{-4}\right)\right)\exp\left(-\frac{2v_{x}T}{\lambda}+O\left(T\lambda^{-2}\right)\right)\notag\\&=\exp\left(-\frac{2v_{x}T}{\lambda}\right)\left(1+O\left(T\lambda^{-2}\right)\right)\Vert\mathbf{w}^{\star}-\mathbf{w}_{0}\Vert^{2}\notag\\&\qquad-\frac{v_{z}}{2\ratio\lambda}\left(1+O(\lambda^{-1})\right)\exp\left(-\frac{2v_{x}T}{\lambda}\right)\left(1+O\left(T\lambda^{-2}\right)\right).
\label{eq:st:new}
\end{align}

Using \cref{lem:differentiable}, we differentiate with respect to $\lambda$, term by term.

Differentiating \cref{eq:/t:new},
\begin{align}
&\frac{\dd}{\dd\lambda}\frac{2\mathrm{Tr}(\boldsymbol{\mSigma})}{T}\left(2(b')^{T}-1\right)\notag\\&=\frac{2\mathrm{Tr}(\boldsymbol{\mSigma})}{T}\left(2\frac{v_{x}T}{\lambda^{2}}\exp\left(-\frac{v_{x}T}{\lambda}\right)\left(1+O\left(T\lambda^{-2}\right)\right)\right)\notag\\&\quad+\frac{2\mathrm{Tr}(\boldsymbol{\mSigma})}{T}\left(2\exp\left(-\frac{v_{x}T}{\lambda}\right)O\left(T\lambda^{-3}\right)\right).
\label{eq:dtdz:new}
\end{align}

Differentiating \cref{eq:at:new},
\begin{align}
&\frac{\dd}{\dd\lambda}\mathrm{Tr}\left(\boldsymbol{\mSigma}\right)(a')^{T}\left[1-2\left(\frac{1-b'}{1-a'}\right)\right]\notag\\&=\mathrm{Tr}\left(\boldsymbol{\mSigma}\right)\Bigg(\frac{2v_{x}T}{\lambda^{2}}\exp\left(-\frac{2v_{x}T}{\lambda}\right)\left(1+O\left(T\lambda^{-2}\right)\right)\notag\\&\hspace{5em}+\exp\left(-\frac{2v_{x}T}{\lambda}\right)O\left(T\lambda^{-3}\right)\Bigg)O(\lambda^{-1})\notag\\&\quad+\mathrm{Tr}\left(\boldsymbol{\mSigma}\right)\left(\exp\left(-\frac{2v_{x}T}{\lambda}\right)\left(1+O\left(T\lambda^{-2}\right)\right)\right)O(\lambda^{-2}).
\label{eq:dzdtr:new}
\end{align}

Differentiating \cref{eq:st:new},
\begin{align}
&\frac{\dd}{\dd\lambda}\left[(a')^{T}\Vert\mathbf{w}^{\star}-\mathbf{w}_{0}\Vert^{2}+f\frac{-(a')^{T}}{1-a'}\right]\notag&\\&=\frac{v_{z}}{2\ratio\lambda^{2}}\left(1+O(\lambda^{-1})\right)\exp\left(-\frac{2v_{x}T}{\lambda}\right)\left(1+O\left(T\lambda^{-2}\right)\right)\notag\\&\quad-\frac{v_{z}}{2\ratio\lambda}\left(1+O(\lambda^{-1})\right)\Bigg(\frac{2v_{x}T}{\lambda^{2}}\exp\left(-\frac{2v_{x}T}{\lambda}\right)\left(1+O\left(T\lambda^{-2}\right)\right)+\exp\left(-\frac{2v_{x}T}{\lambda}\right)O\left(T\lambda^{-3}\right)\Bigg)\notag\\&\quad+\frac{2v_{x}T}{\lambda^{2}}\exp\left(-\frac{2v_{x}T}{\lambda}\right)\left(1+O\left(T\lambda^{-2}\right)\right)\Vert\mathbf{w}^{\star}-\mathbf{w}_{0}\Vert^{2}+\exp\left(-\frac{2v_{x}T}{\lambda}\right)O\left(T\lambda^{-3}\right)\Vert\mathbf{w}^{\star}-\mathbf{w}_{0}\Vert^{2}.
\label{eq:dzds:new}
\end{align}

Retaining only the leading-order terms,
\begin{align*}
&\frac{\dd}{\dd\lambda}\left[\mathbb{E}\left[G_{T}^{\mathrm{esgd}}\right]-\lim\limits_{T\to\infty}\mathbb{E}\left[G_{T}^{\mathrm{esgd}}\right]\right]=\frac{2 v_{x}T}{\lambda^{2}}\exp\left(-q\right)\Vert\mathbf{w}^{\star}-\mathbf{w}_0\Vert^{2}+\frac{4\mathrm{Tr}(\boldsymbol{\mSigma})}{\lambda^{2}} v_{x}\exp\left(-\frac{q}{2}\right),
\end{align*}
where $q=\frac{2v_{x}T}{\lambda}.$
Taking the Laurent expansion in \cref{eq:56}, yields
\[
\frac{\dd}{\dd\lambda}\lim\limits_{T\to\infty}G_{T}^{\mathrm{esgd}}=-\frac{1}{2}\left(\frac{v_{z}}{\ratio}+v_{x}\mathrm{Tr}(\boldsymbol{\Sigma})\left(1+\frac{1}{\ratio}\right)\right)\lambda^{-2}+O\left(\lambda^{-3}\right).
\]
Therefore, the stationary point satisfies
\begin{align*}
&\frac{2v_{x}T}{\lambda^{2}}\exp\left(-q\right)\Vert\mathbf{w}^{\star}-\mathbf{w}_{0}\Vert^{2}+\frac{4\mathrm{Tr}(\boldsymbol{\mSigma})}{\lambda^{2}}v_{x}\exp\left(-\frac{q}{2}\right)&\\&\hspace{16em}=\frac{1}{2}\left(\frac{v_{z}}{\ratio}+v_{x}\mathrm{Tr}(\boldsymbol{\Sigma})\left(1+\frac{1}{\ratio}\right)\right)\lambda^{-2}+O\left(\lambda^{-3}\right).
\end{align*}

Multiplying by $\frac{2\lambda^2}{v_xv_z}$, we define
\[
\widetilde{H}(q)=\frac{4 T\Vert\mathbf{w}^{\star}-\mathbf{w}_0\Vert^{2}}{v_{z}}e^{-q}+\frac{8\mathrm{Tr}(\mSigma)}{v_{z}}e^{-q/2}+z,
\]
where $z=-\frac{1}{v_{z}}\left(\frac{v_{z}}{v_{x}\ratio}+\mathrm{Tr}(\boldsymbol{\Sigma})\left(1+\frac{1}{\ratio}\right)\right).$
The derivative satisfies
\[
\widetilde{H}'(q)=-\frac{4 T\Vert\mathbf{w}^{\star}-\mathbf{w}_0\Vert^{2}}{v_{z}}e^{-q}-\frac{4\mathrm{Tr}(\mSigma)}{v_{z}}e^{-q/2}
<0.
\]
Hence \(\widetilde H\) is strictly decreasing.
Moreover, if \(q\in o(\ln T)\), then
\[
\widetilde{H}(q)=\frac{4 T\Vert\mathbf{w}^{\star}-\mathbf{w}_0\Vert^{2}}{v_{z}}e^{-q}+z+o(1),
\]
which is positive for all sufficiently large \(T\). On the other hand, if
\(q\in \omega(\ln T)\), then $\widetilde H(q)=z+o(1)<0.$
Therefore, it suffices to identify \(q^\star\) satisfying $\widetilde H(q^\star)\approx0.$
Retaining only the leading-order term in the regime
\(q^\star\in\Theta(\ln T)\), we obtain $\frac{4\alpha T\Vert\mathbf{w}^\star-\mathbf{w}_0\Vert^2}{v_z}e^{-q^\star}\approx-z,$ which gives $q^{\star}\approx\ln\frac{4 T\Vert\mathbf{w}^{\star}-\mathbf{w}_0\Vert^{2}}{-v_{z}z}.$
Let \(\epsilon>0\). Then,
\begin{align}
\label{eqq:1}
\widetilde{H}\left(\frac{q^{\star}}{1+\epsilon}\right)=-z\left(\frac{4 T\Vert\mathbf{w}^{\star}-\mathbf{w}_0\Vert^{2}}{-v_{z}z}\right)^{\frac{\epsilon}{1+\epsilon}}+z+o(1)\xrightarrow[T\to\infty]{}+\infty,
\end{align}
and hence $\widetilde H\left(\frac{q^\star}{1+\epsilon}\right)>0$ for all sufficiently large \(T\).
Similarly,
\begin{align}
\label{eqq:2}
\widetilde{H}\left(\frac{q^{\star}}{1-\epsilon}\right)=-z\left(\frac{4 T\Vert\mathbf{w}^{\star}-\mathbf{w}_0\Vert^{2}}{-v_{z}z}\right)^{-\frac{\epsilon}{1-\epsilon}}+z+o(1)\xrightarrow[T\to\infty]{}z<0,
\end{align}
and therefore $\widetilde H\left(\frac{q^\star}{1-\epsilon}\right)<0$ for all sufficiently large \(T\). This proves that the optimal parameters satisfy
\begin{align*}   
\frac{1}{1+\epsilon}\frac{\ln\left(T\right)+\ln\left(\frac{4\ratio\featnoise\left\Vert \mathbf{w}^{\star}-\mathbf{w}_{0}\right\Vert ^{2}}{v_{z}+\featnoise\mathrm{Tr}(\boldsymbol{\Sigma})\left(\ratio+1\right)}\right)}{2N\featnoise T}<\gamma^{\star}<\frac{1}{1-\epsilon}\frac{\ln\left(T\right)+\ln\left(\frac{4\ratio\featnoise\left\Vert \mathbf{w}^{\star}-\mathbf{w}_{0}\right\Vert ^{2}}{v_{z}+\featnoise\mathrm{Tr}(\boldsymbol{\Sigma})\left(\ratio+1\right)}\right)}{2N\featnoise T}.
\end{align*}

To satisfy $\tilde H\left(\frac{1}{1+\epsilon}q^\star\right)>0$, we begin with \cref{eqq:1}. In this case, the non-leading terms contribute $O\left(T^{-\frac{1}{2(1+\epsilon)}}\right)$ --- as $\exp(-q/2)$, the dominating term becomes $O\left(T^{-\frac{1}{2(1+\epsilon)}}\right)$.
\begin{align*}
&-z\cdot\left(\left(\frac{4 T\left\Vert \mathbf{w}^{\star}-\mathbf{w}_{0}\right\Vert ^{2}}{-v_{z}z}\right)^{\frac{\epsilon}{1+\epsilon}}-1\right)+O\left(T^{-\frac{1}{2(1+\epsilon)}}\right)\\&=-z\cdot\left(\exp\left(\frac{\epsilon}{1+\epsilon}\ln\left(\frac{4 T\left\Vert \mathbf{w}^{\star}-\mathbf{w}_{0}\right\Vert ^{2}}{-v_{z}z}\right)\right)-1\right)+O\left(T^{-\frac{1}{2(1+\epsilon)}}\right)\\
\explain{e^{x}-1\geq x}
&
\geq
-z\cdot\frac{\epsilon}{1+\epsilon}\ln\left(\frac{4 T\left\Vert \mathbf{w}^{\star}-\mathbf{w}_{0}\right\Vert ^{2}}{-v_{z}z}\right)+O\left(T^{-\frac{1}{2(1+\epsilon)}}\right)\\&\gtrsim-z\cdot\frac{\epsilon}{1+\epsilon}\ln\left(T\right)+O\left(T^{-\frac{1}{2(1+\epsilon)}}\right).
\end{align*}
There exists a constant $C_1$ such that for all sufficiently large $T$ we have $O\left(T^{-\frac{1}{2(1+\epsilon)}}\right)\leq C_1T^{-\frac{1}{2(1+\epsilon)}},$
and hence it suffices to require $-z\frac{\epsilon}{1+\epsilon}\ln T\gtrsim C_1T^{-\frac{1}{2(1+\epsilon)}}$. This holds for $T>T_0=\Theta\left(\frac{1}{\epsilon^2\ln^2(1/\epsilon)}\right)$, with a sufficiently large constant, since at this scale $\ln T_0\asymp\ln(1/\epsilon)$ and $T_0^{-\frac{1}{2(1+\epsilon)}}\asymp\epsilon\ln(1/\epsilon)$.
Note that $\frac{4 T\left\Vert \mathbf{w}^{\star}-\mathbf{w}_{0}\right\Vert ^{2}}{-v_{z}z}>0$.
We proceed with \cref{eqq:2}, in this case, the non-leading terms contribute $O\left(T^{-\frac{1}{2(1-\epsilon)}}\right)$ --- as $\exp(-q/2)$, the dominating term becomes $O\left(T^{-\frac{1}{2(1-\epsilon)}}\right)$).
Splitting it to cases.

\paragraph{Case (1)} When $\frac{\epsilon}{1-\epsilon}\ln\left(\frac{4 T\left\Vert \mathbf{w}^{\star}-\mathbf{w}_{0}\right\Vert ^{2}}{-v_{z}z}\right)\in (0,1)$,
\begin{align*}
&-z\cdot\left(\exp\left(-\frac{\epsilon}{1-\epsilon}\ln\left(\frac{4 T\left\Vert \mathbf{w}^{\star}-\mathbf{w}_{0}\right\Vert ^{2}}{-v_{z}z}\right)\right)-1\right)+O\left(T^{-\frac{1}{2(1-\epsilon)}}\right)\\
\explain{e^{-x}-1\leq-\frac{1}{2}x,~\forall x\in\left(0,1\right)}
&
\leq z\cdot\frac{\epsilon}{2\left(1-\epsilon\right)}\ln\left(\frac{4 T\left\Vert \mathbf{w}^{\star}-\mathbf{w}_{0}\right\Vert ^{2}}{-v_{z}z}\right)+O\left(T^{-\frac{1}{2(1-\epsilon)}}\right)\\&\lesssim z\cdot\frac{\epsilon}{1-\epsilon}\ln\left(T\right)+O\left(T^{-\frac{1}{2(1-\epsilon)}}\right).
\end{align*}

As before, there exists a constant $C_2$ such that for all sufficiently large $T$ we have $O\left(T^{-\frac{1}{2(1-\epsilon)}}\right)\leq C_2T^{-\frac{1}{2(1-\epsilon)}},$ and hence we require $-z\frac{\epsilon}{1-\epsilon}\ln T\gtrsim C_2T^{-\frac{1}{2(1-\epsilon)}}.$
Although this equation does not admit a closed-form solution for $T$, choosing $T_0$ so that $T>T_0=\Theta\left(\frac{1}{\epsilon^2\ln^2(1/\epsilon)}\right)$, with a sufficiently large implicit constant, ensures that $\ln T_0\asymp\ln(1/\epsilon)$ and $T_0^{-\frac{1}{2(1+\epsilon)}}\asymp\epsilon\ln(1/\epsilon)$.

\paragraph{Case (2)} 
When
$\frac{\epsilon}{1-\epsilon}\ln\left(\frac{4 T\left\Vert \mathbf{w}^{\star}-\mathbf{w}_{0}\right\Vert ^{2}}{-v_{z}z}\right)\in [1,\infty)$, 
\begin{align*}
&-z\left(\exp\left(-\frac{\epsilon}{1-\epsilon}\ln\left(\frac{4 T\left\Vert \mathbf{w}^{\star}-\mathbf{w}_{0}\right\Vert ^{2}}{-v_{z}z}\right)\right)-1\right)+O\left(T^{-\frac{1}{2(1-\epsilon)}}\right)\\&\leq -z\left(e^{-1}-1\right)+O\left(T^{-\frac{1}{2(1-\epsilon)}}\right).  
\end{align*}
Therefore, we require, 
\begin{align*}
-z\left(e^{-1}-1\right)\gtrsim K_{2}T^{-\frac{1}{2(1-\epsilon)}}
\end{align*}
which holds for all sufficiently large $T$ independent of $\epsilon$.

In this regime ($\gamma\asymp \frac{\ln T}{NT}$), the maximal gap between $G^{\mathrm{esgd}}_T$ and $G^{\mathrm{ridge}}_T$ is bounded by $O(\frac{\ln T}{T^{1+\beta}})$ (for some $\beta>0$, since they share the same expansion in \cref{eq:42}). Consequently, we have $G^{\mathrm{esgd}}_T-G^{\mathrm{ridge}}_T=O\left(\frac{1}{T}\right)$.

% By~\citet[Theorem 6]{karpel2026optimal} the optimal solution in the noiseless case of $\mathrm{Tr}(\boldsymbol\Sigma)=\labelnoise=0$, is $\lambda^\star\to 0$. 

% For $\X\in \mathbb R^{n\times d}$ define $B(\lambda)=(\mathbf X^\top \mathbf X+\lambda \mathbf I)^{-1}\mathbf X^\top$. By~\citet[][Theorem 4.3]{barata2012moore} we know that $B(\lambda)\to \X^{+}$ as $\lambda \to 0$ where $\X^{+}$ denotes the Moore–Penrose pseudoinverse. Furthermore, underdetermined least squares gradient descent can be shown to converge to the minimum Euclidean norm \cite{soudry2017implicit}, that is, the solution of $\X^{+}\y$  (otherwise, in the overdetermined case then $\X^{+}=(\X^\top \X)^{-1}\X$ is well defined, almost surely). Therefore, for any task $t\in [1, T]$ the optimal choice would be any $\gamma^\star\in \left(0,\frac{2}{\featnoise\left(1+\sqrt\frac{1}{\ratio}\right)^2}\right)$ and $N\to \infty$.

\end{proof}

\newpage

\section{Proofs for \cref{sec:2}: Generalized Ridge Regression}
\label{pr:sec5}
\subsection{Auxiliary Lemmas and definitions}
\begin{definition}[\citet{zhao2025highdimensional}, Definition 7]
\label{def:1}
Let $\mathbf A_n, \mathbf B_n \in \mathbb{R}^{n \times n}$ be sequences of random or deterministic symmetric matrices.  
We say that $\mathbf A_n$ and $\mathbf B_n$ are equivalent if, for any sequence of deterministic matrices
$\mathbf C_n \in \mathbb{R}^{n \times n}$ such that
$
\limsup_{n \to \infty} \Vert\mathbf C_n\Vert_{\mathrm{op}} < \infty,
$
and for any sequence of deterministic vectors $\mathbf v_n \in \mathbb{R}^n$ such that
$
\limsup_{n \to \infty} \Vert\mathbf v_n\Vert < \infty,
$
we have, almost surely,
$$
\frac{1}{n}\mathrm{Tr}\left[\mathbf C_n (\mathbf A_n - \mathbf B_n)\right] \to0,
\qquad
\mathbf v_n^\top (\mathbf A_n - \mathbf B_n) \mathbf v_n \to 0.
$$

Here, the operator norm is defined by
$
\Vert\mathbf A\Vert_{\mathrm{op}} := \sqrt{\lambda_{\max}(\mathbf A^\top \mathbf A)}
$
for any real matrix $\mathbf A \in \mathbb{R}^{p \times q}$.

\end{definition}
We write $\mathbf A_n \asymp \mathbf B_n$ if the sequences of matrices
$\{\mathbf A_n\}$ and $\{\mathbf B_n\}$ are equivalent.
For a random matrix sequence $\{\mathbf A_n\}$, if there exists a sequence of
deterministic matrices $\{\bar{\mathbf A}_n\}$ such that $\mathbf A_n \asymp \bar{\mathbf A}_n ,$
then $\bar{\mathbf A}_n$ is referred to as a \emph{deterministic equivalent} of
$\mathbf A_n$.
\begin{definition}[\citet{liao2023rmtml}, Notation 2]
\label{def:Z}
For $\mathcal A\subset \mathbb{C}, z\in \mathcal A$ and $m\in \mathbb{C}$, we denote $\mathcal{Z}(\mathcal A)$ the set of scalar pairs
\begin{align*}
\mathcal{Z}(A)
&=
\Bigl\{
(z,m)\in\mathcal A\times\mathbb{C},
\ \text{such that }
\Im[z]\cdot\Im[m] > 0 \ \text{ if } \Im[z]\neq 0
\\&
\text{or } (m>0 \text{ if } z<\inf\mathcal A \cap \mathbb{R})
\ \text{or } (m<0 \text{ if } z>\sup\mathcal A \cap \mathbb{R})
\Bigr\}.
\end{align*}

\end{definition}
\begin{proposition}[\citet{zhao2025highdimensional}, Proposition~8]
\label{prop:1}
\begin{enumerate}[label=\arabic*., itemsep=1ex]
    \item If $\mathbf A_n \asymp \mathbf B_n$ and $\mathbf C_n \asymp \mathbf D_n$, then
    $$
    \mathbf A_n + \mathbf C_n \asymp \mathbf B_n +\mathbf D_n.
    $$

    \item If $\mathbf A_n \asymp \mathbf B_n$ and $\{\mathbf E_n\}$ is a sequence of deterministic matrices such that
    $
    \Vert\mathbf E_n\Vert_{\mathrm{op}} < \infty,
    $
    then
    $$
    \mathbf E_n \mathbf A_n \mathbf E_n^\top \asymp \mathbf E_n \mathbf B_n \mathbf E_n^\top .
    $$

    \item If $\mathbf A_n \asymp \mathbf B_n$, then
    $$
    \frac{1}{n}\mathrm{Tr}[\mathbf A_n - \mathbf B_n] \xrightarrow{\text{a.s.}} 0.
    $$
\end{enumerate}

\end{proposition}
\begin{theorem}[Sample covariance matrix, 
\citet{silverstein1995empirical,liao2023rmtml}, Theorem 2.6]
\label{thm:romain}
Let $\mathbf X =\mathbf C^{1/2}\mathbf Z \in \mathbb{R}^{d \times n}$ with symmetric nonnegative definite
$\mathbf C \in \mathbb{R}^{d \times d}$ of bounded operator norm
(i.e., $\limsup_{d\to\infty}\Vert\mathbf C\Vert_{\mathrm{op}}<\infty$),
and $\mathbf Z \in \mathbb{R}^{d\times n}$ having independent zero-mean and unit-variance
entries with $4+\epsilon$ bounded moment for some $\epsilon>0$.
Then, as $n,d\to\infty$ with $d/n\to c\in(0,\infty)$, letting
\[
\mathbf{Q}(z)=\left(\frac{1}{n}\mathbf{X}\mathbf{X}^{\top}-z\mathbf{I}_{d}\right)^{-1},\qquad\widetilde{\mathbf{Q}}(z)=\left(\frac{1}{n}\mathbf{X}^{\top}\mathbf{X}-z\mathbf{I}_{n}\right)^{-1},
\]
we have
\[
\mathbf{Q}(z)\asymp\bar{\mathbf{Q}}(z)=-\frac{1}{z}\bigl(\mathbf{I}_{d}+\widetilde{m}_{d}(z)\mathbf{C}\bigr)^{-1},
\]
\[
\widetilde{\mathbf{Q}}(z)\asymp\overline{\widetilde{\mathbf{Q}}}(z)=\widetilde{m}_{d}(z)\mathbf{I}_{n},
\]
where $(z,\widetilde m_p(z))$ is the unique solution in
$\mathcal{Z}\left(\mathbb{C} \setminus \mathbb{R}_+\right)$ (see \cref{def:Z}) of
\begin{equation}
\widetilde{m}_{d}(z)=\left(-z+\frac{1}{n}\mathrm{Tr}\mathbf{C}\bigl(\mathbf{I}_{d}+\widetilde{m}_{d}(z)\mathbf{C}\bigr)^{-1}\right)^{-1}.
\end{equation}

In particular, if the empirical spectral measure of $\mathbf C$ converges,
i.e., $\mu_{\mathbf C} \Rightarrow \nu$ as $d\to \infty$, then
\[
\mu_{\frac{1}{n}\mathbf X\mathbf X^\top} \xrightarrow{\text{a.s.}} \mu,
\qquad
\mu_{\frac{1}{n}\mathbf X^\top \mathbf X} \xrightarrow{\text{a.s.}} \widetilde\mu
\quad \text{as } d,n \to \infty,
\]
where $\mu$ and $\widetilde\mu$ are the unique measures having Stieltjes transforms
$m(z)$ and $\widetilde m(z)$, respectively, with
\[
m(z)=\frac{1}{c}\widetilde m(z) + \frac{1-c}{c z},\qquad\widetilde m(z)=\left(-z + c \int \frac{t\nu(\mathrm{d}t)}{1+\widetilde m(z) t}\right)^{-1}.
\]
\end{theorem}
\begin{theorem}
\label{thm:zhao_resolvent}
Assume $\mathbf{X}=\mathbf{Z}\mathbf{\Sigma}^{\frac{1}{2}}$, where $\mathbf{Z}\in\mathbb{R}^{n\times d}$ has i.i.d. entries with zero mean, unit variance and finite $4+\epsilon$ moment. Let $\lambda>0$ be a fixed real number, $\mathbf Q\left(-\lambda\right)=\left(\frac{1}{n}\mathbf{X}^{\top}\mathbf{X}+\lambda\mathbf{I}_{d}\right)^{-1}$ be the resolvent matrix, $\mathbf{A}\in\mathbb{R}^{d\times d}$ is a symmetric deterministic matrix or symmetric random matrix independent of $\mathbf Q\left(-\lambda\right)$, and having bounded operator norm. Suppose $\mathbf{\Sigma}\in\mathbb{R}^{d\times d}$ is a deterministic positive semi–definite matrix with bounded operator norm, then as $n,d\to\infty$  with $n/d\to\alpha\in\left(0,\infty\right)$ we have
$$
\mathbf{Q}\left(-\lambda\right)\mathbf{A}\mathbf{Q}\left(-\lambda\right)\asymp\bar{\mathbf{Q}}\mathbf{A}\bar{\mathbf{Q}}+\frac{1}{n}\mathrm{Tr}\left[\mathbf{\Sigma}\bar{\mathbf{Q}}\mathbf{A}\bar{\mathbf{Q}}\right]\left[\left[1+\frac{1}{n}\mathrm{Tr}\left(\mathbf{\Sigma}\bar{\mathbf{Q}}\right)\right]^{2}-\frac{1}{n}\mathrm{Tr}\left(\mathbf{\Sigma}\bar{\mathbf{Q}}\right)^{2}\right]^{-1}\bar{\mathbf{Q}}\mathbf{\Sigma}\bar{\mathbf{Q}}.$$
\end{theorem}

\begin{proof}
If $\mathbf A$ is random, we condition throughout on $\mathbf A$. Since
$\mathbf A$ is independent of $\mathbf Q(-\lambda)$ and has bounded operator
norm, it may then be treated as deterministic. For $|\rho|$ sufficiently small,
$\mathbf I_p+\rho\mathbf A$ is positive definite.

The key observation (as also identified in \citet{patil2022mitigating}, Lemma S.6.10) is that
\begin{align}
\lambda^{2}\mathbf{Q}(-\lambda)\mathbf{A}\mathbf{Q}(-\lambda)&=-\left.\frac{\partial}{\partial\rho}\left[\lambda\left(\frac{1}{n}\mathbf{X}^{\top}\mathbf{X}+\lambda\mathbf{I}_{d}+\lambda\rho\mathbf{A}\right)^{-1}\right]\right|_{\rho=0}.
\label{eq:47}
\end{align}
Thus, it suffices to obtain a deterministic equivalent for the
parameterized resolvent in \cref{eq:47}, differentiate it
with respect to $\rho$, and set $\rho=0$.

We would like to make use of \cref{thm:romain}, note that,
\begin{equation}
\begin{aligned}
&\lambda\left(\frac{1}{n}\mathbf{X}^{\top}\mathbf{X}+\lambda\mathbf{I}_{d}+\lambda\rho\mathbf{A}\right)^{-1}=\lambda\left(\frac{1}{n}\mathbf{X}^{\top}\mathbf{X}+\lambda\left(\mathbf{I}_{d}+\rho\mathbf{A}\right)\right)^{-1}\\&=(\mathbf{I}_{d}+\rho\mathbf{A})^{-1/2}\lambda\left[\left(\mathbf{I}_{d}+\rho\mathbf{A}\right)^{-1/2}\frac{1}{n}\mathbf{X}^{\top}\mathbf{X}\left(\mathbf{I}_{d}+\rho\mathbf{A}\right)^{-1/2}+\lambda\mathbf{I}_{d}\right]^{-1}\left(\mathbf{I}_{d}+\rho\mathbf{A}\right)^{-1/2}.
\end{aligned}
\label{eq:48}
\end{equation}

The population covariance corresponding to the resolvent in the middle of
\cref{eq:48} is
\[
(\mathbf I_p+\rho\mathbf A)^{-1/2}\mathbf\Sigma(\mathbf I_p+\rho\mathbf A)^{-1/2}.
\]
Since its operator norm remains uniformly bounded for $\rho$ in a
sufficiently small neighborhood of zero, and since
$\mathbb E|Z_{ij}|^{4+\epsilon}<\infty$, we may apply
\cref{thm:romain}. It follows that
\begin{equation}
\begin{aligned}
&\lambda\left[(\mathbf{I}_{d}+\rho\mathbf{A})^{-1/2}\frac{1}{n}\mathbf{X}^{\top}\mathbf{X}(\mathbf{I}_{d}+\rho\mathbf{A})^{-1/2}+\lambda\mathbf{I}_{d}\right]^{-1}\\&\hspace{12em}\asymp\left[\widetilde{m}_{p}(-\lambda,\rho)(\mathbf{I}_{d}+\rho\mathbf{A})^{-1/2}\mathbf{\Sigma}(\mathbf{I}_{d}+\rho\mathbf{A})^{-1/2}+\mathbf{I}_{d}\right]^{-1},
\end{aligned}
\label{eq:49}
\end{equation}

where $\widetilde m_p(-\lambda,\rho)$ is the unique positive solution of the fixed-point equation from \cref{thm:romain},
\begin{align}
\frac{1}{\widetilde{m}_{p}(-\lambda,\rho)}=\lambda+\frac{1}{n}\mathrm{Tr}\left[\mathbf{\Sigma}\left(\widetilde{m}_{p}(-\lambda,\rho)\mathbf{\Sigma}+\mathbf{I}_{d}+\rho\mathbf{A}\right)^{-1}\right].
\label{eq:52}
\end{align}

Combining \cref{eq:48,eq:49} and using the
product rule for deterministic equivalents (\cref{prop:1}) gives
\begin{equation}
\begin{aligned}
&\lambda\left(\frac{1}{n}\mathbf{X}^{\top}\mathbf{X}+\lambda\mathbf{I}_{d}+\lambda\rho\mathbf{A}\right)^{-1}\\&=(\mathbf{I}_{d}+\rho\mathbf{A})^{-1/2}\lambda\left[\left(\mathbf{I}_{d}+\rho\mathbf{A}\right)^{-1/2}\frac{1}{n}\mathbf{X}^{\top}\mathbf{X}\left(\mathbf{I}_{d}+\rho\mathbf{A}\right)^{-1/2}+\lambda\mathbf{I}_{d}\right]^{-1}\left(\mathbf{I}_{d}+\rho\mathbf{A}\right)^{-1/2}.\\&\asymp(\mathbf{I}_{d}+\rho\mathbf{A})^{-1/2}\left[\widetilde{m}_{p}(-\lambda,\rho)(\mathbf{I}_{d}+\rho\mathbf{A})^{-1/2}\mathbf{\Sigma}(\mathbf{I}_{d}+\rho\mathbf{A})^{-1/2}+\mathbf{I}_{d}\right]^{-1}(\mathbf{I}_{d}+\rho\mathbf{A})^{-1/2}\\&=\left(\widetilde{m}_{p}(-\lambda,\rho)\mathbf{\Sigma}+\mathbf{I}_{d}+\rho\mathbf{A}\right)^{-1}
\end{aligned}
\label{eq:50}
\end{equation}
At $\rho=0$, 
\begin{align}
\left(\widetilde{m}_{p}(-\lambda,0)\mathbf{\Sigma}+\mathbf{I}_{d}\right)^{-1}=\lambda\bar{\mathbf{Q}}.
\label{eq:51}
\end{align}
Setting $\rho=0$ in \cref{eq:52} and using
\cref{eq:51} gives
\begin{align}
\frac{1}{\widetilde{m}_{p}(-\lambda,0)}=\lambda+\frac{\lambda}{n}\mathrm{Tr}\left(\mathbf{\Sigma}\bar{\mathbf{Q}}\right)=\lambda\left[1+\frac{1}{n}\mathrm{Tr}\left(\mathbf{\Sigma}\bar{\mathbf{Q}}\right)\right].
\label{eq:53}
\end{align}

Both sides of \cref{eq:50} are analytic in $\rho$ in a
neighborhood of zero and are uniformly bounded there, since $\lambda>0$
and $\Vert\mathbf A\Vert_{\mathrm{op}}=O(1)$. We use the differentiation technique for deterministic equivalents (see \citet{patil2022mitigating}, Lemma S.7.2 (5)). 
Let $\mathbf T\in\mathbb R^{d\times d}$ have trace norm uniformly bounded in $d$.
For $|\rho|$ sufficiently small such that
$\mathbf I_d+\rho\mathbf A\succ0$, we have
\begin{align*}
&\left|\trace\left[\mathbf{T}\lambda\left(\frac{1}{n}\mathbf{X}^{\top}\mathbf{X}+\lambda\mathbf{I}_{d}+\lambda\rho\mathbf{A}\right)^{-1}\mathbf{A}\right]\right|\le\Vert\mathbf{T}\Vert_{*}\left\Vert \lambda\left(\frac{1}{n}\mathbf{X}^{\top}\mathbf{X}+\lambda(\mathbf{I}_{d}+\rho\mathbf{A})\right)^{-1}\right\Vert _{\mathrm{op}}\Vert\mathbf{A}\Vert_{\mathrm{op}}.
\end{align*}
Since
\[
\frac{1}{n}\mathbf{X}^{\top}\mathbf{X}+\lambda(\mathbf{I}_{d}+\rho\mathbf{A})\succeq\lambda(\mathbf{I}_{d}+\rho\mathbf{A}),
\]
we obtain
\[
\left\Vert \lambda\left(\frac{1}{n}\mathbf{X}^{\top}\mathbf{X}+\lambda(\mathbf{I}_{d}+\rho\mathbf{A})\right)^{-1}\right\Vert _{\mathrm{op}}\le\Vert(\mathbf{I}_{d}+\rho\mathbf{A})^{-1}\Vert_{\mathrm{op}}.
\]
Therefore, for $\rho$ in a sufficiently small fixed neighborhood of zero,
\[
\left|\trace\left[\mathbf{T}\lambda\left(\frac{1}{n}\mathbf{X}^{\top}\mathbf{X}+\lambda\mathbf{I}_{d}+\lambda\rho\mathbf{A}\right)^{-1}\mathbf{A}\right]\right|\le M,
\]
for some constant $M<\infty$ independent of $n,d$. Thus, we can apply the differentiation rule from \citet[Lemma S.7.2 (5)]{patil2022mitigating}.
Differentiating
\cref{eq:50} at $\rho=0$ and using
\cref{eq:47} yields
\begin{equation}
\begin{aligned}
\lambda^{2}\mathbf{Q}(-\lambda)\mathbf{A}\mathbf{Q}(-\lambda)\asymp&\left(\widetilde{m}_{p}(-\lambda,0)\mathbf{\Sigma}+\mathbf{I}_{d}\right)^{-1}\left[\left.\frac{\partial}{\partial\rho}\widetilde{m}_{p}(-\lambda,\rho)\right|_{\rho=0}\mathbf{\Sigma}+\mathbf{A}\right]\left(\widetilde{m}_{p}(-\lambda,0)\mathbf{\Sigma}+\mathbf{I}_{d}\right)^{-1}.\\&
\end{aligned}
\label{eq:39}
\end{equation}

It remains to evaluate
$\left.\frac{\partial}{\partial\rho}
\widetilde m_p(-\lambda,\rho)\right|_{\rho=0}$.
Differentiating \cref{eq:52} with respect to $\rho$
and then setting $\rho=0$ gives
\begin{equation}
\begin{aligned}
&-\left.\frac{\partial}{\partial\rho}\widetilde{m}_{p}(-\lambda,\rho)\right|_{\rho=0}\frac{1}{\widetilde{m}_{p}(-\lambda,0)^{2}}\\&\quad=-\frac{1}{n}\mathrm{Tr}\Bigg[\mathbf{\Sigma}\left(\widetilde{m}_{p}(-\lambda,0)\mathbf{\Sigma}+\mathbf{I}_{d}\right)^{-1}\left(\left.\frac{\partial}{\partial\rho}\widetilde{m}_{p}(-\lambda,\rho)\right|_{\rho=0}\mathbf{\Sigma}+\mathbf{A}\right)\left(\widetilde{m}_{p}(-\lambda,0)\mathbf{\Sigma}+\mathbf{I}_{d}\right)^{-1}\Bigg].
\end{aligned}
\label{eq:54}
\end{equation}

Since $\mathbf\Sigma$ commutes with
$\left(\widetilde m_p(-\lambda,0)\mathbf\Sigma+\mathbf I_p\right)^{-1}$,
solving \cref{eq:54} gives
\begin{align}
\left.\frac{\partial}{\partial\rho}\widetilde{m}_{p}(-\lambda,\rho)\right|_{\rho=0}=\frac{\frac{1}{n}\mathrm{Tr}\left[\mathbf{\Sigma}\left(\widetilde{m}_{p}(-\lambda,0)\mathbf{\Sigma}+\mathbf{I}_{d}\right)^{-1}\mathbf{A}\left(\widetilde{m}_{p}(-\lambda,0)\mathbf{\Sigma}+\mathbf{I}_{d}\right)^{-1}\right]}{\frac{1}{\widetilde{m}_{p}(-\lambda,0)^{2}}-\frac{1}{n}\mathrm{Tr}\left[\mathbf{\Sigma}^{2}\left(\widetilde{m}_{p}(-\lambda,0)\mathbf{\Sigma}+\mathbf{I}_{d}\right)^{-2}\right]}.
\label{eq:55}
\end{align}

Using \cref{eq:51,eq:53}, the numerator of
\cref{eq:55} becomes
$\frac{\lambda^{2}}{n}\mathrm{Tr}\left[\mathbf{\Sigma}\bar{\mathbf{Q}}\mathbf{A}\bar{\mathbf{Q}}\right],$ while its denominator becomes $\lambda^{2}\left[\left(1+\frac{1}{n}\mathrm{Tr}\left(\mathbf{\Sigma}\bar{\mathbf{Q}}\right)\right)^{2}-\frac{1}{n}\mathrm{Tr}\left[\left(\mathbf{\Sigma}\bar{\mathbf{Q}}\right)^{2}\right]\right].$
Therefore,
\begin{align}
\left.\frac{\partial}{\partial\rho}\widetilde{m}_{p}(-\lambda,\rho)\right|_{\rho=0}=\frac{\frac{1}{n}\mathrm{Tr}\left[\mathbf{\Sigma}\bar{\mathbf{Q}}\mathbf{A}\bar{\mathbf{Q}}\right]}{\left(1+\frac{1}{n}\mathrm{Tr}\left(\mathbf{\Sigma}\bar{\mathbf{Q}}\right)\right)^{2}-\frac{1}{n}\mathrm{Tr}\left[\left(\mathbf{\Sigma}\bar{\mathbf{Q}}\right)^{2}\right]}.
\label{eq:55-final}
\end{align}

Finally, substituting
\cref{eq:51,eq:55-final} into
\cref{eq:39} and dividing by $\lambda^2$ yields
\begin{align*}
\mathbf{Q}(-\lambda)\mathbf{A}\mathbf{Q}(-\lambda)\asymp\bar{\mathbf{Q}}\mathbf{A}\bar{\mathbf{Q}}+\frac{{ \frac{1}{n}\mathrm{Tr}\left[\mathbf{\Sigma}\bar{\mathbf{Q}}\mathbf{A}\bar{\mathbf{Q}}\right]}}{{ \left(1+\frac{1}{n}\mathrm{Tr}\left(\mathbf{\Sigma}\bar{\mathbf{Q}}\right)\right)^{2}-\frac{1}{n}\mathrm{Tr}\left[\left(\mathbf{\Sigma}\bar{\mathbf{Q}}\right)^{2}\right]}}\bar{\mathbf{Q}}\mathbf{\Sigma}\bar{\mathbf{Q}},
\end{align*}
as claimed.
\end{proof}
The following Lemma provides the closed-from solution to the iterates produced by the generalized ridge regression algorithm (\cref{schm:generalized}).
\begin{lemma}
\label{lem:4}
At each iteration $t$ the predictor update under \cref{schm:generalized} is $$\mathbf{w}_{t}=\left(\frac{1}{n}\mathbf{X}_{t}^{\top}\mathbf{X}_{t}+\mathbf{M}\right)^{-1}\left(\frac{1}{n}\mathbf{X}_{t}^{\top}\mathbf{y}_{t}+\mathbf{M}\mathbf{w}_{t-1}\right).$$
\end{lemma}
\begin{proof}
Let $t\in [1, T]$ and consider the optimization problem
$$
\mathbf{w}_{t}\in\argmin_{\mathbf{w}\in\mathbb{R}^{d}}
\frac{1}{2n}\left\Vert \mathbf{X}_{t}\mathbf{w}-\mathbf{y}_{t}\right\Vert _{2}^{2}
+\frac{1}{2}\left\Vert \mathbf{w}-\mathbf{w}_{t-1}\right\Vert _{\mathbf{M}}^{2},
$$
where $\mathbf{M}\succ 0$. The objective is a convex quadratic function of
$\mathbf{w}$, hence the minimizer is unique in that case.

Expand the objective and differentiate with respect to $\mathbf{w}$:
\begin{align*}
f_t(\mathbf{w})
&:=\frac{1}{2n}\Vert\mathbf{X}_t\mathbf{w}-\mathbf{y}_t\Vert_2^2
+\frac{1}{2}(\mathbf{w}-\mathbf{w}_{t-1})^\top \mathbf{M}(\mathbf{w}-\mathbf{w}_{t-1}),\\
\nabla f_t(\mathbf{w})
&=\frac{1}{n}\mathbf{X}_t^\top(\mathbf{X}_t\mathbf{w}-\mathbf{y}_t)
+\mathbf{M}(\mathbf{w}-\mathbf{w}_{t-1}).
\end{align*}
Setting $\nabla f_t(\mathbf{w})=0$ gives the normal equations
\[
\Big(\tfrac{1}{n}\mathbf{X}_t^\top\mathbf{X}_t+\mathbf{M}\Big)\mathbf{w}
=
\tfrac{1}{n}\mathbf{X}_t^\top\mathbf{y}_t+\mathbf{M}\mathbf{w}_{t-1}.
\]
Since $\tfrac{1}{n}\mathbf{X}_t^\top\mathbf{X}_t+\mathbf{M}$ is invertible,
we obtain the closed-form solution
\[
\mathbf{w}_t
=
\Big(\tfrac{1}{n}\mathbf{X}_{t}^{\top}\mathbf{X}_{t}+\mathbf{M}\Big)^{-1}
\Big(\tfrac{1}{n}\mathbf{X}_{t}^{\top}\mathbf{y}_{t}+\mathbf{M}\mathbf{w}_{t-1}\Big).
\]
\end{proof}

\newpage

\subsection{Proof of \cref{thm:4}: Closed-Form Generalization Loss with Diagonal Regularization}
\begin{recall}[\cref{thm:4}]
Under \cref{assum:1}, the generalization loss after T tasks of the generalized ridge algorithm (\cref{schm:generalized}), specialized to a fixed regularization matrix $\mathbf M$, as $n,d\to\infty$, satisfies 
\begin{align*}
&G_{T}^{\mathrm{gen}}\xrightarrow{P}\left(\mathbf{w}_{0}-\mathbf{w}^{\star}\right)^{\top}\left(\bar{\mathbf{Q}}^{2T}+\featnoise\sum_{k=1}^{T}\bar{\mathbf{Q}}^{2k}\tilde{\mu}\tilde{\rho}_{k}\right)\left(\mathbf{w}_{0}-\mathbf{w}^{\star}\right)\\&+\frac{\labelnoise}{n}\mathrm{Tr}\Bigg[\sum_{k=1}^{T}\left[\mathbf{I}-\bar{\mathbf{Q}}\left[\mathbf{I}+\frac{\featnoise^{2}}{n}\mathrm{Tr}\left[\mathbf{M}^{-1}\bar{\mathbf{Q}}^{2}\right]\tilde{\mu}\mathbf{M}^{-1}\right]\right]\left[\bar{\mathbf{Q}}^{2\left(T-k\right)}+\featnoise\sum_{i=k+1}^{T}\bar{\mathbf{Q}}^{2\left(i-k\right)}\tilde{\mu}\tilde{\rho}_{i}\right]\mathbf{M}^{-1}\bar{\mathbf{Q}}\Bigg]
\end{align*}
where $\bar{\mathbf{Q}}=\big(\mathbf{I}+\featnoise\tilde{m}_{p}(-1)\mathbf{M}^{-1}\big)^{-1},$
with $\tilde m_p(-1)$ being the unique positive solution of
$$
\frac{1}{\tilde m_p(-1)}=1+\frac{1}{n}
\sum_{i=1}^d
\frac{\featnoise  M_{i,i}^{-1}}
{1+\tilde m_p(-1)\featnoise M_{i,i}^{-1}}.
$$
And,
\begin{align*}   
\tilde{\rho}_{k}=\frac{\featnoise}{n}\mathrm{Tr}\Big[\mathbf{M}^{-2}\Big(\bar{\mathbf{Q}}^{2(T-k+1)}+\featnoise\tilde{\mu}\sum_{i=k+1}^{T}\tilde{\rho}_{i}\bar{\mathbf{Q}}^{2(i-k+1)}\Big)\Big],\\\tilde{\mu}=\left[\left[1+\frac{\featnoise}{n}\mathrm{Tr}\left(\mathbf{M}^{-1}\bar{\mathbf{Q}}\right)\right]^{2}-\frac{\featnoise^2}{n}\mathrm{Tr}\left(\mathbf{M}^{-1}\bar{\mathbf{Q}}\right)^{2}\right]^{-1}.
\end{align*}
\end{recall}
\begin{proof}
We will develop a slightly more general result, the outcome would be able to accommodate general diagonal covariance matrix $\mathbf \Sigma$ (spectrally bounded), in contrast to the setting considered in \cref{sec:statistical_setting}. 
Specifically, here we may assume that the data for task $t$ is generated as $\mathbf X_t=\mathbf Z_t \mathbf \Sigma^{1/2}$, where $\mathbf Z_t \in \mathbb{R}^{n \times d}$ has i.i.d.\ entries with zero mean and unit variance.

For each $t\in [1, T]$, we are interested in solving 
$$\mathbf{w}^{\mathrm{gen}}_{t}=\min_{\mathbf{w}\in\mathbb{R}^{d}}\frac{1}{n}\left\Vert \mathbf{X}_{t}\mathbf{w}-\mathbf{y}\right\Vert _{2}^{2}+\left\Vert \mathbf{w}-\mathbf{w}_{t-1}\right\Vert^2 _{\mathbf{M}}.$$
By \cref{lem:4},
$$
\mathbf{w}^{\mathrm{gen}}_{t}=\left(\frac{1}{n}\mathbf{X}_{t}^{\top}\mathbf{X}_{t}+\mathbf{M}\right)^{-1}\left(\frac{1}{n}\mathbf{X}_{t}^{\top}\mathbf{y}_{t}+\mathbf{M}\mathbf{w}^{\mathrm{gen}}_{t-1}\right).
$$
Therefore,
\begin{align*}
    \mathbf{w}^{\mathrm{gen}}_{T}-\mathbf{w}^{\star}&=\left(\frac{1}{n}\mathbf{X}_{T}^{\top}\mathbf{X}_{T}+\mathbf{M}\right)^{-1}\left(\frac{1}{n}\mathbf{X}_{T}^{\top}\left(\mathbf{X}_{T}\mathbf{w}^{\star}+\mathbf{z}_{T}\right)+\mathbf{M}\mathbf{w}^{\mathrm{gen}}_{T-1}\right)-\mathbf{w}^{\star}\\&=\left(\frac{1}{n}\mathbf{X}_{T}^{\top}\mathbf{X}_{T}+\mathbf{M}\right)^{-1}\left(\frac{1}{n}\mathbf{X}_{T}^{\top}\left(\mathbf{X}_{T}\mathbf{w}^{\star}+\mathbf{z}_{T}\right)+\mathbf{M}\mathbf{w}^{\mathrm{gen}}_{T-1}-\left(\frac{1}{n}\mathbf{X}_{T}^{\top}\mathbf{X}_{T}+\mathbf{M}\right)\mathbf{w}^{\star}\right)\\&=\left(\frac{1}{n}\mathbf{X}_{T}^{\top}\mathbf{X}_{T}+\mathbf{M}\right)^{-1}\left(\frac{1}{n}\mathbf{X}_{T}^{\top}\mathbf{z}_{T}+\mathbf{M}\left(\mathbf{w}^{\mathrm{gen}}_{T-1}-\mathbf{w}^{\star}\right)\right)\\&=\mathbf{P}_{T}\mathbf{M}\left(\mathbf{w}^{\mathrm{gen}}_{T-1}-\mathbf{w}^{\star}\right)+\mathbf{P}_{T}\frac{1}{n}\mathbf{X}_{T}^{\top}\mathbf{z}_{T}\\&=\ldots\\&=\left[\prod_{m=T}^{1}\mathbf{P}_{m}\mathbf{M}\right]\left(\mathbf{w}_{0}-\mathbf{w}^{\star}\right)+\sum_{k=1}^{T}\left[\prod_{m=T}^{k+1}\mathbf{P}_{m}\mathbf{M}\right]\mathbf{P}_{k}\frac{1}{n}\mathbf{X}_{k}^{\top}\mathbf{z}_{k},
\end{align*}
where $\mathbf{P}_{t}=\left(\frac{1}{n}\mathbf{X}_{t}^{\top}\mathbf{X}_{t}+\mathbf{M}\right)^{-1}=\mathbf{M}^{-\frac{1}{2}}\left(\frac{1}{n}\mathbf{M}^{-\frac{1}{2}}\mathbf{X}_{t}^{\top}\mathbf{X}_{t}\mathbf{M}^{-\frac{1}{2}}+\mathbf{I}\right)^{-1}\mathbf{M}^{-\frac{1}{2}}$ for all $t \in [1,T]$.

Then, 
\begin{align}
&G^{\mathrm{gen}}_T=\left[\left[\prod_{m=T}^{1}\mathbf{P}_{m}\mathbf{M}\right]\left(\mathbf{w}_{0}-\mathbf{w}^{\star}\right)+\sum_{k=1}^{T}\left[\prod_{m=T}^{k+1}\mathbf{P}_{m}\mathbf{M}\right]\mathbf{P}_{k}\frac{1}{n}\mathbf{\Sigma}^{\frac{1}{2}}\mathbf{Z}_{k}^{\top}\mathbf{z}_{k}\right]^{\top}\notag\\&\hspace{12em}\left[\left[\prod_{m=T}^{1}\mathbf{P}_{m}\mathbf{M}\right]\left(\mathbf{w}_{0}-\mathbf{w}^{\star}\right)+\sum_{k=1}^{T}\left[\prod_{m=T}^{k+1}\mathbf{P}_{m}\mathbf{M}\right]\mathbf{P}_{k}\frac{1}{n}\mathbf{\Sigma}^{\frac{1}{2}}\mathbf{Z}_{k}^{\top}\mathbf{z}_{k}\right]\notag\\&\stackrel{\explain{1}}{=}\left(\mathbf{w}_{0}-\mathbf{w}^{\star}\right)^{\top}\underbrace{\left[\prod_{m=T}^{1}\mathbf{P}_{m}\mathbf{M}\right]^{\top}\left[\prod_{m=T}^{1}\mathbf{P}_{m}\mathbf{M}\right]}_{\text{term 1}}\left(\mathbf{w}_{0}-\mathbf{w}^{\star}\right)\notag\\&\hspace{8em}+\underbrace{\left[\frac{1}{n^{2}}\sum_{k=1}^{T}\mathbf{z}_{k}^{\top}\mathbf{Z}_{k}\mathbf{\Sigma}^{\frac{1}{2}}\mathbf{P}_{k}\left[\prod_{m=T}^{k+1}\mathbf{P}_{m}\mathbf{M}\right]^{\top}\sum_{k'=1}^{T}\left[\prod_{m=T}^{k'+1}\mathbf{P}_{m}\mathbf{M}\right]\mathbf{P}_{k'}\mathbf{\Sigma}^{\frac{1}{2}}\mathbf{Z}_{k'}^{\top}\mathbf{z}_{k'}\right]}_{\text{term 2}},\label{eq:22}
\end{align}
where [1] follows since under the statistical design (\cref{sec:statistical_setting}) in which 
$\mathbb E[\mathbf z_k]=\mathbf 0$ for all $k\in[1, T]$, and the noise variables are sampled independently of all tasks, so by the weak law of large numbers cross terms vanish as $n\to\infty.$

We begin by analyzing the first term.
Let $\mathbf{A}_{i}=\mathbf{M}^{\frac{1}{2}}\mathbf{P}_{i}\left[\prod_{m=T}^{i+1}\mathbf{P}_{m}\mathbf{M}\right]^{\top}\left[\prod_{m=T}^{i+1}\mathbf{P}_{m}\mathbf{M}\right]\mathbf{P}_{i}\mathbf{M}^{\frac{1}{2}}$. 
To build some intuition for the application of \cref{thm:zhao_resolvent}, note that in order to compute term 1 we would need to solve
\begin{align*}
&\left[\prod_{m=T}^{1}\mathbf{P}_{m}\mathbf{M}\right]^{\top}\left[\prod_{m=T}^{1}\mathbf{P}_{m}\mathbf{M}\right]=\mathbf{M}\mathbf{P}_{1}\ldots\mathbf{M}\mathbf{P}_{T}\mathbf{P}_{T}\mathbf M\ldots\mathbf{P}_{1}\mathbf{M}\\&=\mathbf{M}^{\frac{1}{2}}\underbrace{\underbrace{\left[\mathbf{M}^{\frac{1}{2}}\mathbf{P}_{1}\mathbf{M}^{\frac{1}{2}}\right]}_{\mathbf Q\left(-1\right)}\underbrace{\ldots\mathbf{M}\mathbf{P}_{T}\mathbf{M}^{\frac{1}{2}}\mathbf{M}^{-1}\mathbf{M}^{\frac{1}{2}}\mathbf{P}_{T}\mathbf{M}\ldots}_{\text{independent}}\underbrace{\left[\mathbf{M}^{\frac{1}{2}}\mathbf{P}_{1}\mathbf{M}^{\frac{1}{2}}\right]}_{\mathbf Q\left(-1\right)}}_{\mathbf Q\left(-1\right)\mathbf{A}_{2}\mathbf Q\left(-1\right)}\mathbf{M}^{\frac{1}{2}}.
\end{align*}
where $\mathbf Q(z) :=\left(\frac{1}{n}\mathbf{M}^{-\frac{1}{2}}\mathbf{\Sigma}^{\frac{1}{2}}\mathbf{Z}^{\top}\mathbf{Z}\mathbf{\Sigma}^{\frac{1}{2}}\mathbf{M}^{-\frac{1}{2}}-z\mathbf{I}\right)^{-1}$ is the resolvent matrix with $\mathbf Z$ having i.i.d. entries with zero mean and unit variance.
The middle block is a symmetric random matrix independent of $\mathbf Q(-1)$ with bounded operator norm, therefore we use \cref{thm:zhao_resolvent} to evaluate $\mathbf Q(-1)\mathbf A_2\mathbf Q(-1)$.

Using \cref{thm:romain}, $\mathbf Q(-1)$ admits a deterministic equivalent, where,
\begin{align*}
\bar{\mathbf{Q}}\left(-1\right)=\left(\mathbf{I}+\tilde{m}_{p}\left(-1\right)\mathbf{\Sigma}\mathbf{M}^{-1}\right)^{-1}\text{ and }\frac{1}{\tilde{m}_{p}\left(-1\right)}=1+\frac{1}{n}\sum_{i=1}^{d}\frac{\sigma_{i}m_i^{-1}}{1+\tilde{m}_{p}\left(-1\right)\sigma_{i}m_i^{-1}}.
\end{align*}
Here $\sigma_i,m_i$ denote respectively, the $i$'th eigenvalues of the matrices $\mSigma, \mathbf M$. For notational convenience, we set
$\mathbf Q:= \mathbf Q(-1)$, and $\bar{\mathbf Q}:=\bar {\mathbf Q}(-1)$. Both matrices are diagonal since both $\mathbf \Sigma$ and $\mathbf M$ are diagonal. Moreover, we note that under our standing assumption, we have $\sigma_i=\featnoise$ for all $i$. Nonetheless, as previously mentioned, our analysis also allows for non-isotropic covariance matrices.

For all $k\geq 1$, using \cref{thm:zhao_resolvent} we attain the following recurrence,
\begin{align}
\label{eq:21}
    \mathbf{A}_{k}\asymp\bar{\mathbf Q}\left(-1\right)\mathbf{A}_{k+1}\bar{\mathbf Q}\left(-1\right)+\tilde{\mu}\tilde{\rho}_{k}\mathbf{D}
\end{align}
where 
\begin{align*}
&\tilde{\rho}_{k}=\frac{1}{n}\mathrm{Tr}\left[\mathbf{\Sigma}\mathbf{M}^{-1}\bar{\mathbf Q}\mathbf{A}_{k+1}\bar{\mathbf Q}\right], \\&\tilde{\mu}=\bigg[\left[1+\frac{1}{n}\mathrm{Tr}\left(\mathbf{\Sigma}\mathbf{M}^{-1}\bar{\mathbf{Q}}\right)\right]^{2}-\frac{1}{n}\mathrm{Tr}\Big(\mathbf{\Sigma}\mathbf{M}^{-1}\bar{\mathbf{Q}}\Big)^{2}\Bigg]^{-1},\\&
\mathbf{D}=\bar{\mathbf Q}\mathbf{\Sigma}\mathbf{M}^{-1}\bar{\mathbf Q}.
\end{align*}

Using the recurrence from \cref{eq:21},

\begin{align*}
\mathbf{A}_{k}&\asymp\bar{\mathbf Q}\mathbf{A}_{k+1}\bar{\mathbf Q}+\tilde{\mu}\tilde{\rho}_{k}\mathbf{D}\\&\asymp\bar{\mathbf Q}\left[\bar{\mathbf Q}\mathbf{A}_{k+2}\bar{\mathbf Q}+\tilde{\mu}\tilde{\rho}_{k+1}\mathbf{D}\right]\bar{\mathbf Q}+\tilde{\mu}\tilde{\rho}_{k}\mathbf{D}\\&\asymp\ldots\\&\asymp\bar{\mathbf Q}^{T-k+1}\underbrace{\mathbf{A}_{T+1}}_{\mathbf{M}^{-1}}\bar{\mathbf Q}^{T-k+1}+\sum_{i=k}^{T}\bar{\mathbf Q}^{i-k}\tilde{\mu}\tilde{\rho_i}\bar{\mathbf Q}^{i-k}\mathbf{D}.
\end{align*}
Therefore, using \cref{prop:1},
\begin{align*}
\tilde{\rho}_{k}&=\frac{1}{n}\mathrm{Tr}\left[\mathbf{\Sigma}\mathbf{M}^{-1}\bar{\mathbf Q}\mathbf{A}_{k+1}\bar{\mathbf Q}\right]=\frac{1}{n}\mathrm{Tr}\left[\bar{\mathbf Q}\mathbf{\Sigma}\mathbf{M}^{-1}\bar{\mathbf Q}\mathbf{A}_{k+1}\right]\\&\to\frac{1}{n}\mathrm{Tr}\left[\bar{\mathbf Q}\mathbf{\Sigma}\mathbf{M}^{-1}\bar{\mathbf Q}\left[\bar{\mathbf Q}^{2\left(T-k\right)}\mathbf{M}^{-1}+\sum_{i=k+1}^{T}\bar{\mathbf Q}^{2\left(i-k-1\right)}\tilde{\mu}\tilde{\rho_i}\mathbf{D}\right]\right]\\&=\frac{1}{n}\mathrm{Tr}\left[\bar{\mathbf Q}^{2\left(T-k+1\right)}\mathbf{\Sigma}\mathbf{M}^{-2}\right]+\frac{1}{n}\mathrm{Tr}\left[\sum_{i=k+1}^{T}\bar{\mathbf Q}^{2\left(i-k+1\right)}\mathbf{\Sigma}^{2}\mathbf{M}^{-2}\tilde{\mu}\tilde{\rho_i}\right],
\end{align*}
almost surely.

And term 2 in \cref{eq:22}, becomes,
\begin{align*}
    &\frac{1}{n^{2}}\left[\sum_{k=1}^{T}\mathbf{z}_{k}^{\top}\mathbf{Z}_{k}\mathbf{\Sigma}^{\frac{1}{2}}\mathbf{P}_{k}\left[\prod_{m=T}^{k+1}\mathbf{P}_{m}\mathbf{M}\right]^{\top}\sum_{k'=1}^{T}\left[\prod_{m=T}^{k'+1}\mathbf{P}_{m}\mathbf{M}\right]\mathbf{P}_{k'}\mathbf{\Sigma}^{\frac{1}{2}}\mathbf{Z}_{k'}^{\top}\mathbf{z}_{k'}\right]&\\&=\frac{1}{n^{2}}\left[\sum_{k,k'=1}^{T}\mathbf{z}_{k}^{\top}\mathbf{Z}_{k}\mathbf{\Sigma}^{\frac{1}{2}}\mathbf{P}_{k}\left[\prod_{m=T}^{k+1}\mathbf{P}_{m}\mathbf{M}\right]^{\top}\left[\prod_{m=T}^{k'+1}\mathbf{P}_{m}\mathbf{M}\right]\mathbf{P}_{k'}\mathbf{\Sigma}^{\frac{1}{2}}\mathbf{Z}_{k'}^{\top}\mathbf{z}_{k'}\right]\\&\stackrel{\explain{1}}{\xrightarrow[]{P}}\frac{1}{n^{2}}\left[\sum_{k=1}^{T}\mathbf{z}_{k}^{\top}\mathbf{Z}_{k}\mathbf{\Sigma}^{\frac{1}{2}}\mathbf{P}_{k}\left[\prod_{m=T}^{k+1}\mathbf{P}_{m}\mathbf{M}\right]^{\top}\left[\prod_{m=T}^{k+1}\mathbf{P}_{m}\mathbf{M}\right]\mathbf{P}_{k}\mathbf{\Sigma}^{\frac{1}{2}}\mathbf{Z}_{k}^{\top}\mathbf{z}_{k}\right]\\&=\frac{1}{n^{2}}\left[\sum_{k=1}^{T}\mathrm{Tr}\left[\mathbf{z}_{k}^{\top}\mathbf{Z}_{k}\mathbf{\Sigma}^{\frac{1}{2}}\mathbf{P}_{k}\left[\prod_{m=T}^{k+1}\mathbf{P}_{m}\mathbf{M}\right]^{\top}\left[\prod_{m=T}^{k+1}\mathbf{P}_{m}\mathbf{M}\right]\mathbf{P}_{k}\mathbf{\Sigma}^{\frac{1}{2}}\mathbf{Z}_{k}^{\top}\mathbf{z}_{k}\right]\right]\\&=\frac{1}{n^{2}}\left[\sum_{k=1}^{T}\mathrm{Tr}\left[\mathbf{z}_{k}\mathbf{z}_{k}^{\top}\mathbf{Z}_{k}\mathbf{\Sigma}^{\frac{1}{2}}\mathbf{P}_{k}\left[\prod_{m=T}^{k+1}\mathbf{P}_{m}\mathbf{M}\right]^{\top}\left[\prod_{m=T}^{k+1}\mathbf{P}_{m}\mathbf{M}\right]\mathbf{P}_{k}\mathbf{\Sigma}^{\frac{1}{2}}\mathbf{Z}_{k}^{\top}\right]\right]\\&=\frac{1}{n^{2}}\sum_{k=1}^{T}\mathrm{Tr}\left[\mathbf{z}_{k}\mathbf{z}_{k}^{\top}\left[\mathbf{Z}_{k}\mathbf{\Sigma}^{\frac{1}{2}}\mathbf{P}_{k}\left[\prod_{m=T}^{k+1}\mathbf{P}_{m}\mathbf{M}\right]^{\top}\left[\prod_{m=T}^{k+1}\mathbf{P}_{m}\mathbf{M}\right]\mathbf{P}_{k}\mathbf{\Sigma}^{\frac{1}{2}}\mathbf{Z}_{k}^{\top}\right]\right]\\&\stackrel{\explain{2}}{\xrightarrow[]{P}}\frac{\labelnoise}{n^{2}}\left[\sum_{k=1}^{T}\mathrm{Tr}\left[\mathbf{Z}_{k}\mathbf{\Sigma}^{\frac{1}{2}}\mathbf{P}_{k}\left[\prod_{m=T}^{k+1}\mathbf{P}_{m}\mathbf{M}\right]^{\top}\left[\prod_{m=T}^{k+1}\mathbf{P}_{m}\mathbf{M}\right]\mathbf{P}_{k}\mathbf{\Sigma}^{\frac{1}{2}}\mathbf{Z}_{k}^{\top}\right]\right]\\&=\frac{\labelnoise}{n^{2}}\left[\sum_{k=1}^{T}\mathrm{Tr}\left[\mathbf{P}_{k}\mathbf{\Sigma}^{\frac{1}{2}}\mathbf{Z}_{k}^{\top}\mathbf{Z}_{k}\mathbf{\Sigma}^{\frac{1}{2}}\mathbf{P}_{k}\left[\prod_{m=T}^{k+1}\mathbf{P}_{m}\mathbf{M}\right]^{\top}\left[\prod_{m=T}^{k+1}\mathbf{P}_{m}\mathbf{M}\right]\right]\right]\\&=\frac{\labelnoise}{n^{2}}\mathrm{Tr}\left[\sum_{k=1}^{T}\left[\mathbf{P}_{k}\mathbf{\Sigma}^{\frac{1}{2}}\mathbf{Z}_{k}^{\top}\mathbf{Z}_{k}\mathbf{\Sigma}^{\frac{1}{2}}\mathbf{P}_{k}\right]\left[\left[\prod_{m=T}^{k+1}\mathbf{P}_{m}\mathbf{M}\right]^{\top}\left[\prod_{m=T}^{k+1}\mathbf{P}_{m}\mathbf{M}\right]\right]\right].
\end{align*}
Where [1] and [2] follows since the noise $\mathbf z_k$ is sample independently of the data matrices $\mathbf X_1,\ldots, \mathbf X_T$ and $\mathbb E[\mathbf z_k\mathbf z_k^\top]=\labelnoise\mathbf I$, with the weak law of large numbers. 
Note that, 
\begin{align*}
    &\mathbf{P}_{k}\mathbf{\Sigma}^{\frac{1}{2}}\mathbf{Z}_{k}^{\top}\mathbf{Z}_{k}\mathbf{\Sigma}^{\frac{1}{2}}\mathbf{P}_{k}&\\&=n^{2}\left(\mathbf{\Sigma}^{\frac{1}{2}}\mathbf{Z}_{k}^{\top}\mathbf{Z}_{k}\mathbf{\Sigma}^{\frac{1}{2}}+n\mathbf{M}\right)^{-1}\left(\mathbf{\Sigma}^{\frac{1}{2}}\mathbf{Z}_{k}^{\top}\mathbf{Z}_{k}\mathbf{\Sigma}^{\frac{1}{2}}+n\mathbf{M}-n\mathbf{M}\right)\left(\mathbf{\Sigma}^{\frac{1}{2}}\mathbf{Z}_{k}^{\top}\mathbf{Z}_{k}\mathbf{\Sigma}^{\frac{1}{2}}+n\mathbf{M}\right)^{-1}\\&=n\left(\mathbf{P}_{k}-\mathbf{P}_{k}\mathbf{M}\mathbf{P}_{k}\right)
\end{align*}
Therefore using \cref{thm:zhao_resolvent}, 
\begin{align*}
&\overline{\frac{1}{n}\mathbf{P}_{k}\mathbf{\Sigma}^{\frac{1}{2}}\mathbf{Z}_{k}^{\top}\mathbf{Z}_{k}\mathbf{\Sigma}^{\frac{1}{2}}\mathbf{P}_{k}}=\overline{\mathbf{P}_{k}-\mathbf{P}_{k}\mathbf{M}\mathbf{P}_{k}}=\mathbf{M}^{-\frac{1}{2}}\bar{\mathbf{Q}}\mathbf{M}^{-\frac{1}{2}}-\overline{\mathbf{M}^{-\frac{1}{2}}\mathbf{Q}^{2}\mathbf{M}^{-\frac{1}{2}}}\\&=\mathbf{M}^{-\frac{1}{2}}\left(\bar{\mathbf{Q}}-\left(\bar{\mathbf{Q}}^{2}+\frac{1}{n}\mathrm{Tr}\left[\mathbf\Sigma\mathbf{M}^{-1}\bar{\mathbf{Q}}^{2}\right]\left[\left[1+\frac{1}{n}\mathrm{Tr}\left(\mathbf\Sigma\mathbf{M}^{-1}\bar{\mathbf{Q}}\right)\right]^{2}-\frac{1}{n}\mathrm{Tr}\left(\mathbf\Sigma\mathbf{M}^{-1}\bar{\mathbf{Q}}\right)^{2}\right]^{-1}\mathbf\Sigma\mathbf{M}^{-1}\bar{\mathbf{Q}}^{2}\right)\right)\mathbf{M}^{-\frac{1}{2}}.
\end{align*}
Since both $\left[\prod_{m=T}^{k+1}\mathbf{P}_{m}\mathbf{M}\right]^{\top}\left[\prod_{m=T}^{k+1}\mathbf{P}_{m}\mathbf{M}\right]$ and $\mathbf{P}_{k}\mathbf{\Sigma}^{\frac{1}{2}}\mathbf{Z}_{k}^{\top}\mathbf{Z}_{k}\mathbf{\Sigma}^{\frac{1}{2}}\mathbf{P}_{k}$ have almost surely finite operator norm, then by the properties of deterministic equivalent we have 
\begin{align*}
&\frac{\labelnoise}{n^{2}}\mathrm{Tr}\left[\sum_{k=1}^{T}\mathbf{P}_{k}\mathbf{\Sigma}^{\frac{1}{2}}\mathbf{Z}_{k}^{\top}\mathbf{Z}_{k}\mathbf{\Sigma}^{\frac{1}{2}}\mathbf{P}_{k}\left[\prod_{m=T}^{k+1}\mathbf{P}_{m}\mathbf{M}\right]^{\top}\left[\prod_{m=T}^{k+1}\mathbf{P}_{m}\mathbf{M}\right]\right]\\&\xrightarrow[]{P}
\frac{\labelnoise}{n}\mathrm{Tr}\left[\sum_{k=1}^{T}\overline{\frac{1}{n}\mathbf{P}_{k}\mathbf{\Sigma}^{\frac{1}{2}}\mathbf{Z}_{k}^{\top}\mathbf{Z}_{k}\mathbf{\Sigma}^{\frac{1}{2}}\mathbf{P}_{k}}\left[\prod_{m=T}^{k+1}\mathbf{P}_{m}\mathbf{M}\right]^{\top}\left[\prod_{m=T}^{k+1}\mathbf{P}_{m}\mathbf{M}\right]\right]\\&\xrightarrow[]{P}
\frac{\labelnoise}{n}\mathrm{Tr}\left[\sum_{k=1}^{T}\overline{\frac{1}{n}\mathbf{P}_{k}\mathbf{\Sigma}^{\frac{1}{2}}\mathbf{Z}_{k}^{\top}\mathbf{Z}_{k}\mathbf{\Sigma}^{\frac{1}{2}}\mathbf{P}_{k}}\overline{\left[\prod_{m=T}^{k+1}\mathbf{P}_{m}\mathbf{M}\right]^{\top}\left[\prod_{m=T}^{k+1}\mathbf{P}_{m}\mathbf{M}\right]}\right],
\end{align*}
where in each step we condition on the event that each one of the matrices has finite operator norm and then use \cref{def:1}.
Using
\begin{align*}
\left[\prod_{m=T}^{k}\mathbf{P}_{m}\mathbf{M}_{m}\right]^{\top}\left[\prod_{m=T}^{k}\mathbf{P}_{m}\mathbf{M}_{m}\right]=\mathbf{M}^{\frac{1}{2}}\mathbf{A}_{k}\mathbf{M}^{\frac{1}{2}}\notag\asymp\bar{\mathbf{Q}}^{2\left(T-k+1\right)}+\sum_{i=k}^{T}\bar{\mathbf{Q}}^{2\left(i-k+1\right)}\tilde{\mu}\tilde{\rho}_{i},
\end{align*}
and combining all the terms in \cref{eq:22}.

\begin{align*}
&G^{\mathrm{gen}}_T\xrightarrow[]{P}\left(\mathbf{w}_{0}-\mathbf{w}^{\star}\right)^{\top}\left(\bar{\mathbf Q}^{2T}+\sum_{k=1}^{T}\bar{\mathbf Q}^{2k}\tilde{\mu}\tilde{\rho}_{k}\right)\left(\mathbf{w}_{0}-\mathbf{w}^{\star}\right)\\&+\frac{\labelnoise}{n}\mathrm{Tr}\left[\sum_{k=1}^{T}\mathbf{M}^{-1}\bar{\mathbf Q}\left[\mathbf{I}-\bar{\mathbf Q}\left[\mathbf{I}+\frac{1}{n}\mathrm{Tr}\left[\mathbf{\Sigma}\mathbf{M}^{-1}\bar{\mathbf Q}^{2}\right]\tilde{\mu}\mathbf{\Sigma}\mathbf{M}^{-1}\right]\right]\left[\bar{\mathbf Q}^{2\left(T-k\right)}+\sum_{i=k+1}^{T}\bar{\mathbf Q}^{2\left(i-k\right)}\tilde{\mu}\tilde{\rho}_{i}\right]\right].
\end{align*}
\end{proof}
\newpage
\subsection{Closed-Form Generalization Loss with Task-Dependent Diagonal Regularization}
\label{app:mul}
In this appendix, we provide an \emph{exact} generalization analysis under arbitrary positive-definite diagonal regularization matrices $\{\mathbf M_t\}_{t=1}^T$ (spectrally bounded). The resulting expression accommodates general task-dependent diagonal covariance matrices $\mathbf \Sigma_t$ (spectrally bounded), in contrast to the isotropic setting considered in \cref{sec:statistical_setting}. 
Specifically, here we may assume that the data for task $t$ is generated as $\mathbf X_t = \mathbf Z_t\mathbf \Sigma_t^{1/2}$, where $\mathbf Z_t \in \mathbb{R}^{n \times d}$ has i.i.d.\ entries with zero mean and unit variance.

\begin{theorem}    
Under \cref{assum:1}, the generalization loss after $T$ tasks of the generalized ridge algorithm (\cref{schm:generalized}) with task dependent diagonal regularization, satisfies, as $n,d\to\infty$, 
\begin{equation}
\label{eq:46}
\begin{aligned}
&G_{T}^{\mathrm{gen}}\xrightarrow{P}\left(\mathbf{w}_{0}-\mathbf{w}^{\star}\right)^{\top}\left[\left[\prod_{m=1}^{T}\bar{\mathbf{Q}}_{m}\right]^{2}+\sum_{i=2}^{T}\mathbf{\Sigma}_{i}\left[\prod_{r=1}^{i}\bar{\mathbf{Q}}_{r}\right]^{2}\tilde{\mu}_{i}\tilde{\rho}_{i}+\tilde{\rho}_{1}\tilde{\mu}_{1}\bar{\mathbf{Q}}_{1}^{2}\right]\left(\mathbf{w}_{0}-\mathbf{w}^{\star}\right)\\&\hspace{6em}+\frac{\labelnoise}{n}\mathrm{Tr}\Bigg[\sum_{k=1}^{T}\bar{\mathbf{Q}}_{k}\mathbf{M}_{k}^{-1}\left[\mathbf{I}-\bar{\mathbf{Q}}_{k}\left[\mathbf{I}+\frac{1}{n}\mathrm{Tr}\left[\mathbf{\Sigma}_{k}^{-1}\mathbf{M}_{k}^{-1}\bar{\mathbf{Q}}_{k}^{2}\right]\tilde{\mu}_{k}\mathbf{M}_{k}^{-1}\right]\right]\\&\hspace{12em}\left[\left[\prod_{m=k+1}^{T}\bar{\mathbf{Q}}_{m}\right]^{2}+\sum_{i=k+2}^{T}\mathbf{\Sigma}_{i}\left[\prod_{r=k+1}^{i}\bar{\mathbf{Q}}_{r}\right]^{2}\tilde{\mu}_{i}\tilde{\rho}_{i}+\tilde{\rho}_{k+1}\tilde{\mu}_{k+1}\bar{\mathbf{Q}}_{k+1}^{2}\right]\Bigg].
\end{aligned}
\end{equation}

where $\bar{\mathbf Q}_j=\left(\mathbf{I}+\tilde{m}^{j}_{p}\left(-1\right)\mathbf{\Sigma}_j\mathbf{M}_j^{-1}\right)^{-1},$
with $\tilde{m}^{j}_{p}\left(-1\right)$ being the unique positive solution of
$$
\frac{1}{\tilde{m}^{j}_{p}\left(-1\right)}=1+\frac{1}{n}\sum_{i=1}^{d}\frac{\sigma_{i}^j({m_{i}^{j}})^{-1}}{1+\tilde{m}^{j}_{p}\left(-1\right)\sigma_{i}^j({m_{i}^{j}})^{-1}},$$
and,
\begin{align*} 
&\tilde{\rho}_{j}=\frac{1}{n}\mathrm{Tr}\left[\mathbf{\Sigma}_{j}\mathbf{M}_{j}^{-2}\left[\left[\prod_{m=j}^{T}\bar{\mathbf{Q}}_{m}\right]^{2}+\sum_{i=j+1}^{T}\mathbf{\Sigma}_{i}\left[\prod_{r=j}^{i}\bar{\mathbf{Q}}_{r}\right]^{2}\tilde{\mu}_{i}\tilde{\rho}_{i}\right]\right]\\&\tilde{\mu}_{j}=\left[\left[1+\frac{1}{n}\mathrm{Tr}\left(\mathbf{\Sigma}_{j}\mathbf{M}_{j}^{-1}\bar{\mathbf{Q}}_{j}\right)\right]^{2}-\frac{1}{n}\mathrm{Tr}\left(\mathbf{\Sigma}_{j}\mathbf{M}_{j}^{-1}\bar{\mathbf{Q}}_{j}\right)^{2}\right]^{-1}.
\end{align*} 
Here $\sigma_i^j,m_i^j$ denote respectively, the $i$'th eigenvalues of the matrices $\mSigma_j, \mathbf M_j$. 
\end{theorem}

\begin{proof}
For each $t\in [1,T]$, we are interested in solving 
\begin{align*}
\mathbf{w}^{\mathrm{gen}}_{t}=\min_{\mathbf{w}\in\mathbb{R}^{d}}\frac{1}{n}\left\Vert \mathbf{X}_{t}\mathbf{w}-\mathbf{y}\right\Vert _{2}^{2}+\left\Vert \mathbf{w}-\mathbf{w}^{\mathrm{gen}}_{t-1}\right\Vert^2_{\mathbf{M}_{t}}.
\end{align*}
By \cref{lem:4} we have,
\begin{align*}
\mathbf{w}^{\mathrm{gen}}_{t}=\left(\frac{1}{n}\mathbf{X}_{t}^{\top}\mathbf{X}_{t}+\mathbf{M}_{t}\right)^{-1}\left(\frac{1}{n}\mathbf{X}_{t}^{\top}\mathbf{y}_{t}+\mathbf{M}_{t}\mathbf{w}^{\mathrm{gen}}_{t-1}\right)
\end{align*}
so,
\begin{align*}
\mathbf{w}_{T}^{\mathrm{gen}}-\mathbf{w}^{\star}&=\left(\frac{1}{n}\mathbf{X}_{T}^{\top}\mathbf{X}_{T}+\mathbf{M}_{T}\right)^{-1}\left(\frac{1}{n}\mathbf{X}_{T}^{\top}\left(\mathbf{X}_{T}\mathbf{w}^{\star}+\mathbf{z}_{T}\right)+\mathbf{M}_{T}\mathbf{w}_{T-1}^{\mathrm{gen}}\right)-\mathbf{w}^{\star}\\&=\left(\frac{1}{n}\mathbf{X}_{T}^{\top}\mathbf{X}_{T}+\mathbf{M}_{T}\right)^{-1}\left(\frac{1}{n}\mathbf{X}_{T}^{\top}\left(\mathbf{X}_{T}\mathbf{w}^{\star}+\mathbf{z}_{T}\right)+\mathbf{M}_{T}\mathbf{w}_{T-1}^{\mathrm{gen}}-\left(\frac{1}{n}\mathbf{X}_{T}^{\top}\mathbf{X}_{T}+\mathbf{M}_{T}\right)\mathbf{w}^{\star}\right)\\&=\left(\frac{1}{n}\mathbf{X}_{T}^{\top}\mathbf{X}_{T}+\mathbf{M}_{T}\right)^{-1}\left(\frac{1}{n}\mathbf{X}_{T}^{\top}\mathbf{z}_{T}+\mathbf{M}_{T}\left(\mathbf{w}_{T-1}^{\mathrm{gen}}-\mathbf{w}^{\star}\right)\right)\end{align*}\begin{align*}&=\mathbf{P}_{T}\mathbf{M}_{T}\left(\mathbf{w}_{T-1}^{\mathrm{gen}}-\mathbf{w}^{\star}\right)+\mathbf{P}_{T}\frac{1}{n}\mathbf{X}_{T}^{\top}\mathbf{z}_{T}\\&=\ldots\\&=\left[\prod_{m=T}^{1}\mathbf{P}_{m}\mathbf{M}_{m}\right]\left(\mathbf{w}_{0}-\mathbf{w}^{\star}\right)+\sum_{k=1}^{T}\left[\prod_{m=T}^{k+1}\mathbf{P}_{m}\mathbf{M}_{m}\right]\mathbf{P}_{k}\frac{1}{n}\mathbf{X}_{k}^{\top}\mathbf{z}_{k},
\end{align*}
where $\mathbf{P}_{t}=\left(\frac{1}{n}\mathbf{X}_{t}^{\top}\mathbf{X}_{t}+\mathbf{M}_{t}\right)^{-1}=\mathbf{M}_{t}^{-\frac{1}{2}}\left(\mathbf{I}+\frac{1}{n}\mathbf{M}_{t}^{-\frac{1}{2}}\mathbf{X}_{t}^{\top}\mathbf{X}_{t}\mathbf{M}_{t}^{-\frac{1}{2}}\right)^{-1}\mathbf{M}_{t}^{-\frac{1}{2}}$ for all $t \in [1,T]$.

Note that $\X_t$ can be re-written as $\X_t=\mathbf Z_t\mSigma_t^{1/2}$ where $\mathbf Z_t$ has i.i.d. entries with zero mean and unit variance.
Substituting into $G_T^{\mathrm{gen}}=\left\Vert\mathbf{w}^{\mathrm{gen}}_{T}-\mathbf{w}^{\star}\right\Vert ^{2}$,
\begin{align}
G_T^{\mathrm{gen}}&=\left[\left[\prod_{m=T}^{1}\mathbf{P}_{m}\mathbf{M}_{m}\right]\left(\mathbf{w}_{0}-\mathbf{w}^{\star}\right)+\sum_{k=1}^{T}\left[\prod_{m=T}^{k+1}\mathbf{P}_{m}\mathbf{M}_{m}\right]\mathbf{P}_{k}\frac{1}{n}\mathbf{X}_{k}^{\top}\mathbf{z}_{k}\right]^{\top}\notag\\&\qquad\qquad\left[\left[\prod_{m=T}^{1}\mathbf{P}_{m}\mathbf{M}_{m}\right]\left(\mathbf{w}_{0}-\mathbf{w}^{\star}\right)+\sum_{k=1}^{T}\left[\prod_{m=T}^{k+1}\mathbf{P}_{m}\mathbf{M}_{m}\right]\mathbf{P}_{k}\frac{1}{n}\mathbf{X}_{k}^{\top}\mathbf{z}_{k}\right]\notag\\&=\underbrace{\left(\mathbf{w}_{0}-\mathbf{w}^{\star}\right)^{\top}\left[\prod_{m=T}^{1}\mathbf{P}_{m}\mathbf{M}_{m}\right]^{\top}\left[\prod_{m=T}^{1}\mathbf{P}_{m}\mathbf{M}_{m}\right]\left(\mathbf{w}_{0}-\mathbf{w}^{\star}\right)}_{\text{term 1}}\notag\\&\hspace{4em}+\underbrace{\frac{\labelnoise}{n^{2}}\sum_{k=1}^{T}\mathbf{z}_{k}^{\top}\mathbf{X}_{k}\mathbf{P}_{k}\left[\prod_{m=T}^{k+1}\mathbf{P}_{m}\mathbf{M}_{m}\right]^{\top}\sum_{k'=1}^{T}\left[\prod_{m=T}^{k'+1}\mathbf{P}_{m}\mathbf{M}_{m}\right]\mathbf{P}_{k'}\mathbf{X}_{k'}^{\top}\mathbf{z}_{k'}}_{\text{term 2}}
    \label{eq:25}.
\end{align}
Let $\mathbf{A}_{k}:=\mathbf{M}_{k}^{-\frac{1}{2}}\left[\prod_{m=T}^{k+1}\mathbf{P}_{m}\mathbf{M}_{m}\right]^{\top}\left[\prod_{m=T}^{k+1}\mathbf{P}_{m}\mathbf{M}_{m}\right]\mathbf{M}_{k}^{-\frac{1}{2}}$.
Using this definition we can find the equivalent to $\left[\prod_{m=T}^{1}\mathbf{P}_{m}\mathbf{M}_{m}\right]^{\top}\left[\prod_{m=T}^{1}\mathbf{P}_{m}\mathbf{M}_{m}\right]$, in the following manner,
\begin{equation}
\label{eq:24}
\begin{aligned}
&\left[\prod_{m=T}^{1}\mathbf{P}_{m}\mathbf{M}_{m}\right]^{\top}\left[\prod_{m=T}^{1}\mathbf{P}_{m}\mathbf{M}_{m}\right]=\mathbf{M}_{1}\left[\mathbf{P}_{1}\ldots\mathbf{M}_{T}\mathbf{P}_{T}\mathbf{P}_{T}\ldots\mathbf{P}_{1}\right]\mathbf{M}_{1}&\\&=\mathbf{M}_{1}^{\frac{1}{2}}\left[\underbrace{\left[\mathbf{M}_{1}^{\frac{1}{2}}\mathbf{P}_{1}\mathbf{M}_{1}^{\frac{1}{2}}\right]}_{\mathbf{Q}_{1}(-1)}\underbrace{\mathbf{M}_{1}^{-\frac{1}{2}}\mathbf{M}_{2}\ldots\mathbf{M}_{T}\mathbf{P}_{T}\mathbf{P}_{T}\mathbf{M}_{T}\ldots\mathbf{M}_{2}\mathbf{M}_{1}^{-\frac{1}{2}}}_{\mathbf{A}_{1}}\underbrace{\left[\mathbf{M}_{1}^{\frac{1}{2}}\mathbf{P}_{1}\mathbf{M}_{1}^{\frac{1}{2}}\right]}_{\mathbf{Q}_{1}(-1)}\right]\mathbf{M}_{1}^{\frac{1}{2}}
\end{aligned}
\end{equation}
where we let $\mathbf Q_t(z):=\left(\frac{1}{n}\mathbf{M}_t^{-\frac{1}{2}}\mathbf{\Sigma}_{t}^{\frac{1}{2}}\mathbf{Z}^{\top}\mathbf{Z}\mathbf{\Sigma}_{t}^{\frac{1}{2}}\mathbf{M}_t^{-\frac{1}{2}}-z\mathbf I\right)^{-1}$ with $\mathbf Z$ having i.i.d. entries with zero mean and unit variance, then by \cref{thm:romain}, $\mathbf Q_t(-1)$ admits a deterministic equivalent, where,
\begin{align*}
\mathbf Q_t\left(-1\right)\asymp\bar{\mathbf Q}_t\left(-1\right)=\left(\mathbf{I}+\tilde{m}^t_{p}\left(-1\right)\mathbf{\Sigma}_t\mathbf{M}_t^{-1}\right)^{-1}\text{ and }\frac{1}{\tilde{m}^{t}_{p}\left(-1\right)}=1+\frac{1}{n}\sum_{i=1}^{d}\frac{\sigma_{i}^t({m_{i}^{t}})^{-1}}{1+\tilde{m}^{t}_{p}\left(-1\right)\sigma_{i}^t({m_{i}^{t}})^{-1}}.
\end{align*}

Here $\sigma_i^t,m_i^t$ denote respectively, the $i$'th eigenvalues of the matrices $\mSigma_t, \mathbf M_t$. For notational convenience, we set
$\mathbf Q_t:= \mathbf Q_t(-1)$, and $\bar{\mathbf Q}_t:=\bar {\mathbf Q}_t(-1)$. Both matrices are diagonal since both $\mathbf \Sigma_t$ and $\mathbf M_t$ are diagonal for all $t\in [1,T]$. Moreover, as previously mentioned, we note that under our standing assumption, we have $\sigma^t_i=\featnoise$, for all $i$ and $t$. Nevertheless, our analysis accommodates non-isotropic covariance matrices.

Using \cref{thm:zhao_resolvent}, we get,
\begin{align*}
&\mathbf{A}_{k}=\mathbf{M}_{k}^{-\frac{1}{2}}\left[\prod_{m=T}^{k+1}\mathbf{P}_{m}\mathbf{M}_{m}\right]^{\top}\left[\prod_{m=T}^{k+1}\mathbf{P}_{m}\mathbf{M}_{m}\right]\mathbf{M}_{k}^{-\frac{1}{2}}&\\\asymp&\mathbf{M}_{k}^{-\frac{1}{2}}\mathbf{M}_{k+1}^{\frac{1}{2}}\left(\bar{\mathbf{Q}}_{k+1}\mathbf{A}_{k+1}\bar{\mathbf{Q}}_{k+1}+\tilde{\mu}_{k+1}\tilde{\rho}_{k+1}\mathbf{D}_{k+1}\right)\mathbf{M}_{k+1}^{\frac{1}{2}}\mathbf{M}_{k}^{-\frac{1}{2}}\end{align*}\begin{align*}
\asymp&\mathbf{M}_{k}^{-\frac{1}{2}}\mathbf{M}_{k+1}^{\frac{1}{2}}\Big(\bar{\mathbf{Q}}_{k+1}\Big(\mathbf{M}_{k+1}^{-\frac{1}{2}}\mathbf{M}_{k+2}^{\frac{1}{2}}\Big(\bar{\mathbf{Q}}_{k+2}\mathbf{A}_{k+2}\bar{\mathbf{Q}}_{k+2}\\&\hspace{8em}+\tilde{\mu}_{k+2}\tilde{\rho}_{k+2}\mathbf{D}_{k+2}\Big)\mathbf{M}_{k+2}^{\frac{1}{2}}\mathbf{M}_{k+1}^{-\frac{1}{2}}\Big)\bar{\mathbf{Q}}_{k+1}+\tilde{\mu}_{k+1}\tilde{\rho}_{k+1}\mathbf{D}_{k+1}\Big)\mathbf{M}_{k+1}^{\frac{1}{2}}\mathbf{M}_{k}^{-\frac{1}{2}}\\\asymp&\mathbf{M}_{k}^{-\frac{1}{2}}\mathbf{M}_{T}^{\frac{1}{2}}\left[\prod_{m=k+1}^{T}\bar{\mathbf{Q}}_{m}\right]\underbrace{\mathbf{A}_{T}}_{\mathbf{M}_{T}^{-1}}\left[\prod_{m=k+1}^{T}\bar{\mathbf{Q}}_{m}\right]\mathbf{M}_{T}^{\frac{1}{2}}\mathbf{M}_{k}^{-\frac{1}{2}}\\&\hspace{6em}+\mathbf{M}_{k}^{-\frac{1}{2}}\sum_{i=k+1}^{T}\mathbf{M}_{i}^{\frac{1}{2}}\left[\prod_{r=k+1}^{i-1}\bar{\mathbf{Q}}_{r}\right]\tilde{\mu}_{i}\tilde{\rho}_{i}\mathbf{D}_{i}\left[\prod_{r=k+1}^{i-1}\bar{\mathbf{Q}}_{r}\right]\mathbf{M}_{i}^{\frac{1}{2}}\mathbf{M}_{k}^{-\frac{1}{2}},
\end{align*}
where 
\begin{align*}   
\tilde{\rho}_{k}&=\frac{1}{n}\mathrm{Tr}\left[\mathbf{\mathbf{\Sigma}}_{k}\mathbf{M}_{k}^{-1}\bar{\mathbf{Q}}_{k}\mathbf{A}_{k}\bar{\mathbf{Q}}_{k}\right],\\\tilde{\mu}_{k}&=\left[\left[1+\frac{1}{n}\mathrm{Tr}\left(\mathbf{\mathbf{\Sigma}}_{k}\mathbf{M}_{k}^{-1}\bar{\mathbf{Q}}_{k}\right)\right]^{2}-\frac{1}{n}\mathrm{Tr}\left(\mathbf{\mathbf{\Sigma}}_{k}\mathbf{M}_{k}^{-1}\bar{\mathbf{Q}}_{k}\right)^{2}\right]^{-1},\\\mathbf{D}_{k}&=\bar{\mathbf{Q}}_{k}\mathbf{\mathbf{\Sigma}}_{k}\mathbf{M}_{k}^{-1}\bar{\mathbf{Q}}_{k}.
\end{align*}

Substituting back into \cref{eq:24} and using \cref{prop:1},
\begin{align}
&\left[\prod_{m=T}^{k}\mathbf{P}_{m}\mathbf{M}_{m}\right]^{\top}\left[\prod_{m=T}^{k}\mathbf{P}_{m}\mathbf{M}_{m}\right]=\mathbf{M}_{k}^{\frac{1}{2}}\mathbf{Q}_{k}\mathbf{A}_{k}\mathbf{Q}_{k}\mathbf{M}_{k}^{\frac{1}{2}}\notag\\&\hspace{4em}\asymp\mathbf{M}_{k}^{\frac{1}{2}}\left[\bar{\mathbf{Q}}_{k}\mathbf{A}_{k}\bar{\mathbf{Q}}_{k}+\tilde{\rho}_{k}\tilde{\mu}_{k}\bar{\mathbf{Q}}_{k}\mathbf{\Sigma}_{k}\mathbf{M}_{k}^{-1}\bar{\mathbf{Q}}_{k}\right]\mathbf{M}_{k}^{\frac{1}{2}}\notag\\&\hspace{8em}\asymp\left[\prod_{m=k}^{T}\bar{\mathbf{Q}}_{m}\right]^{2}+\sum_{i=k+1}^{T}\left[\prod_{r=k}^{i}\bar{\mathbf{Q}}_{r}\right]^{2}\mathbf{\Sigma}_{i}\tilde{\mu}_{i}\tilde{\rho}_{i}+\tilde{\rho}_{k}\tilde{\mu}_{k}\mathbf{\Sigma}_{k}\bar{\mathbf{Q}}_{k}^{2}.
\label{eq:26}
\end{align}
Additionally, using \cref{prop:1}, we have, 
\begin{align*}
\tilde{\rho}_{k}&=\frac{1}{n}\mathrm{Tr}\left[\mathbf{\Sigma}_{k}\mathbf{M}_{k}^{-1}\bar{\mathbf{Q}}_{k}\mathbf{A}_{k}\bar{\mathbf{Q}}_{k}\right]=\frac{1}{n}\mathrm{Tr}\left[\mathbf{M}_{k}^{-1}\bar{\mathbf{Q}}_{k}^{2}\mathbf{A}_{k}\right]\\&\to\frac{1}{n}\mathrm{Tr}\left[\mathbf{\Sigma}_{k}\mathbf{M}_{k}^{-2}\left[\left[\prod_{m=k}^{T}\bar{\mathbf{Q}}_{m}\right]^{2}+\sum_{i=k+1}^{T}\mathbf{\Sigma}_{i}\left[\prod_{r=k}^{i}\bar{\mathbf{Q}}_{r}\right]^{2}\tilde{\mu}_{i}\tilde{\rho}_{i}\right]\right].
\end{align*}
almost surely.

And term 2 in \cref{eq:25} is,
\begin{align*}
&\frac{1}{n^{2}}\left[\sum_{k,k'=1}^{T}\mathbf{z}_{k}^{\top}\mathbf{W}_{k}\mathbf{\Sigma}_{k}^{\frac{1}{2}}\mathbf{P}_{k}\left[\prod_{m=T}^{k+1}\mathbf{P}_{m}\mathbf{M}_{m}\right]^{\top}\left[\prod_{m=T}^{k'+1}\mathbf{P}_{m}\mathbf{M}_{m}\right]\mathbf{P}_{k'}\mathbf{\Sigma}_{k'}^{\frac{1}{2}}\mathbf{W}_{k'}^{\top}\mathbf{z}_{k'}\right]&\\&\stackrel{\explain{1}}{\xrightarrow[]{P}}\frac{1}{n^{2}}\left[\sum_{k=1}^{T}\mathbf{z}_{k}^{\top}\mathbf{W}_{k}\mathbf{\Sigma}_{k}^{\frac{1}{2}}\mathbf{P}_{k}\left[\prod_{m=T}^{k+1}\mathbf{P}_{m}\mathbf{M}_{m}\right]^{\top}\left[\prod_{m=T}^{k+1}\mathbf{P}_{m}\mathbf{M}_{m}\right]\mathbf{P}_{k}\mathbf{\Sigma}_{k}^{\frac{1}{2}}\mathbf{W}_{k}^{\top}\mathbf{z}_{k}\right]\\&=\frac{1}{n^{2}}\left[\sum_{k=1}^{T}\mathrm{Tr}\left[\mathbf{z}_{k}^{\top}\mathbf{W}_{k}\mathbf{\Sigma}_{k}^{\frac{1}{2}}\mathbf{P}_{k}\left[\prod_{m=T}^{k+1}\mathbf{P}_{m}\mathbf{M}_{m}\right]^{\top}\left[\prod_{m=T}^{k+1}\mathbf{P}_{m}\mathbf{M}_{m}\right]\mathbf{P}_{k}\mathbf{\Sigma}_{k}^{\frac{1}{2}}\mathbf{W}_{k}^{\top}\mathbf{z}_{k}\right]\right]\\&=\frac{1}{n^{2}}\left[\sum_{k=1}^{T}\mathrm{Tr}\left[\mathbf{z}_{k}\mathbf{z}_{k}^{\top}\mathbf{W}_{k}\mathbf{\Sigma}_{k}^{\frac{1}{2}}\mathbf{P}_{k}\left[\prod_{m=T}^{k+1}\mathbf{P}_{m}\mathbf{M}_{m}\right]^{\top}\left[\prod_{m=T}^{k+1}\mathbf{P}_{m}\mathbf{M}_{m}\right]\mathbf{P}_{k}\mathbf{\Sigma}_{k}^{\frac{1}{2}}\mathbf{W}_{k}^{\top}\right]\right]\\&=\frac{1}{n^{2}}\sum_{k=1}^{T}\mathrm{Tr}\left[\mathbf{z}_{k}\mathbf{z}_{k}^{\top}\left[\mathbf{W}_{k}\mathbf{\Sigma}_{k}^{\frac{1}{2}}\mathbf{P}_{k}\left[\prod_{m=T}^{k+1}\mathbf{P}_{m}\mathbf{M}_{m}\right]^{\top}\left[\prod_{m=T}^{k+1}\mathbf{P}_{m}\mathbf{M}_{m}\right]\mathbf{P}_{k}\mathbf{\Sigma}_{k}^{\frac{1}{2}}\mathbf{W}_{k}^{\top}\right]\right]\end{align*}\begin{align*}&\stackrel{\explain{2}}{\xrightarrow[]{P}}\frac{\labelnoise}{n^{2}}\left[\sum_{k=1}^{T}\mathrm{Tr}\left[\mathbf{W}_{k}\mathbf{\Sigma}_{k}^{\frac{1}{2}}\mathbf{P}_{k}\left[\prod_{m=T}^{k+1}\mathbf{P}_{m}\mathbf{M}_{m}\right]^{\top}\left[\prod_{m=T}^{k+1}\mathbf{P}_{m}\mathbf{M}_{m}\right]\mathbf{P}_{k}\mathbf{\Sigma}_{k}^{\frac{1}{2}}\mathbf{W}_{k}^{\top}\right]\right]\\&=\frac{\labelnoise}{n^{2}}\left[\sum_{k=1}^{T}\mathrm{Tr}\left[\mathbf{P}_{k}\mathbf{\Sigma}_{k}^{\frac{1}{2}}\mathbf{W}_{k}^{\top}\mathbf{W}_{k}\mathbf{\Sigma}_{k}^{\frac{1}{2}}\mathbf{P}_{k}\left[\prod_{m=T}^{k+1}\mathbf{P}_{m}\mathbf{M}_{m}\right]^{\top}\left[\prod_{m=T}^{k+1}\mathbf{P}_{m}\mathbf{M}_{m}\right]\right]\right]\\&=\frac{\labelnoise}{n^{2}}\mathrm{Tr}\left[\sum_{k=1}^{T}\left[\mathbf{P}_{k}\mathbf{\Sigma}_{k}^{\frac{1}{2}}\mathbf{W}_{k}^{\top}\mathbf{W}_{k}\mathbf{\Sigma}_{k}^{\frac{1}{2}}\mathbf{P}_{k}\right]\left[\left[\prod_{m=T}^{k+1}\mathbf{P}_{m}\mathbf{M}_{m}\right]^{\top}\left[\prod_{m=T}^{k+1}\mathbf{P}_{m}\mathbf{M}_{m}\right]\right]\right]
\end{align*}
where [1] and [2] follows since the noise $\mathbf z_k$ is sample independently of the data matrices $\mathbf X_1,\ldots, \mathbf X_T$ and $\mathbb E[\mathbf z_k\mathbf z_k^\top]=\labelnoise\mathbf I$ with the weak law of large numbers. 

We already know the equivalent corresponding to $\left[\prod_{m=T}^{k+1}\mathbf{P}_{m}\mathbf{M}_{m}\right]^{\top}\left[\prod_{m=T}^{k+1}\mathbf{P}_{m}\mathbf{M}_{m}\right]$ from \cref{eq:26}. Hence, we now focus on the remaining term. 
\begin{align*}
&\mathbf{P}_{k}\mathbf{\Sigma}_{k}^{\frac{1}{2}}\mathbf{W}_{k}^{\top}\mathbf{W}_{k}\mathbf{\Sigma}_{k}^{\frac{1}{2}}\mathbf{P}_{k}&\\&=n^{2}\left(\mathbf{\Sigma}_{k}^{\frac{1}{2}}\mathbf{W}_{k}^{\top}\mathbf{W}_{k}\mathbf{\Sigma}_{k}^{\frac{1}{2}}+n\mathbf{M}_{k}\right)^{-1}\left(\mathbf{\Sigma}_{k}^{\frac{1}{2}}\mathbf{W}_{k}^{\top}\mathbf{W}_{k}\mathbf{\Sigma}_{k}^{\frac{1}{2}}+n\mathbf{M}_{k}-n\mathbf{M}_{k}\right)\left(\mathbf{\Sigma}_{k}^{\frac{1}{2}}\mathbf{W}_{k}^{\top}\mathbf{W}_{k}\mathbf{\Sigma}_{k}^{\frac{1}{2}}+n\mathbf{M}_{k}\right)^{-1}\\&=n\left(\mathbf{P}_{k}-\mathbf{P}_{k}\mathbf{M}_{k}\mathbf{P}_{k}\right).
\end{align*}
Thus,
\begin{align*}
&\overline{\frac{1}{n}\mathbf{P}_{k}\mathbf{\Sigma}_{k}^{\frac{1}{2}}\mathbf{W}_{k}^{\top}\mathbf{W}_{k}\mathbf{\Sigma}_{k}^{\frac{1}{2}}\mathbf{P}_{k}}=\mathbf{M}_{k}^{-\frac{1}{2}}\bar{\mathbf{Q}}_{k}\mathbf{M}_{k}^{-\frac{1}{2}}-\overline{\mathbf{M}_{k}^{-\frac{1}{2}}{\mathbf{Q}}_{k}^{2}\mathbf{M}_{k}^{-\frac{1}{2}}}\\&=\mathbf{M}_{k}^{-\frac{1}{2}}\Bigg(\bar{\mathbf{Q}}_{k}-\Big[\bar{\mathbf{Q}}_{k}^{2}+\frac{1}{n}\mathrm{Tr}\left[\mathbf{\Sigma}_{k}\mathbf{M}_{k}^{-1}\bar{\mathbf{Q}}_{k}^{2}\right]\\&\hspace{10em}\left[\left[1+\frac{1}{n}\mathrm{Tr}\left(\mathbf{\Sigma}_{k}\mathbf{M}_{k}^{-1}\bar{\mathbf{Q}}_{k}\right)\right]^{2}-\frac{1}{n}\mathrm{Tr}\left(\mathbf{\Sigma}_{k}\mathbf{M}_{k}^{-1}\bar{\mathbf{Q}}_{k}\right)^{2}\right]^{-1}\mathbf{\Sigma}_{k}\mathbf{M}_{k}^{-1}\bar{\mathbf{Q}}_{k}^{2}\Big]\Bigg)\mathbf{M}_{k}^{-\frac{1}{2}}.
\end{align*}
Since both $\left[\prod_{m=T}^{k+1}\mathbf{P}_{m}\mathbf{M}_{m}\right]^{\top}\left[\prod_{m=T}^{k+1}\mathbf{P}_{m}\mathbf{M}_{m}\right]$ and $\mathbf{P}_{k}\mathbf{\Sigma}_{k}^{\frac{1}{2}}\mathbf{Z}_{k}^{\top}\mathbf{Z}_{k}\mathbf{\Sigma}_{k}^{\frac{1}{2}}\mathbf{P}_{k}$ have almost surely finite operator norm, then by the properties of deterministic equivalent we have 
\begin{align*}
&\frac{\labelnoise}{n^{2}}\mathrm{Tr}\left[\sum_{k=1}^{T}\mathbf{P}_{k}\mathbf{\Sigma}_{k}^{\frac{1}{2}}\mathbf{Z}_{k}^{\top}\mathbf{Z}_{k}\mathbf{\Sigma}_{k}^{\frac{1}{2}}\mathbf{P}_{k}\left[\prod_{m=T}^{k+1}\mathbf{P}_{m}\mathbf{M}_{m}\right]^{\top}\left[\prod_{m=T}^{k+1}\mathbf{P}_{m}\mathbf{M}_{m}\right]\right]\\&\xrightarrow[]{P}
\frac{\labelnoise}{n}\mathrm{Tr}\left[\sum_{k=1}^{T}\overline{\frac{1}{n}\mathbf{P}_{k}\mathbf{\Sigma}_{k}^{\frac{1}{2}}\mathbf{Z}_{k}^{\top}\mathbf{Z}_{k}\mathbf{\Sigma}_{k}^{\frac{1}{2}}\mathbf{P}_{k}}\left[\prod_{m=T}^{k+1}\mathbf{P}_{m}\mathbf{M}_{m}\right]^{\top}\left[\prod_{m=T}^{k+1}\mathbf{P}_{m}\mathbf{M}_{m}\right]\right]\\&\xrightarrow[]{P}
\frac{\labelnoise}{n}\mathrm{Tr}\left[\sum_{k=1}^{T}\overline{\frac{1}{n}\mathbf{P}_{k}\mathbf{\Sigma}_{k}^{\frac{1}{2}}\mathbf{Z}_{k}^{\top}\mathbf{Z}_{k}\mathbf{\Sigma}_{k}^{\frac{1}{2}}\mathbf{P}_{k}}\overline{\left[\prod_{m=T}^{k+1}\mathbf{P}_{m}\mathbf{M}_{m}\right]^{\top}\left[\prod_{m=T}^{k+1}\mathbf{P}_{m}\mathbf{M}_{m}\right]}\right],
\end{align*}
where in each step we condition on the event that each one of the matrices has finite operator norm and then use \cref{def:1}.
Summing the terms, we get
\begin{align*}
&G_{T}^{\mathrm{gen}}\xrightarrow{P}\left(\mathbf{w}_{0}-\mathbf{w}^{\star}\right)^{\top}\left[\left[\prod_{m=1}^{T}\bar{\mathbf{Q}}_{m}\right]^{2}+\sum_{i=2}^{T}\mathbf{\Sigma}_{i}\left[\prod_{r=1}^{i}\bar{\mathbf{Q}}_{r}\right]^{2}\tilde{\mu}_{i}\tilde{\rho}_{i}+\tilde{\rho}_{1}\tilde{\mu}_{1}\bar{\mathbf{Q}}_{1}^{2}\right]\left(\mathbf{w}_{0}-\mathbf{w}^{\star}\right)\\&\hspace{6em}+\frac{\labelnoise}{n}\mathrm{Tr}\Bigg[\sum_{k=1}^{T}\bar{\mathbf{Q}}_{k}\mathbf{M}_{k}^{-1}\left[\mathbf{I}-\bar{\mathbf{Q}}_{k}\left[\mathbf{I}+\frac{1}{n}\mathrm{Tr}\left[\mathbf{\Sigma}_{k}^{-1}\mathbf{M}_{k}^{-1}\bar{\mathbf{Q}}_{k}^{2}\right]\tilde{\mu}_{k}\mathbf{M}_{k}^{-1}\right]\right]\\&\hspace{12em}\left[\left[\prod_{m=k+1}^{T}\bar{\mathbf{Q}}_{m}\right]^{2}+\sum_{i=k+2}^{T}\mathbf{\Sigma}_{i}\left[\prod_{r=k+1}^{i}\bar{\mathbf{Q}}_{r}\right]^{2}\tilde{\mu}_{i}\tilde{\rho}_{i}+\tilde{\rho}_{k+1}\tilde{\mu}_{k+1}\bar{\mathbf{Q}}_{k+1}^{2}\right]\Bigg].
\end{align*}
\end{proof}
\subsection{Proof of \cref{thm:sparse}: A Failure of Isotropic Ridge Under Sparse Teachers}
\begin{recall}[\cref{thm:sparse}]
Assume that $\alpha:=\frac{n}{d}<1$. Let $\mathbf{w}_1^\star-\mathbf{w}_0$ be $k$-sparse with support $S$ so that $\mathbf{w}_1^\star\neq\mathbf{w}_0$, and assume that $\mathbf X_S$ (the submatrix of $\mathbf X$ consisting of the columns indexed by $S$) has full column rank, then as $n,d\to \infty$ after a single task with no noise, it holds almost surely that,
$$
\sup_{\mathbf M\in \mathcal M}\inf_{\lambda>0} \frac{G_1^{\mathrm{ridge}}}{G_1^{\mathrm{gen}}}=\infty.$$
where $\mathcal M=\bigcup_{0<l\leq u<\infty}
\left\{\mathbf M:l\mathbf I_d\preceq\mathbf M\preceq u\mathbf I_d\right\}$ and $u,l$ are constants independent of $n,d$.
\end{recall}
\begin{proof}
In this proof, we show that ridge has an error floor in the overparmaterized regime, while the generalized ridge can be optimized to sufficiently small error. 

Since $\mathbf y=\mathbf X\mathbf w^\star$, the ridge solution is
\[
\mathbf{w}_{1}^{\mathrm{ridge}}=\left(\frac{1}{n}\X^{\top}\X+\lambda\mathbf{I}_{d}\right)^{-1}\left(\frac{1}{n}\X^{\top}\X\mathbf{w}_{1}^{\star}+\lambda\mathbf{w}_{0}\right).
\]
Equivalently,
\begin{align*}  
\mathbf{w}_{1}^{\mathrm{ridge}}-\mathbf{w}_{1}^{\star}&=\left(\left(\frac{1}{n}\X^{\top}\X+\lambda\mathbf{I}_{d}\right)^{-1}\left(\frac{1}{n}\X^{\top}\X\mathbf{w}_{1}^{\star}+\lambda\mathbf{w}_{0}\right)-\mathbf{I}_{d}\right)\mathbf{w}_{1}^{\star}\\&=-\lambda\left(\frac{1}{n}\X^{\top}\X+\lambda\mathbf{I}_{d}\right)^{-1}\left(\mathbf{w}_{1}^{\star}-\mathbf{w}_{0}\right).
\end{align*}

Hence the squared estimation error is
\[
G_{1}^{\mathrm{ridge}}=\left\Vert \mathbf{w}_{1}^{\mathrm{ridge}}-\mathbf{w}^{\star}\right\Vert _{2}^{2}=\lambda^{2}\left(\mathbf{w}_{1}^{\star}-\mathbf{w}_{0}\right)^{\top}\left(\frac{1}{n}\X^{\top}\X+\lambda\mathbf{I}_{d}\right)^{-2}\left(\mathbf{w}_{1}^{\star}-\mathbf{w}_{0}\right).
\]
where $m_\alpha$ is the Stieltjes transform of the MP law with aspect ratio $\alpha$.
Therefore, almost surely, as $d\to\infty$ by the Marchenko-Pastur law,
\[
G_{1}^{\mathrm{ridge}}=\Vert\mathbf{w}_{1}^{\mathrm{ridge}}-\mathbf{w}^{\star}\Vert_{2}^{2}\to\frac{\lambda^{2}}{\featnoise^{2}}m_{\alpha}'\left(-\frac{\lambda}{\featnoise}\right)\left\Vert \mathbf{w}^{\star}-\mathbf{w}_{0}\right\Vert ^{2}.
\]
$\frac{\lambda^{2}}{\featnoise^{2}}m_{\alpha}'\left(-\frac{\lambda}{\featnoise}\right)=\tfrac{1}{2}\biggprn{1-\ratio+\frac{\lambda(1+\ratio)+\ratio\featnoise(1-\ratio)^{2}}{\sqrt{\lambda^{2}+\frac{2\lambda\featnoise\prn{1+\ratio}}{\ratio}+\frac{\featnoise^{2}\left(1-\ratio\right)^{2}}{\ratio^{2}}}}}$ is monotonic increasing as function of $\lambda$ as established in~\citet[Lemma E.3]{karpel2026optimal}, therefore $\frac{\lambda^2}{\featnoise^2}m_\alpha'\left(-\frac{\lambda}{\featnoise}\right)$ is bounded below by $1-\ratio > 0$, which implies that $\min_\lambda G^{\mathrm{ridge}}_1 > (1-\ratio)\left\Vert \mathbf{w_1}^{\star}-\mathbf w_0\right\Vert ^{2}$, almost surely, as $d\to \infty$.

Write $\mathbf P=\tfrac1n\mathbf X^\top \mathbf X$ and partition, 
\[
\mathbf{P}=\begin{pmatrix}\mathbf{P}_{SS} & \mathbf{P}_{SS^{c}}\\
\mathbf{P}_{S^{c}S} & \mathbf{P}_{S^{c}S^{c}}
\end{pmatrix},\qquad\mathbf{M}_{\tilde{\lambda},\psi}=\begin{pmatrix}\tilde{\lambda}\mathbf{I} & \mathbf{0}\\
\mathbf{0} & \psi\mathbf{I}
\end{pmatrix}.
\]
Without loss of generality, we assume $S=\{1,\dots,k\}$; otherwise, apply a permutation of coordinates. Since $\mathbf y=\mathbf X\mathbf w_1^\star$ with
$\mathbf w_1^\star-\mathbf{w}_0=(\left(\mathbf w_{1}^\star\right)_{S}-(\mathbf{w}_0)_S,\mathbf 0)$, we have
\[
\tfrac{1}{n}\mathbf{X}^{\top}(\mathbf{y}-\mathbf{X}\mathbf{w}_{0})=\mathbf{P}(\mathbf{w}_{1}^{\star}-\mathbf{w}_{0})=\binom{\mathbf{P}_{SS}\left(\left(\mathbf{w}_{1}^{\star}\right)_{S}-(\mathbf{w}_{0})_{S}\right)}{\mathbf{P}_{S^{c}S}\left(\left(\mathbf{w}_{1}^{\star}\right)_{S}-(\mathbf{w}_{0})_{S}\right)}.
\]
The estimator $\hat{\mathbf w}=(\hat{\mathbf w}_S,\hat{\mathbf w}_{S^c})$
satisfies
\[
(\mathbf{P}+\mathbf{M}_{\tilde{\lambda},\psi})(\hat{\mathbf{w}}-\mathbf{w}_{0})=\mathbf{P}(\mathbf{w}_{1}^{\star}-\mathbf{w}_{0}),
\]
that is,
\begin{align*}
(\mathbf{P}_{SS}+\tilde{\lambda}\mathbf{I})\left(\hat{\mathbf{w}}_{S}-(\mathbf{w}_{0})_{S}\right)+\mathbf{P}_{SS^{c}}\left(\hat{\mathbf{w}}_{S^{c}}-(\mathbf{w}_{0})_{S^{c}}\right)&=\mathbf{P}_{SS}\left(\left(\mathbf{w}_{1}^{\star}\right)_{S}-(\mathbf{w}_{0})_{S}\right),\\\mathbf{P}_{S^{c}S}\left(\hat{\mathbf{w}}_{S}-(\mathbf{w}_{0})_{S}\right)+(\mathbf{P}_{S^{c}S^{c}}+\psi\mathbf{I})\left(\hat{\mathbf{w}}_{S^{c}}-(\mathbf{w}_{0})_{S^{c}}\right)&=\mathbf{P}_{S^{c}S}\left(\left(\mathbf{w}_{1}^{\star}\right)_{S}-(\mathbf{w}_{0})_{S}\right).
\end{align*}
Solving the second equation gives
\[
\hat{\mathbf{w}}_{S^{c}}-(\mathbf{w}_{0})_{S^{c}}=(\mathbf{P}_{S^{c}S^{c}}+\psi\mathbf{I})^{-1}\mathbf{P}_{S^{c}S}\left(\left(\mathbf{w}_{1}^{\star}\right)_{S}-(\mathbf{w}_{0})_{S}-\left(\hat{\mathbf{w}}_{S}-(\mathbf{w}_{0})_{S}\right)\right).
\]
Substituting into the first equation yields
\begin{align*}
&\left(\mathbf{P}_{SS}+\tilde{\lambda}\mathbf{I}-\mathbf{P}_{SS^{c}}(\mathbf{P}_{S^{c}S^{c}}+\psi\mathbf{I})^{-1}\mathbf{P}_{S^{c}S}\right)\left(\hat{\mathbf{w}}_{S}-(\mathbf{w}_{0})_{S}\right)\\&\qquad=\mathbf{P}_{SS}\left(\left(\mathbf{w}_{1}^{\star}\right)_{S}-(\mathbf{w}_{0})_{S}\right)-\mathbf{P}_{SS^{c}}(\mathbf{P}_{S^{c}S^{c}}+\psi\mathbf{I})^{-1}\mathbf{P}_{S^{c}S}\left(\left(\mathbf{w}_{1}^{\star}\right)_{S}-(\mathbf{w}_{0})_{S}\right).
\end{align*}

Since $\Vert(\mathbf P_{S^cS^c}+\psi\mathbf I)^{-1}\Vert_2\leq\psi^{-1}$, the terms involving
$(\mathbf P_{S^cS^c}+\psi\mathbf I)^{-1}$ vanish as $\psi\to\infty$, giving
\[
(\mathbf{P}_{SS}+\tilde{\lambda}\mathbf{I})\left(\hat{\mathbf{w}}_{S}-(\mathbf{w}_{0})_{S}\right)\to\mathbf{P}_{SS}\left(\left(\mathbf{w}_{1}^{\star}\right)_{S}-(\mathbf{w}_{0})_{S}\right),\qquad\hat{\mathbf{w}}_{S^{c}}-(\mathbf{w}_{0})_{S^{c}}\to\mathbf{0}.
\]
Thus
\[
\lim_{\psi\to\infty}\hat{\mathbf{w}}(\tilde{\lambda},\psi)=\mathbf{w}_{0}+\left((\mathbf{P}_{SS}+\tilde{\lambda}\mathbf{I})^{-1}\mathbf{P}_{SS}\left(\left(\mathbf{w}_{1}^{\star}\right)_{S}-(\mathbf{w}_{0})_{S}\right),\mathbf{0}\right).
\]
Since $\mathbf P_{SS}\succ0$, then
$(\mathbf P_{SS}+\tilde{\lambda}\mathbf I)^{-1}\mathbf P_{SS}\to\mathbf I$
as $\tilde{\lambda}\to0$, hence $\mathbf{w}_1\to\mathbf w^\star_1$, which implies that
$G_1^{\mathrm{gen}}\to0$, proving the claim.

In the statement of \cref{thm:sparse} the supremum is taken over the set of matrices which are spectrally bounded. This is an immediate consequence of the proof. 
For each $q\geq 1$ we may define the set $\mathcal M_q=\{\frac{1}{q}\mathbf I \preceq \mathbf M \preceq q\mathbf I\}\in \mathcal M$ for this choice for all finite $q$ we have that $\mathcal M_q$ is spectrally bounded. In the construction of the regularization matrix, we take $\psi\to\infty$ and $\tilde\lambda\to0$, this corresponds to taking $\tilde\lambda=1/q$ and $\psi=q$ . The resulting matrix belongs to $\mathcal M_q$ for every finite $q$. The proof shows that, along this sequence, $G_1^{\mathrm{gen}}\to0$, while the optimal isotropic-ridge loss remains bounded away from zero. Consequently, $\lim_{q\to\infty}\sup_{\mathbf M\in\mathcal M_q}\inf_{\lambda>0}\frac{G_1^{\mathrm{ridge}}}{G_1^{\mathrm{gen}}}=\infty.$
\end{proof}

\newpage

\subsection{Necessary and Sufficient Symmetry Conditions For a Guaranteed Optimal Isotropic Diagonal Regularizer}
\label{sec:nspr}
\begin{theorem}[Necessary and sufficient condition to a guaranteed optimal isotropic regularization]
Assume that the generalization loss admits a unique diagonal minimizer 
$\mathbf M^\star=\mathrm{diag}(m_1^\star,\dots,m_d^\star)$. 
Then $\mathbf M^\star=\lambda^\star \mathbf I_d$ for some $\lambda^\star>0$ 
if and only if, for each coordinate, the teacher is equally far from the initialization (i.e., $|w_1^\star-\left(w_{0}\right)_{1}|=|w_2^\star-\left(w_{0}\right)_{2}|=\ldots=|w_d^\star-\left(w_{0}\right)_{d}|$), either deterministically or with probability tending to 1 (as $n,d\to\infty$) for random teacher.
\end{theorem}
\begin{proof}
We write $\bm m=(m_1,\dots,m_d)$ and $\mathbf M=\mathrm{diag}(\bm m)$, 

Recall \cref{thm:4}, under \cref{assum:1}, the expected generalization loss after $T$ tasks of the generalized ridge algorithm (\cref{schm:generalized}), satisfies,
\begin{align}
\label{eq:30}
&G^{\mathrm{gen}}_T\xrightarrow[]{P}\left(\mathbf{w}_{0}-\mathbf{w}^{\star}\right)^{\top}\left(\bar{\mathbf Q}^{2T}+\sum_{k=1}^{T}\bar{\mathbf Q}^{2k}\tilde{\mu}\tilde{\rho}_{k}\right)\left(\mathbf{w}_{0}-\mathbf{w}^{\star}\right)\\&+\frac{\labelnoise}{n}\mathrm{Tr}\Bigg[\sum_{k=1}^{T}\left[\mathbf{I}-\bar{\mathbf Q}\left[\mathbf{I}+\frac{1}{n}\mathrm{Tr}\left[\mathbf{M}^{-1}\bar{\mathbf Q}^{2}\right]\tilde{\mu}\mathbf{M}^{-1}\right]\right]\left[\bar{\mathbf Q}^{2\left(T-k\right)}+\sum_{i=k+1}^{T}\bar{\mathbf Q}^{2\left(i-k\right)}\tilde{\mu}\tilde{\rho}_{i}\right]\mathbf{M}^{-1}\bar{\mathbf Q}\Bigg]
\end{align}
where
$$
\bar{\mathbf Q}=\big(\mathbf I+\featnoise \tilde m_p(-1)\mathbf M^{-1}\big)^{-1},
$$
with $\tilde m_p(-1)$ being the unique positive solution of
$$
\frac{1}{\tilde m_p(-1)}=1+\frac{1}{n}\sum_{i=1}^d\frac{\featnoise  M_{i,i}^{-1}}{1+\tilde m_p(-1)\featnoise M_{i,i}^{-1}}.
$$
And,
\begin{align*}   
\tilde{\rho}_{k}=\frac{1}{n}\mathrm{Tr}\Big[\mathbf{M}^{-2}\Big(\bar{\mathbf{Q}}^{2(T-k+1)}+\tilde{\mu}\sum_{i=k+1}^{T}\tilde{\rho}_{i}\bar{\mathbf{Q}}^{2(i-k+1)}\Big)\Big],\\\tilde{\mu}=\left[\left[1+\frac{1}{n}\mathrm{Tr}\left(\mathbf{M}^{-1}\bar{\mathbf{Q}}\right)\right]^{2}-\frac{1}{n}\mathrm{Tr}\left(\mathbf{M}^{-1}\bar{\mathbf{Q}}\right)^{2}\right]^{-1}.
\end{align*}

With diagonal $\mathbf M$, the deterministic equivalent resolvent appearing in \cref{eq:30} reduces to
\begin{equation}
\label{eq:Qbar_diag_pf}
\bar{\mathbf{Q}}=\left(\mathbf{I}_{d}+\featnoise\tilde{m}_{p}(-1)\mathbf{M}^{-1}\right)^{-1}=\left(\mathbf{I}_{d}+\tilde{m}_{p}(-1)\featnoise\mathbf{M}^{-1}\right)^{-1}.
\end{equation}
Hence $\bar{\mathbf Q}$ is diagonal with entries
\[
q_i=q_i(m_i)=\frac{1}{1+\frac{\tilde m_{p}(-1)\featnoise}{m_i}}=\frac{m_i}{m_i+\tilde m_{p}(-1)\featnoise}.
\]
All matrices built from $\bar{\mathbf Q}$, $\mathbf M^{\pm1}$
(e.g.\ powers of $\bar{\mathbf Q}$, $\bar{\mathbf Q}\mathbf M^{-1}\bar{\mathbf Q}$,
and trace terms such as $\frac1n\trace[(\mathbf M^{-1}\bar{\mathbf Q})^k]$)
are therefore diagonal, and depend on $\mathbf m$ only through the coordinate-wise pairs
$(m_i,q_i(m_i))$ and through symmetric trace scalars.

Using the closed-form expression in \cref{eq:30} and the diagonal structure
from \cref{eq:Qbar_diag_pf}, the expected error decomposes as
\[
\mathcal R_T(\bm m)=\mathcal R_T^{\mathrm{teacher}}(\bm m)+\mathcal R_T^{\mathrm{noise}}(\bm m).
\]
The teacher-dependent term can be written as
\begin{equation}
\label{eq:teach_term_diag}
\mathcal R_T^{\mathrm{teacher}}(\bm m)=(\mathbf w^\star-\mathbf{w}_0)^\top \mathbf B_T(\bm m)(\mathbf w^\star-\mathbf{w}_0)=\sum_{i=1}^d b_{T,i}(\bm m)({w^\star_i}-{w_0}_i)^2,
\end{equation}
where $\mathbf B_T(\bm m)$ is diagonal and $b_{T,i}(\bm m)$ denotes its $i$'th diagonal entry.
The noise term $\mathcal R_T^{\mathrm{noise}}(\bm m)$ does not depend on $\mathbf w^\star$ and depends on $\bm m$ only through symmetric trace scalars.
Consequently, $\mathcal R_T^{\mathrm{noise}}(\bm m)$ is invariant under any permutation
of the coordinates of $\bm m$.

\medskip
Fix $\lambda>0$ and consider the isotropic point $\bm m=\lambda\mathbf 1$.
By permutation symmetry of $\mathcal R_T^{\mathrm{noise}}$, there exists a scalar $B_T(\lambda)$ such that
\[
\frac{\partial}{\partial m_{i}}\mathcal{R}_{T}^{\mathrm{noise}}(\bm{m})\Big|_{\bm{m}=\lambda\mathbf{1}}=B_{T}(\lambda)\qquad\text{for all }i.
\]
Likewise, at $\bm m=\lambda\mathbf 1$ all diagonal entries $b_{t,i}(\bm m)$ coincide, and their
coordinate derivatives are equal. We define
\begin{equation*}
A_{T}(\lambda):=\frac{\partial}{\partial m_{i}}b_{T,i}(\bm{m})\Big|_{\bm{m}=\lambda\mathbf{1}},\qquad B_{T}(\lambda):=\frac{\partial}{\partial m_{i}}\mathcal{R}_{T}^{\mathrm{noise}}(\bm{m})\Big|_{\bm{m}=\lambda\mathbf{1}}.
\end{equation*}
Differentiating $\mathcal R_T(\bm m)$ and evaluating at $\bm m=\lambda\mathbf 1$ yields
\begin{equation}
\label{eq:grad_iso_pf}
\frac{\partial}{\partial m_i}\mathcal R_T(\bm m)\Big|_{\bm m=\lambda\mathbf 1}=({w^\star_i}-{w_0}_i)^2A_T(\lambda)+B_T(\lambda),
\qquad i=1,\dots,d .
\end{equation}

Assume that $\mathbf M=\lambda^\star\mathbf I_d$ is a diagonal minimizer.
Then the first-order optimality conditions imply that
\cref{eq:grad_iso_pf} vanishes for every $i$, i.e.
\[
({w^\star_i}-{w_0}_i)^2A_T(\lambda^\star)+B_T(\lambda^\star)=0\qquad\text{for all }i .
\]
Subtracting the conditions for indices $i$ and $j$ gives
$$
\left(\left({w^\star_i}-{w_0}_i\right)^2-\left({w^\star_j}-{w_0}_j\right)^2\right)A_T(\lambda)=0.
$$
In the case of $\mathbf M=\lambda \mathbf I_d$. We have a compact closed solution given in~\citet[Theorem 3]{karpel2026optimal} where\footnote{While their expression is stated for the expectation, their analysis can be naturally extended to yield the corresponding convergence-in-probability statement under the same asymptotic regime.}, 
\begin{align*}
\left\Vert\mathbf{w}^{\mathrm{ridge}}_{T}-\mathbf{w}^{\star}\right\Vert ^{2}\xrightarrow[]{P}&\labelnoise{c}\frac{1-a^{T}}{1-a}+a^{T}\left\Vert\mathbf{w}_0-\mathbf{w}^{\star}\right\Vert ^{2}.
\end{align*}
and 
\[
\begin{aligned}
a&=\tfrac{1}{2}\biggprn{1-\ratio+\frac{\lambda\ratio(1+\ratio)+\featnoise(1-\ratio)^{2}}{\sqrt{\left(\lambda\ratio\right)^{2}+2\lambda\ratio\featnoise\prn{1+\alpha}+\featnoise^{2}\prn{1-\ratio}^{2}}}},\\c&=\tfrac{1}{2\featnoise}\biggprn{\frac{\featnoise(1+\ratio)+\lambda\ratio}{\sqrt{\left(\lambda\ratio\right)^{2}+2\lambda\ratio\featnoise(1+\ratio)+\featnoise^{2}(1-\ratio)^{2}}}-1}.
\end{aligned}
\]

By~\citet[Lemma D.3]{karpel2026optimal} $\frac{\dd}{\dd \lambda}a=\frac{2\tilde\lambda\ratio}{D^3\featnoise}$ where $D=\sqrt{\tilde \lambda^2+2\tilde\lambda(1+\ratio)+(1-\ratio)^2}$ and $\tilde \lambda=\frac{\lambda\ratio}{\featnoise}$ which is therefore strictly positive for all $\lambda>0$, hence $\frac{\dd}{\dd\lambda}\mathcal R_T^{\mathrm{noise}}(\lambda^\star\mathbf 1)\neq 0$.\footnote{In general, if $\left\Vert\mathbf w^\star-\mathbf{w}_0\right\Vert=0$, then the optimal value of the regularization parameter would force $\lambda^\star\to\infty$. Consequently, an optimal diagonal minimizer does not exist, contrary to the assumption of the Theorem.} By permutation symmetry, $\frac{\dd}{\dd\lambda}\mathcal R_T^{\mathrm{noise}}(\lambda\mathbf 1)=dB_T(\lambda)$,
therefore $B_T(\lambda^\star)\neq 0$.
Using the isotropic first-order conditions $\left({w^\star_i}-{w_0}_i\right)^2A_T(\lambda^\star)+B_T(\lambda^\star)=0$ for all $i$, we conclude that $A_T(\lambda^\star)\neq 0$ (otherwise stationarity would force $B_T(\lambda^\star)=0$), and therefore $\left({w^\star_i}-{w_0}_i\right)^2=\left({w^\star_j}-{w_0}_j\right)^2$ for all $i,j$ deterministically, or with probability tending to 1 (as $n,d \to\infty$) for random teacher.

Conversely, assume $\left({w^\star_1}-{w_0}_1\right)^2=\ldots=\left({w^\star_d}-{w_0}_d\right)^2=\tau^2$  deterministically, or with probability tending to 1 (as $n,d \to\infty$) for random teacher.
Then \cref{eq:teach_term_diag} reduces to
$\tau^2\sum_{i=1}^d b_{T,i}(\bm m)$, which is permutation invariant.
Together with the invariance of $\mathcal R_T^{\mathrm{noise}}(\bm m)$,
this shows that $\mathcal R_T(\bm m)$ is a symmetric function of $\bm m$.

Let $\bm m^\star$ be a diagonal minimizer.
For any permutation $\pi$ of $\{1,\dots,d\}$,
symmetry implies $\mathcal R_T(\bm m^\star)=\mathcal R_T(\pi\bm m^\star)$.
If the diagonal minimizer is unique, it must be fixed by all permutations,
which is only possible if $m_1^\star=\ldots=m_d^\star$.
Thus $\mathbf M^\star=\lambda^\star\mathbf I_d$ for some $\lambda^\star>0$.

\end{proof}

\newpage
\section{Proofs for \cref{sec:3}}
\subsection{Proof of \cref{thm:sin_closed}: Closed-Form Expression for Generalization Error Under the Sinusoidal Model}
\label{pr:sin}
\begin{recall}[\cref{thm:sin_closed}]
Under \cref{assump:sin_teachers} we have,
$$
G_{\infty}\xrightarrow[]{P}C_{0}+\sqrt{C_{1}^{2}+C_{2}^{2}}\sup_{T \ge 1}\cos\left(4\pi f T + \varphi\right)
$$
where,
\begin{equation}
\label{eq:40}    
\begin{aligned}
&\varphi=2\pi f-\arctan2\left(C_{2},C_{1}\right)
\\&C_{0}=v_{z}c'\frac{1}{1-a'}+2\left\Vert \mathbf B\right\Vert ^{2}\sin^{2}\left(\pi f\right)\left(\frac{1}{1-a'}-1\right)\left(2c_{2}-1\right),\\&C_{1}=4\left\Vert \mathbf B\right\Vert ^{2}\sin^{2}\left(\pi f\right)\left(\frac{1}{2}\left(c_{1}-1\right)+\left(c_{2}-1\right)\left(c_{1}-1\right)-c_{4}c_{3}\right),\\&C_{2}=4\left\Vert \mathbf B\right\Vert ^{2}\sin^{2}\left(\pi f\right)\left(-\frac{1}{2}c_{3}+c_{4}\left(c_{1}-1\right)+c_{2}c_{3}\right),\\&c_1=\frac{1-a'\cos\left(4\pi f\right)}{1-2a'\cos\left(4\pi f\right)+a'^{2}},\\&
c_2=\frac{1-b'\cos\left(2\pi f\right)}{1-2b'\cos\left(2\pi f\right)+b'^{2}},\\&
c_3=\frac{a'\sin\left(4\pi f\right)}{1-2a'\cos\left(4\pi f\right)+a'^{2}},\\&
c_4=\frac{b'\sin\left(2\pi f\right)}{1-2b'\cos\left(2\pi f\right)+b'^{2}},\\&
a'=1+\sum_{k=0}^{2N}{\binom{2N}{k}}\left(-\gamma\right)^{k}\sum_{r=0}^{k-1}\frac{\featnoise^{k}}{r+1}{\binom{k}{r}}{\binom{k-1}{r}}\frac{1}{\ratio^r} ,\\&b'=1+\sum_{k=0}^{N}{\binom{N}{k}}\left(-\gamma\right)^{k}\sum_{r=0}^{k-1}\frac{\featnoise^{k}}{r+1}{\binom{k}{r}}{\binom{k-1}{r}}\frac{1}{\ratio^r},\\&c'=\sum_{s=0}^{2N-2}m_{s}\sum_{l=0}^{s}{\binom{s}{l}}\left(-\gamma\right)^{l+2}\sum_{r=0}^{l}\frac{\featnoise^{l+1}}{r+1}{\binom{l+1}{r}}{\binom{l}{r}}\frac{1}{\ratio^{r+1}}.
\end{aligned}
\end{equation}
\end{recall}

\begin{proof}

We first show that, for all $i\in [T]$,

\begin{align*}
&\limsup_{T\to\infty}\left\Vert \mathbf{w}_{T}-\mathbf{w}_{i}^{\star}\right\Vert ^{2}\\&\xrightarrow{P}\limsup_{T\to\infty}\Bigg\{ v_{z}c\frac{1}{1-a}+2\left\Vert \mathbf{B}\right\Vert ^{2}\sin^{2}\left(\pi f\right)\\&\hspace{8em}\Bigg(\left(\frac{1}{1-a}-1\right)+\cos\left(4\pi fT+2\pi f\right)\left(c_{1}-1\right)+\sin\left(4\pi fT+2\pi f\right)c_{3}\Bigg)\\&\hspace{4em}+4\left\Vert \mathbf{B}\right\Vert ^{2}\sin^{2}\left(\pi f\right)\Big\{\left(\frac{1}{1-a}-1\right)\left(c_{2}-1\right)+\Big[\left(c_{2}-1\right)\cos\left(4\pi fT+2\pi f\right)\\&\hspace{18em}+\frac{b\sin\left(2\pi f\right)}{1-2b\cos\left(2\pi f\right)+b^{2}}\sin\left(4\pi fT+2\pi f\right)\Big]\left(c_{1}-1\right)\Big\}\\&\quad+\left[-c_{4}\cos\left(4\pi fT+2\pi f\right)+\left(c_{2}-1\right)\sin\left(4\pi fT+2\pi f\right)\right]c_{3}\end{align*}\begin{align*}
&-8\left\Vert \mathbf{B}\right\Vert ^{2}\cos\left(\pi f\left(T+i\right)\right)\sin\left(\pi f\left(T-i\right)\right)\sin\left(\pi f\right)\left(\cos\left(2\pi fT+\pi f\right)\left(c_{2}-1\right)+\sin\left(2\pi fT+\pi f\right)c_{4}\right)\\&\hspace{12em}+4\left\Vert \mathbf{B}\right\Vert ^{2}\cos^{2}\left(\pi f\left(T+i\right)\right)\sin^{2}\left(\pi f\left(T-i\right)\right)\Bigg\}
\end{align*}

Recall \cref{thm:closedloss},
\[
\begin{aligned}
G_{T}^{\mathrm{esgd}}&\xrightarrow{P}\labelnoise c'\frac{1-a'^{T}}{1-a'}+\frac{1}{T}\sum_{i=1}^{T}\left\Vert \mathbf{w}_{T}^{\star}-\mathbf{w}_{i}^{\star}\right\Vert ^{2}+2\big(\mathbf{w}_{T}^{\star}-\bar{\mathbf{w}}^{\star}\big)^{\top}\sum_{i=1}^{T}b'^{T-i+1}\big(\mathbf{w}_{i-1}^{\star}-\mathbf{w}_{i}^{\star}\big)\\&\quad+\sum_{i=1}^{T}\sum_{i'=1}^{T}a'^{T-\max(i,i')+1}b'^{|i-i'|}\big(\mathbf{w}_{i-1}^{\star}-\mathbf{w}_{i}^{\star}\big)^{\top}\big(\mathbf{w}_{i'-1}^{\star}-\mathbf{w}_{i'}^{\star}\big),
\end{aligned}
\]
where we define $\mathbf w^\star_0 \triangleq \mathbf{w}_0$ (by abuse of notation), the average teacher $\bar{\mathbf{w}}^{\star} \triangleq \frac{1}{T}\sum_{i=1}^T \mathbf w^\star_i$, and

\begin{align*}
&a'=1+\sum_{k=0}^{2N}{\binom{2N}{k}}\left(-\gamma\right)^{k}\sum_{r=0}^{k-1}\frac{\featnoise^{k}}{r+1}{\binom{k}{r}}{\binom{k-1}{r}}\frac{1}{\ratio^r} ,\\&b'=1+\sum_{k=0}^{N}{\binom{N}{k}}\left(-\gamma\right)^{k}\sum_{r=0}^{k-1}\frac{\featnoise^{k}}{r+1}{\binom{k}{r}}{\binom{k-1}{r}}\frac{1}{\ratio^r},\\&c'=\sum_{s=0}^{2N-2}m_{s}\sum_{l=0}^{s}{\binom{s}{l}}\left(-\gamma\right)^{l+2}\sum_{r=0}^{l}\frac{\featnoise^{l+1}}{r+1}{\binom{l+1}{r}}{\binom{l}{r}}\frac{1}{\ratio^{r+1}}.
\end{align*}

In particular, it is shown in the proof of the Theorem that for all $i\in [T]$ (see \cref{eq:mainresult1}), 
\begin{equation}
\label{eq:34} 
\begin{aligned}
\left\Vert \mathbf{w}^{\mathrm{esgd}}_{T}-\mathbf{w}_{i}^{\star}\right\Vert ^{2}&\xrightarrow[]{P}a'^{T}\left\Vert \mathbf{w}_{0}-\mathbf{w}_{1}^{\star}\right\Vert ^{2}+v_{z}c'\frac{1-a'^{T}}{1-a'}\\&\quad+2\underbrace{\sum_{k=2}^{T}a'^{T-k+1}b'^{k-1}\left(\mathbf{w}_{k-1}^{\star}-\mathbf{w}_{k}^{\star}\right)^{\top}\left(\mathbf{w}_{0}-\mathbf{w}_{1}^{\star}\right)}_{\text{term 1}}+2b'^{T}\left(\mathbf{w}_{T}^{\star}-\mathbf{w}_{i}^{\star}\right)^{\top}\left(\mathbf{w}_{0}-\mathbf{w}_{1}^{\star}\right)\\&\quad+\underbrace{\sum_{k=2}^{T}\sum_{k'=2}^{T}a'^{T-(k \lor k')+1}b'^{\left|k-k'\right|}\left(\mathbf{w}_{k-1}^{\star}-\mathbf{w}_{k}^{\star}\right)^{\top}\left(\mathbf{w}_{k'-1}^{\star}-\mathbf{w}_{k'}^{\star}\right)}_{\text{term 2}}\\&\quad+2\underbrace{\sum_{k=2}^{T}b'^{T-k+1}\left(\mathbf{w}_{T}^{\star}-\mathbf{w}_{i}^{\star}\right)^{\top}\left(\mathbf{w}_{k-1}^{\star}-\mathbf{w}_{k}^{\star}\right)}_{\text{term 3}}+\left\Vert \mathbf{w}_{T}^{\star}-\mathbf{w}_{i}^{\star}\right\Vert ^{2}
\end{aligned}
\end{equation}
Let $\mathbf{w}_i^\star=\mathbf A+\sin\left(2\pi fi\right) \mathbf B$, for all $k \in [2, T]$ we have 
\begin{align*}
\mathbf{w}_{k-1}^{\star}-\mathbf{w}_{k}^{\star}&=\mathbf A+\sin\left(2\pi f\left(k-1\right)\right)\cdot \mathbf B-\left(\mathbf A+\sin\left(2\pi fk\right)\cdot \mathbf B\right)\\&=\mathbf B\left(\sin\left(2\pi f\left(k-1\right)\right)-\sin\left(2\pi fk\right)\right)\\&=-2\mathbf B\sin\left(\pi f\right)\cos\left(\pi f\left(2k-1\right)\right)
\end{align*}
and for all $i\in [1,T]$,
\begin{align*}
\mathbf{w}_{T}^{\star}-\mathbf{w}_{i}^{\star}&=\mathbf A+\sin\left(2\pi fT\right)\cdot \mathbf B-\left(\mathbf A+\sin\left(2\pi fi\right)\cdot \mathbf B\right)\\&=\mathbf B\left(\sin\left(2\pi fT\right)-\sin\left(2\pi fi\right)\right)\\&=2\mathbf B\cos\left(\pi f\left(T+i\right)\right)\sin\left(\pi f\left(T-i\right)\right)
\end{align*}

We first evaluate the following geometric sums,
\begin{equation}
\label{eq:31}
\begin{aligned}
&\sum_{n=0}^{\infty}q^{n}\cos\left(nx\right)=\frac{1}{2}\sum_{n=0}^{\infty}\left[\left(qe^{ix}\right)^{n}+\left(qe^{-ix}\right)^{n}\right]=\frac{1}{2}\left[\frac{1}{1-qe^{ix}}+\frac{1}{1-qe^{-ix}}\right]\\&=\frac{1}{2}\left[\frac{1-qe^{-ix}}{\left(1-qe^{ix}\right)\left(1-qe^{-ix}\right)}+\frac{1-qe^{ix}}{\left(1-qe^{-ix}\right)\left(1-qe^{ix}\right)}\right]\\&=\frac{1}{2}\left[\frac{2-q\left(e^{-ix}+e^{ix}\right)}{1-qe^{ix}-qe^{-ix}+q^{2}}\right]=\frac{1-q\cos\left(x\right)}{1-2q\cos\left(x\right)+q^{2}}
\end{aligned}
\end{equation}

\begin{equation}
\label{eq:32}
\begin{aligned}
&\sum_{n=0}^{\infty}q^{n}\sin\left(nx\right)=\frac{1}{2i}\sum_{n=0}^{\infty}\left[\left(qe^{ix}\right)^{n}-\left(qe^{-ix}\right)^{n}\right]=\frac{1}{2i}\left[\frac{1}{1-qe^{ix}}-\frac{1}{1-qe^{-ix}}\right]\\&=\frac{1}{2i}\left[\frac{1-qe^{-ix}}{\left(1-qe^{ix}\right)\left(1-qe^{-ix}\right)}-\frac{1-qe^{ix}}{\left(1-qe^{-ix}\right)\left(1-qe^{ix}\right)}\right]\\&=\frac{1}{2i}\left[\frac{1-qe^{-ix}}{1-qe^{ix}-qe^{-ix}+q^{2}}-\frac{1-qe^{ix}}{1-qe^{ix}-qe^{-ix}+q^{2}}\right]\\&=\frac{1}{2i}\frac{q\left(e^{ix}-e^{-ix}\right)}{1-2q\cos\left(x\right)+q^{2}}=\frac{q\sin\left(x\right)}{1-2q\cos\left(x\right)+q^{2}}
\end{aligned}
\end{equation}
Then, we proceed with evaluating term 1 in \cref{eq:34}
\begin{align*}
&\left|\limsup_{T\to\infty}\sum_{k=2}^{T}a'^{T-k+1}b'^{k-1}\cos\left(\pi f(2k-1)\right)\right|
\le
\limsup_{T\to\infty}\sum_{k=2}^{T}a'^{T-k+1}b'^{k-1}=
\limsup_{T\to\infty}\sum_{k=2}^{T}a'^{T+1}\left(\frac{b'}{a'}\right)^{k-1}\\
&=
\limsup_{T\to\infty} a'^{T+1}
\sum_{j=1}^{T-1}\left(\frac{b'}{a'}\right)^j =
\limsup_{T\to\infty}
a'^{T+1}
\frac{\frac{b'}{a'}\left(\left(\frac{b'}{a'}\right)^{T-1}-1\right)}{\frac{b'}{a'}-1}\stackrel{\explain{1}}{=}
\limsup_{T\to\infty}
\frac{a'b'^{T}-a'^{T} b'}{b'-a'}=0.
\end{align*}
Where [1] follows as $a',b'<1$ for $\gamma\in \big(0,2\featnoise^{-1}\left(1+{\alpha}^{-1/2}\right)^{-2}\big)$. That implies,
$$
\limsup_{T\to\infty}\sum_{k=2}^{T}a'^{T-k+1}b'^{k-1}\cos\left(\pi f(2k-1)\right)=0.
$$
And term 2 in \cref{eq:34} becomes,
\begin{align*}
&\limsup_{T\to\infty}4\left\Vert \mathbf{B}\right\Vert ^{2}\sin^{2}\left(\pi f\right)\sum_{k=2}^{T}\sum_{k'=2}^{T}a'^{T-(k\lor k')+1}b'^{|k-k'|}\cos\left(\pi f\left(2k-1\right)\right)\cos\left(\pi f\left(2k'-1\right)\right)\\=&\limsup_{T\to\infty}4\left\Vert \mathbf{B}\right\Vert ^{2}\sin^{2}\left(\pi f\right)\sum_{k=0}^{T-2}a'^{T-k-1}\cos^{2}\left(\pi f\left(2k+3\right)\right)\\&\quad+8\left\Vert \mathbf{B}\right\Vert ^{2}\sin^{2}\left(\pi f\right)\sum_{p=0}^{T-3}b'^{p+1}\sum_{k=1}^{T-p-2}a'^{k}\cos\left(\pi f\left(2\left(T-k-p\right)-1\right)\right)\cos\left(\pi f\left(2\left(T-k\right)+1\right)\right)\\=&\limsup_{T\to\infty}4\left\Vert \mathbf{B}\right\Vert ^{2}\sin^{2}\left(\pi f\right)\sum_{k=1}^{T-1}a'^{k}\left(\frac{1+\cos\left(2\left(\pi f\left(2\left(T-k\right)+1\right)\right)\right)}{2}\right)\\&\quad+8\left\Vert \mathbf{B}\right\Vert ^{2}\sin^{2}\left(\pi f\right)\sum_{p=0}^{T-3}b'^{p+1}\sum_{r=1}^{T-p-1}a'^{r}\left(\frac{1}{2}\cos\left(-2\pi f\left(1+p\right)\right)+\frac{1}{2}\cos\left(4\pi f\left(T-r\right)-2p\pi f\right)\right)\\=&\limsup_{T\to\infty}4\left\Vert \mathbf{B}\right\Vert^{2}\sin^{2}\left(\pi f\right)\sum_{k=1}^{T-1}a'^{k}\left(\frac{1+\cos\left(-4\pi fk+4\pi fT+2\pi f\right)}{2}\right)\\&\quad+4\left\Vert \mathbf{B}\right\Vert ^{2}\sin^{2}\left(\pi f\right)\sum_{p=0}^{T-2}b'^{p+1}\sum_{r=1}^{T-p-1}a'^{r}\left(\cos\left(-2\pi f\left(1+p\right)\right)+\cos\left(-4\pi fr+4\pi fT-2p\pi f\right)\right)\end{align*}
\begin{align*}&=\limsup_{T\to\infty}2\left\Vert \mathbf{B}\right\Vert ^{2}\sin^{2}\left(\pi f\right)\\&\qquad\left[\sum_{k=1}^{T-1}a'^{k}+\sum_{k=1}^{T-1}a'^{k}\left(\cos\left(-4\pi fk\right)\cos\left(4\pi fT+2\pi f\right)-\sin\left(-4\pi fk\right)\sin\left(4\pi fT+2\pi f\right)\right)\right]\\&\quad+4\left\Vert \mathbf{B}\right\Vert ^{2}\sin^{2}\left(\pi f\right)\Big[\sum_{p=0}^{T-2}b'^{p+1}\sum_{r=1}^{T-p-1}a'^{r}\cos\left(-2\pi f\left(1+p\right)\right)\\&\hspace{10em}+\cos\left(4\pi fr\right)\cos\left(4\pi fT-2p\pi f\right)+\sin\left(4\pi fr\right)\sin\left(4\pi fT-2p\pi f\right)\Big]
\\
&=\limsup_{T\to\infty}2\left\Vert \mathbf{B}\right\Vert ^{2}\sin^{2}\left(\pi f\right)\\&\qquad\Bigg(\left(\frac{1}{1-a'}-1\right)+\cos\left(4\pi fT+2\pi f\right)\left(c_{1}-1\right)+\sin\left(4\pi fT+2\pi f\right)c_{3}\Bigg)\\&+4\left\Vert \mathbf{B}\right\Vert ^{2}\sin^{2}\left(\pi f\right)\left[\sum_{p=0}^{T-2}b'^{p+1}\Bigg[\cos\left(-2\pi f\left(1+p\right)\right)\left(\frac{1}{1-a'}-1\right)\right.\\&\hspace{10em}+\cos\left(4\pi fT-2p\pi f\right)\left(c_{1}-1\right)+\sin\left(4\pi fT-2p\pi f\right)c_{3}\Bigg]\\=&\limsup_{T\to\infty}2\left\Vert \mathbf{B}\right\Vert ^{2}\sin^{2}\left(\pi f\right)\left(\left(\frac{1}{1-a'}-1\right)+\cos\left(4\pi fT+2\pi f\right)\left(c_{1}-1\right)+\sin\left(4\pi fT+2\pi f\right)c_{3}\right)\\&+4\left\Vert \mathbf{B}\right\Vert ^{2}\sin^{2}\left(\pi f\right)\\&\quad\Bigg[\sum_{p=0}^{T-2}b'^{p+1}\Big[\left(\frac{1}{1-a'}-1\right)\cos\left(-2\pi f\left(p+1\right)\right)+\left(\cos\left(-2\left(p+1\right)\pi f\right)\cos\left(4\pi fT+2\pi f\right)\right.\\&\qquad+\sin\left(2\left(p+1\right)\pi f\right)\sin\left(4\pi fT+2\pi f\right)\left(c_{1}-1\right)+\sin\left(4\pi fT-2\left(p+1\right)\pi f+2\pi f\right)c_{3}\Big]\Bigg]\\=&\limsup_{T\to\infty}2\left\Vert \mathbf{B}\right\Vert ^{2}\sin^{2}\left(\pi f\right)\left(\left(\frac{1}{1-a'}-1\right)+\cos\left(4\pi fT+2\pi f\right)\left(c_{1}-1\right)+\sin\left(4\pi fT+2\pi f\right)c_{3}\right)\\&+4\left\Vert \mathbf{B}\right\Vert ^{2}\sin^{2}\left(\pi f\right)\\&\qquad\biggr[\left(\frac{1}{1-a'}-1\right)\left(c_{2}-1\right)+\left(\left(c_{2}-1\right)\cos\left(4\pi fT+2\pi f\right)+c_{4}\sin\left(4\pi fT+2\pi f\right)\right)\left(c_{1}-1\right)\\&\hspace{2em}+\sum_{p=0}^{t-2}b'^{p+1}\left[\left(-\sin\left(2\left(p+1\right)\pi f\right)\cos\left(4\pi fT+2\pi f\right)+\cos\left(2\left(p+1\right)\pi f\right)\sin\left(4\pi fT+2\pi f\right)\right)c_{3}\right]\biggl]\\&=\limsup_{t\to\infty}2\left\Vert \mathbf{B}\right\Vert ^{2}\sin^{2}\left(\pi f\right)\left(\left(\frac{1}{1-a'}-1\right)+\cos\left(4\pi fT+2\pi f\right)\left(c_{1}-1\right)+\sin\left(4\pi fT+2\pi f\right)c_{3}\right)\\&+4\left\Vert \mathbf{B}\right\Vert ^{2}\sin^{2}\left(\pi f\right)\\&\qquad\Bigg[\left(\frac{1}{1-a'}-1\right)\left(c_{2}-1\right)+\left[\left(c_{2}-1\right)\cos\left(4\pi fT+2\pi f\right)+c_{4}\sin\left(4\pi fT+2\pi f\right)\right]\left(c_{1}-1\right)\\&\hspace{14em}+\left[-c_{4}\cos\left(4\pi fT+2\pi f\right)+\left(c_{2}-1\right)\sin\left(4\pi fT+2\pi f\right)\right]c_{3}\Bigg].
\end{align*}
And term 3 in \cref{eq:34} is,
\begin{align*}
&\limsup_{T\to\infty}\sum_{k=2}^{T}b'^{T-k+1}\cos\left(\pi f\left(2k-1\right)\right)=\limsup_{T\to\infty}\sum_{j=1}^{T-1}b'^{j}\cos\left(\pi f\left(2\left(T-j+1\right)-1\right)\right)\\&=\limsup_{T\to\infty}\sum_{j=1}^{T-1}b'^{j}\cos\left(-2\pi fj+2\pi fT+\pi f\right)\\&=\limsup_{T\to\infty}\sum_{j=1}^{T-1}b'^{j}\cos\left(-2\pi fj\right)\cos\left(2\pi fT+\pi f\right)-\sum_{j=1}^{T-1}b^{j}\sin\left(-2\pi fj\right)\sin\left(2\pi fT+\pi f\right)\\&=\limsup_{T\to\infty}\cos\left(2\pi fT+\pi f\right)\left(\frac{1-b'\cos\left(-2\pi f\right)}{1-2b'\cos\left(-2\pi f\right)+b'^{2}}-1\right)\\&\hspace{16em}-\sin\left(2\pi fT+\pi f\right)\left(\frac{b'\sin\left(-2\pi f\right)}{1-2b\cos\left(-2\pi f\right)+b'^{2}}\right)
\end{align*}
Finally, the expression to $\limsup_{T\to\infty}\left\Vert \mathbf{w}_{T}-\mathbf{w}_{i}^{\star}\right\Vert^{2}$ follows by adding the expressions, and noting that,
\begin{align*}
&\limsup_{T\to\infty}a'^{T}\left\Vert \mathbf{w}_{0}-\mathbf{w}_{1}^{\star}\right\Vert ^{2}=0,\\&
\limsup_{T\to\infty}v_{z}c'\frac{-a'^{T}}{1-a'}=0,\\&\limsup_{T\to\infty}2b'^{T}\left(\mathbf{w}_{T}^{\star}-\mathbf{w}_{i}^{\star}\right)^{\top}\left(\mathbf{w}_{0}-\mathbf{w}_{1}^{\star}\right)=0,\end{align*}\begin{align*}&
\left\Vert \mathbf{w}_{T}^{\star}-\mathbf{w}_{i}^{\star}\right\Vert ^{2}=4\left\Vert \mathbf B\right\Vert^2\cos^2(\pi f(T+i))\sin^2(\pi f(T-i)).
\end{align*}
Note that, $\frac{1}{|A_T|}\sum_{i\in A_T}\left\Vert\mathbf{w}_T-\mathbf{w}_i^\star \right\Vert^2=\left\Vert\mathbf{w}_T-\mathbf{w}_i^\star \right\Vert^2$ since $
\mathbf{w}_i^\star=\mathbf{w}_j^\star$ for all $i,j\in A_T$.

Therefore, evaluating at the last predictor, we get,
\begin{align*}
\limsup_{T\to\infty}\frac{1}{|A_T|}\sum_{i\in A_T}\left\Vert\mathbf{w}_T-\mathbf{w}_i^\star \right\Vert^2=\limsup_{T\to\infty}C_{0}+C_{1}\cos\left(4\pi fT+2\pi f\right)+C_{2}\sin\left(4\pi fT+2\pi f\right)\\=\limsup_{T\to\infty}C_{0}+\sqrt{C_{1}^{2}+C_{2}^{2}}\cos\left(4\pi fT+2\pi f-\arctan2\left(C_{2},C_{1}\right)\right),
\end{align*}
where 
\begin{align*}
C_{0}&=v_{z}c'\frac{1}{1-a'}+2\left\Vert \mathbf B\right\Vert ^{2}\sin^{2}\left(\pi f\right)\left(\frac{1}{1-a'}-1\right)\left(2c_{2}-1\right),\\C_{1}&=4\left\Vert \mathbf B\right\Vert ^{2}\sin^{2}\left(\pi f\right)\left(\frac{1}{2}\left(c_{1}-1\right)+\left(c_{2}-1\right)\left(c_{1}-1\right)-c_{4}c_{3}\right),\\C_{2}&=4\left\Vert \mathbf B\right\Vert ^{2}\sin^{2}\left(\pi f\right)\left(-\frac{1}{2}c_{3}+c_{4}\left(c_{1}-1\right)+c_{2}c_{3}\right).
\end{align*}
\end{proof}

\end{document}